\documentclass[aap,preprint]{imsart}
\usepackage{amsfonts,latexsym}
\RequirePackage[OT1]{fontenc}
\RequirePackage{amsthm,amsmath}
\usepackage{hyperref}       % hyperlinks
\usepackage{url}            % simple URL typesetting
\usepackage{booktabs}       % professional-quality tables
\usepackage{nicefrac}       % compact symbols for 1/2, etc.
\usepackage{mathtools}
\usepackage{soul}
\mathtoolsset{showonlyrefs}
\usepackage{bm}
\usepackage{multirow}
\usepackage{scalerel}
\allowdisplaybreaks
\usepackage{dsfont}
\usepackage{algorithm}
\usepackage[noend]{algpseudocode}
\usepackage{xcolor}
\newtheorem{problem}{Problem}[section]
\newtheorem{condition}{Condition}[section]
\newtheorem{definition}{Definition}[section]

\newtheorem{theorem}{Theorem}[section]
\newtheorem*{thm*}{Theorem}

\newtheorem{lemma}{Lemma}[section]
\newtheorem{proposition}{Proposition}[section]
\newtheorem*{prop*}{Proposition}

\DeclareMathAlphabet{\mathbsf}{OT1}{cmss}{bx}{n}% bold sans serif
\DeclareMathAlphabet{\mathssf}{OT1}{cmss}{m}{sl}% slanted sans serif

\DeclareMathOperator*{\argmax}{arg\,max}
\DeclareMathOperator*{\argmin}{arg\,min}

\DeclarePairedDelimiterX{\infdivx}[2]{(}{)}{%
	#1\;\delimsize\|\;#2%
}

\newcommand{\rvu}{{\mathssf{u}}}	% random variable u
\newcommand{\rvv}{{\mathssf{v}}}	% random variable v
\newcommand{\rvw}{{\mathssf{w}}}	% random variable v
\newcommand{\rvx}{{\mathssf{x}}}	% random variable x
\newcommand{\rvy}{{\mathssf{y}}}	% random variable y
\newcommand{\hrvy}{\hat{\mathssf{y}}}	% random variable hat y
\newcommand{\trvy}{\tilde{\mathssf{y}}}	% random variable tilde y
\newcommand{\rvbv}{{\mathbsf{v}}} % random vector v
\newcommand{\rvbx}{{\mathbsf{x}}} % random vector x
\newcommand{\svbb}{{\mathbf{b}}} % deterministic vector b
\newcommand{\svbe}{{\mathbf{e}}} % deterministic vector e
\newcommand{\svbf}{{\mathbf{f}}} % deterministic vector f
\newcommand{\svbv}{{\mathbf{v}}} % deterministic vector v
\newcommand{\svbw}{{\mathbf{w}}} % deterministic vector w
\newcommand{\svbx}{{\mathbf{x}}} % deterministic vector x
\newcommand{\svby}{{\mathbf{y}}} % deterministic vector y
\newcommand{\hsvby}{\hat{\mathbf{y}}} % deterministic vector hat y
\newcommand{\tsvby}{\tilde{\mathbf{y}}} % deterministic vector tilde y
\newcommand{\bB}{{\mathbf{B}}} % deterministic matrix B
\newcommand{\bV}{{\mathbf{V}}} % deterministic matrix V
\newcommand{\bH}{{\mathbf{H}}} % deterministic matrix H
\newcommand{\hbH}{\hat{\mathbf{H}}} % deterministic matrix V

\newcommand{\bkappa}{\tilde{\kappa}} %%%%%%%%%%%
\newcommand{\tcO}{\tilde{\mathcal{O}}} % tilde over calligraphic O
\newcommand{\bcO}{\bar{\mathcal{O}}} % bar over calligraphic O
\newcommand{\Reals}{\mathbb{R}} % Reals
\newcommand{\cX}{\mathcal{X}} % calligraphic X 
\newcommand{\Density}{f_{\rvbx}( \svbx )} % density of random vector x
\newcommand{\NodePotential}{g_i}
\newcommand{\EdgePotential}{g_{ij}}
\newcommand{\Xlower}{b_l}
\newcommand{\Xupper}{b_u}
\newcommand{\phiMax}{\phi_{\max}} % maximum of phi
\newcommand{\hphiMax}{\bar{\phi}_{\max}} % maximum of derivative of phi
\newcommand{\btheta}{\boldsymbol{\theta}} % the deterministic vector theta of all the parameters
\newcommand{\ctheta}{\bar{\theta}} % the deterministic vector theta of all the
\newcommand{\ttheta}{\tilde{\theta}} % the deterministic vector theta of all the
\newcommand{\bphi}{\boldsymbol{\phi}} % the deterministic vector function phi
\newcommand{\bpsi}{\boldsymbol{\psi}} % the deterministic vector function psi
\newcommand{\bphiI}{\boldsymbol{\phi}^{(i)}} % the locally centered deterministic vector function phi
\newcommand{\bpsiI}{\boldsymbol{\psi}^{(i)}} % the locally centered deterministic vector function psi
\newcommand{\thetaIStar}{{\theta^{*(i)}}} % the scalar theta parameter associated with node i 
\newcommand{\thetaIJStar}{{\theta^{*(ij)}}} % the scalar theta parameter associated with node i
\newcommand{\thetaIoneStar}{{\theta_1^{*(i)}}} % the scalar theta parameter associated
\newcommand{\thetaItwoStar}{{\theta_2^{*(i)}}} % the scalar theta parameter associated
\newcommand{\thetaoneStar}{{\theta^{*(1)}}} % the scalar theta parameter associated
\newcommand{\thetatwoStar}{{\theta^{*(2)}}} % the scalar theta parameter associated
\newcommand{\thetaIJ}{{\theta^{(ij)}}} % the scalar theta parameter associated with edge i,j 
\newcommand{\thetaIJoneStar}{{\theta_1^{*(ij)}}}
\newcommand{\thetaIJtwoStar}{{\theta_2^{*(ij)}}}
\newcommand{\thetaonetwoStar}{{\theta^{*{(12)}}}}
\newcommand{\thetaIJoneoneStar}{{\theta_{1,1}^{*(ij)}}}
\newcommand{\thetaIJtwooneStar}{{\theta_{2,1}^{*(ij)}}}
\newcommand{\thetaIr}{{\theta^{(i)}_{r}}} % r^th component of the theta parameter associated with node i 
\newcommand{\thetaIrStar}{{\theta^{*(i)}_{r}}} % true r^th component of the theta parameter associated with node i 
\newcommand{\thetaIJrs}{{\theta_{r,s}^{(ij)}}} % rs^th component of the theta parameter associated with edge i,j 
\newcommand{\thetaIJrsStar}{{\theta_{r,s}^{*(ij)}}} % rs^th component of the theta parameter associated with edge i,j 
\newcommand{\thetaIMrsStar}{{\theta_{r,s}^{*(im)}}} % rs^th component of the theta parameter associated with edge i,m
\newcommand{\hthetaIr}{\hat{\theta}_{r}^{(i)}} % rs^th component of the theta parameter associated with edge i,j 
\newcommand{\hthetaIJrs}{\hat{\theta}_{r,s}^{(ij)}} % rs^th component of the theta parameter associated with edge i,j 

\newcommand{\bthetaI}{\boldsymbol{\theta}^{(i)}} % the deterministic vector theta associated with node i
\newcommand{\hbthetaI}{\hat{\boldsymbol{\theta}}^{(i)}} % the estimate of the deterministic vector theta associated with node i
\newcommand{\bthetaIT}{\boldsymbol{\theta}^{(i)^T}} % the transpose of the deterministic vector theta associated with node i
\newcommand{\bthetaIJ}{\boldsymbol{\theta}^{(ij)}} % the deterministic vector theta associated with edge i,j
\newcommand{\hbthetaIJ}{\hat{\boldsymbol{\theta}}^{(ij)}} % the estimate of the deterministic vector theta associated with edge i,j
\newcommand{\bthetaJI}{\boldsymbol{\theta}^{(ji)}}
\newcommand{\bthetaIJT}{\boldsymbol{\theta}^{(ij)^T}} % the transpose of the deterministic vector theta associated with edge i,j

\newcommand{\bthetaStar}{\boldsymbol{\theta}^*} % the true deterministic vector theta of all the parameters
\newcommand{\bthetaIStar}{\boldsymbol{\theta}^{*(i)}} % the true deterministic vector theta associated with node i
\newcommand{\bthetaIJStar}{\boldsymbol{\theta}^{*(ij)}} % the true deterministic vector theta associated with edge i,j
\newcommand{\bthetaIStarT}{\boldsymbol{\theta}^{*(i)^T}} % the true deterministic vector theta associated with node i
\newcommand{\bthetaIJStarT}{\boldsymbol{\theta}^{*(ij)^T}} % the true deterministic vector theta associated with edge i,j
\newcommand{\DensityParametrized}{f_{\rvbx}( \svbx ; \btheta )}
\newcommand{\DensityParametrizedTrue}{f_{\rvbx}( \svbx ; \bthetaStar )}

\newcommand{\DensityParametrizedfunTrue}{f_{\rvx}( \cdot ; \bthetaStar )}
\newcommand{\DensityParametrizedNotITrue}{f_{\rvx_{-i}}( x_{-i} ; \bthetaStar )}
\newcommand{\DensityParametrizedNotIfunTrue}{f_{\rvx_{-i}}( \cdot ; \bthetaStar )}
\newcommand{\thetaMax}{\theta_{\max}}
\newcommand{\thetaMin}{\theta_{\min}}

\newcommand{\hbtheta}{\hat{\boldsymbol{\theta}}}
\newcommand{\cN}{\mathcal{N}}
\newcommand{\hG}{\hat{G}}

\newcommand{\vtheta}{{\vartheta}}
\newcommand{\vthetaStar}{{\vartheta}^*}
\newcommand{\hvthetan}{\hat{{\vartheta}}_n}
\newcommand{\bvtheta}{\boldsymbol{\vartheta}}
\newcommand{\bvthetaT}{\boldsymbol{\vartheta}^T}
\newcommand{\bvthetaEdge}{\boldsymbol{\vartheta}_{E}}
\newcommand{\hbvthetaIn}{\hat{\boldsymbol{\vartheta}}^{(i)}_n}
\newcommand{\hbvthetaIEps}{\hat{\boldsymbol{\vartheta}}^{(i)}_{\epsilon}}
\newcommand{\hbvthetaIEpsUnc}{\hat{\boldsymbol{\vartheta}}^{(i)}_{\epsilon,\text{unc}}}
\newcommand{\hbvthetaIEpsEdge}{\hat{\boldsymbol{\vartheta}}^{(i)}_{\epsilon,E}}

\newcommand{\bvthetaIStar}{\boldsymbol{\vartheta}^{*(i)}}
\newcommand{\bvthetaIStarT}{\boldsymbol{\vartheta}^{*(i)^T}}
\newcommand{\bvthetaJStar}{\boldsymbol{\vartheta}^{*(j)}}

\newcommand{\bvthetaOneStar}{\boldsymbol{\vartheta}^{*(1)}}
\newcommand{\bvthetaTwoStar}{\boldsymbol{\vartheta}^{*(2)}}
\newcommand{\bvthetaIEdgeStar}{\boldsymbol{\vartheta}^{*(i)}_E}
\newcommand{\bvphiI}{\boldsymbol{\varphi}^{(i)}}
\newcommand{\bvphione}{\boldsymbol{\varphi}^{(1)}}
\newcommand{\bvphitwo}{\boldsymbol{\varphi}^{(2)}}

\newcommand{\brho}{\boldsymbol{\rho}}
\newcommand{\tbrho}{\tilde{\boldsymbol{\rho}}}
\newcommand{\hbrho}{\hat{\boldsymbol{\rho}}}
\newcommand{\bbrho}{\bar{\boldsymbol{\rho}}}
\newcommand{\rhoStar}{{\rho}^{*}}
\newcommand{\brhoStar}{\boldsymbol{\rho}^{*}}
\newcommand{\brhoStarT}{\boldsymbol{\rho}^{*T}}

\newcommand{\g}{\gamma}
\newcommand{\varphiMax}{\varphi_{\max}}
\newcommand{\ConditionalDensityNodeI}{f_{\rvx_i}(  x_i | \rvx_{-i} = x_{-i} ; \bvthetaIStar)}
\newcommand{\ConditionalDensityNodeIfun}{f_{\rvx_i}(  \cdot | \rvx_{-i} = x_{-i} ; \bvthetaIStar)}
\newcommand{\cS}{\mathcal{S}}
\newcommand{\cl}{\mathcal{L}}
\newcommand{\hlambda}{\hat{\lambda}}

\newcommand{\hblambda}{\hat{\boldsymbol{\lambda}}}
\newcommand{\lambdaStar}{{\lambda}^{*}}
\newcommand{\blambdaStar}{\boldsymbol{\lambda}^{*}}

\newcommand{\blambdaStarT}{\boldsymbol{\lambda}^{*T}}

\newcommand{\barbmu}{\bar{\boldsymbol{\mu}}}
\newcommand{\muStar}{{\mu}^{*}}
\newcommand{\bmuStar}{\boldsymbol{\mu}^{*}}
\newcommand{\hmu}{\hat{{\mu}}}
\newcommand{\hbmu}{\hat{\boldsymbol{\mu}}}
\newcommand{\UniformProxyDensity}{u_{\rvbx}^{(i)}(\svbx)}
\newcommand{\UniformProxyDensityfun}{u_{\rvbx}^{(i)}(\cdot)}
\newcommand{\DifferenceDensity}{m_{\rvbx}^{(i)}(\svbx; \bvtheta)}
\newcommand{\DifferenceDensityfun}{m_{\rvbx}^{(i)}(\cdot; \bvtheta)}
\newcommand{\cU}{\mathcal{U}}
\newcommand{\infdiv}{D\infdivx}
\newcommand*\conj[1]{\bar{#1}}
\newcommand{\tbtheta}{\tilde{\boldsymbol{\theta}}}
\newcommand{\tbthetaIJ}{\tilde{\boldsymbol{\theta}}^{(ij)}}
\newcommand{\cbtheta}{\conj{\boldsymbol{\theta}}}
\newcommand{\cbthetaIJ}{\conj{\boldsymbol{\theta}}^{(ij)}}
\newcommand{\qs}{q^s}

\newcommand{\tp}{\tilde{p}}
\newcommand{\betaStar}{\beta^*}
\newcommand{\BoundedNoise}{\eta}
\newcommand{\bbetaStar}{\boldsymbol{\beta}^*}
\newcommand{\tbbeta}{\tilde{\boldsymbol{\beta}}}
\newcommand{\bbeta}{{\boldsymbol{\beta}}}
\newcommand{\tbeta}{\tilde{\beta}}
\newcommand{\Indicator}{\mathds{1}}
\newcommand{\Expectation}{\mathbb{E}}
\newcommand{\Entropy}{h}
\newcommand{\cR}{\mathcal{R}}
\newcommand{\Variance}{\mathbb{V}\text{ar}}
\newcommand{\hH}{\hat{H}}
\newcommand{\Prob}{\mathbb{P}}
\newcommand{\ty}{\tilde{y}}
\newcommand{\hy}{\hat{y}}
\newcommand{\cY}{\mathcal{Y}}
\newcommand{\cF}{\mathcal{F}}
\newcommand{\ConditionalDensityNodeJ}{f_{\rvx_j}(  x_j | \rvx_{-j} = x_{-j} ; \bvthetaJStar)}
\newcommand{\ConditionalDensityNodeJY}{f_{\rvy_j}(  y_j | \rvx_{-j} = x_{-j} ; \bvthetaJStar)}
\newcommand{\sininv}{\sin^{-1}}
\newcommand{\taninv}{\tan^{-1}}
\newcommand{\rhoMax}{\rho_{\max}}
\newcommand{\upStar}{\upsilon^*}
\newcommand{\hup}{\hat{\upsilon}}
\newcommand{\bup}{\boldsymbol{\upsilon}}
\newcommand{\hbup}{\hat{\boldsymbol{\upsilon}}}
\newcommand{\bupStar}{\boldsymbol{\upsilon}^*}
\newcommand{\cP}{\mathcal{P}}
\newcommand{\Uniform}{\cU_{\cX_{0}}}
\newcommand{\hunu}{\hat{\boldsymbol{\nu}}}
\newcommand{\bnu}{\boldsymbol{\nu}}
\newcommand{\hnu}{\hat{\nu}}
\newcommand{\length}{t}
\newcommand{\bPsiJ}{\boldsymbol{\Psi}^{(j)}}
\newcommand{\hbPsiJ}{\hat{\boldsymbol{\Psi}}^{(j)}}
\newcommand{\bPsiJT}{\boldsymbol{\Psi}^{(j)^T}}
\newcommand{\hbPsiJT}{\hat{\boldsymbol{\Psi}}^{(j)^T}}
\newcommand{\cI}{\mathcal{I}}
\newcommand{\bzeta}{\boldsymbol{\zeta}}
\newcommand{\cJ}{\mathcal{J}}
\newcommand{\cM}{\mathcal{M}}
\newcommand{\cW}{\mathcal{W}}

\newcommand{\cH}{\mathcal{H}}
\newcommand{\cK}{\mathcal{K}}
\newcommand{\cB}{\mathcal{B}}
\newcommand{\cT}{\mathcal{T}}
\newcommand{\Naturals}{\mathbb{N}}
\newcommand{\Conductance}{\varphi}
\newcommand{\epsOne}{\epsilon_1}
\newcommand{\epsTwo}{\epsilon_2}

\newcommand{\epsFour}{\epsilon_4}
\newcommand{\epsFive}{\epsilon_5}
\newcommand{\epsSix}{\epsilon_6}
\newcommand{\epsSeven}{\epsilon_7}
\newcommand{\epsEight}{\epsilon_8}
\newcommand{\epsNine}{\epsilon_9}
\newcommand{\deltaOne}{\delta_1}
\newcommand{\deltaTwo}{\delta_2}

\newcommand{\deltaFour}{\delta_4}
\newcommand{\deltaFive}{\delta_5}

\newcommand{\deltaNine}{\delta_9}
\newcommand{\bcOne}{\bar{c}_1}
\newcommand{\bcTwo}{\bar{c}_2}
\newcommand{\bcThree}{\bar{c}_3}
\newcommand{\tepsilon}{\tilde{\epsilon}}
\newcommand{\tBoundedNoise}{\tilde{\BoundedNoise}}

\newcommand{\tsigma}{\tilde{\sigma}}
\newcommand{\tcOne}{\tilde{c}_1}
\newcommand{\tcTwo}{\tilde{c}_2}
\newcommand{\bareta}{\bar{\eta}}
\newcommand{\parameterSet}{\Lambda} % the set of weight vector associated with node i
\newcommand{\cHmax}{{\cH}_{\max}}

\begin{document}
	\sloppy
	\begin{frontmatter}
		\title{On Learning Continuous Pairwise Markov Random Fields}
		\runtitle{Continuous Markov Random Fields}
		
		\begin{aug}
			\author{\fnms{Abhin} \snm{Shah} ~~~\fnms{Devavrat} \snm{Shah} ~~~\fnms{Gregory W.} \snm{Wornell}\thanksref{ds.ack}} %\ead[label=e1]{devavrat@mit.edu}}
%			\runauthor{Shah}
			\affiliation{Dept. of EECS\\ Massachusetts Institute of Technology}
			\thankstext{ds.ack}{Email addresses of authors are \{abhin, devavrat, gww\}@mit.edu. }
		\end{aug}
		
		\begin{abstract}

		\end{abstract}

\begin{abstract}
	We consider learning a sparse pairwise Markov Random Field
(MRF) with continuous-valued variables from i.i.d samples. We
adapt the algorithm of Vuffray et al. (2019) \cite{VuffrayML2019} to this setting and
provide finite-sample analysis revealing sample complexity scaling
logarithmically with the number of variables, as in the
discrete and Gaussian settings. Our approach is applicable to
a large class of pairwise MRFs with continuous variables and
also has desirable asymptotic properties, including consistency and
normality under mild conditions.  
Further, we establish that the population version of the
optimization criterion employed in Vuffray et al. (2019) \cite{VuffrayML2019} can be interpreted as local maximum
likelihood estimation (MLE).  
As part of our analysis, we introduce a robust variation of
sparse linear regression \`a la Lasso, which may be of
interest in its own right. 
\end{abstract}
\end{frontmatter}

\section{Introduction}
\label{sec_intro}%LABEL
\subsection{Background}
\label{subsec_background}
%% Question of interest. 
Markov random fields or undirected graphical models are an important class of statistical models and represent the conditional dependencies of a high dimensional probability distribution with a graph structure. There has been considerable interest in learning discrete MRFs in machine learning, statistics, and physics communities under different names \cite{ChowL1968, AbbeelKN2006, NegahbanRWY2010, AckleyHS1985, SessakM2009}. Bresler (2015) \cite{Bresler2015} gave a simple greedy algorithm to learn arbitrary binary pairwise graphical models 
on $p$ nodes and maximum node degree $d$ with sample complexity $O(\exp(\exp(\Theta(d))) \log p)$ 
and runtime $\tcO(p^2)$.\footnote[1]{The $\tcO(\cdot)$ notation hides a factor $\text{poly}(\log p)$ as well as a constant (doubly-exponentially) depending on $d$.} This improved upon the prior work of Bresler et al. (2013) \cite{BreslerMS2013}, with runtime $\bcO(p^{d+2})$,\footnote[2]{The $\bcO(\cdot)$ notation hides a factor $\text{poly}(\log p)$ as well as a constant (exponentially) depending on $d$.} by removing the dependence of $d$ on the degree of the polynomial factor in runtime. Santhanam et al. (2012) \cite{SanthanamW2012} showed that only exponential dependence on $d$ is required in the sample complexity and thus, the doubly-exponential dependence on $d$ of Bresler (2015) \cite{Bresler2015} is provably suboptimal.

%@article{chow1968approximating,
%  title={Approximating discrete probability distributions with dependence trees},
%  author={Chow, C and Liu, Cong},
%  journal={IEEE transactions on Information Theory},
%  volume={14},
%  number={3},
%  pages={462--467},
%  year={1968},
%  publisher={IEEE}
%}

A recent work by Vuffray et al. (2019) \cite{VuffrayML2019} learns $t$-wise MRFs over general discrete alphabets in a  sample-efficient manner ($O(\exp(\Theta(d^{t-1})) \log p))$  with runtime $\bcO(p^t)$. The key to their proposal is a remarkable but seemingly mysterious objective function, the generalized interaction screening objective (GISO) which is an empirical average of an objective designed to screen an individual variable from its neighbors. 
While their approach can be formally extended to the continuous-valued setting, issues arise.
First, as is, their work shows that, for the discrete setting, the condition for learning is satisfied by only the `edge' parameters and their approach does not attempt to recover the `node' parameters.\footnote[3]{For the discrete setup, learning edge parameters is sufficient since, knowing those, node parameters can be recovered using the conditional expectation function; however, the same is not straightforward in the continuous setup.} Second, their condition for learning is cumbersome to verify as it is node-neighborhood-based and involves all the edges associated with the node. 

\begin{table*}
	\centering
	\caption[]{Comparison with existing works on pairwise continuous MRFs (beyond the Gaussian case) in terms of approach, conditions required and sample complexity: $p$ is \# of variables, $d$ is maximum node degree}
	\begin{tabular}{|p{1.6 cm}| p{3.1cm}|p{6.7cm}|p{2.2cm}|} 
		\hline
		Work  & Approach & Conditions & \#samples\\
		\hline
		\hline
		\multirow{3}{1.6cm}  {Yang et al. (2015) \cite{YangRAL2015}} & 
		\multirow{3}{3.1cm}{$\ell_1$ regularized node conditional log-likelihood} &
		1. Incoherence condition &  
		\multirow{6}{2.2cm}{$O(\text{poly}(d) \omega(p))$ s.t $\omega(p)$ $=$ $\bar{\omega}(p) \log p $ and $\bar{\omega}(p)$ is a density dependent function of $p$}\\
		& & 2. Dependency condition  & \\
		& & 3. Bounded moments of the variables&\\\cline{1-2}
		\multirow{3}{1.6cm}  {Tansey et al. (2015) \cite{TanseyPSR2015}} & 
		\multirow{3}{3.1cm}{Group lasso regularized node conditional log-likelihood} & 
		4. Local smoothness of the log-partition function & \\
		&& 5. Conditional distribution lies in exponential family& \\
		&&&\\
		\hline
		\multirow{4}{1.6cm}{ Yang et al. (2018)  \cite{YangNL2018}} & 
		\multirow{4}{3.1cm}{Node conditional pseudo-likelihood regularized by a nonconvex penalty} &
		1. Sparse eigenvalue condition &  
		\multirow{4}{2.2cm}{$O(\text{poly}(d) \log p)$}\\
		& & 2. Bounded moments of the variables&\\
		& & 3. Local smoothness of the log-partition function &\\
		&& 4. Conditional distribution lies in exponential family & \\
		\hline
		\multirow{3}{1.6cm} {Sun et al. (2015) \cite{SunKX2015}} & 
		\multirow{3}{3.1cm}{Penalized score matching objective} &
		1. Incoherence condition &  
		\multirow{3}{2.2cm}{$O(\text{poly}(pd))$}\\
		& & 2. Dependency condition&\\
		%						& & 3. Sufficient statistics lie in a RKHS &\\
		& & 3. Certain structural conditions &\\
		%						&& 4. Joint density is infinite-dimensional exponential family &\\
		%						\hline
		\hline
		\multirow{5}{1.6cm}  {Suggala et al. (2017) \cite{SuggalaKR2017}} & 
		\multirow{5}{3.1cm}{$\ell_1$ regularized node conditional log-likelihood} &
		1. Restricted strong convexity &  
		\multirow{5}{2.2cm}{$O(\text{poly}(d) \log p)$}\\
		& & 2. Assumptions on gradient of the population loss&\\
		& & 3. Bounded domain of the variables&\\
		%						& & 3. Identifiability conditions on the sufficient statistics &\\
		& & 4. Non-negative node parameters&\\
		&& 5. Conditional distribution lies in exponential family &\\
		%						& & 5. Scalar-valued non-linear edge-wise sufficient statistics &\\
		%						\hline
		\hline
		\multirow{3}{1.6cm} {Yuan et al. (2016) \cite{YuanLZLL2016}} & 
		\multirow{3}{3.1cm}{$\ell_{2,1}$ regularized node conditional log-likelihood} &
		1. Restricted strong convexity &  
		\multirow{3}{2.2cm}{$O(\text{poly}(d) \log p)$}\\
		& & 2. Bounded moment-generating function of variables&\\
		&&&\\
		\hline
		\multirow{2}{1.6cm}{{\color{blue} This work}} & 
		\multirow{2}{3.1cm}{{\color{blue} Augmented GISO (Section \ref{sec:algorithm})}} &
		1. Bounded domain of the variables &  
		\multirow{2}{2.2cm}{{\color{blue}$O(\exp{(d)} \log p)$ (Thm. \ref{theorem:structure}-\ref{theorem:parameter})}}\\
		& & 2. Conditional distribution lies in exponential family&\\
		\hline
	\end{tabular}
	\label{table:comp_exp_family}
\end{table*}

In this work, we consider the problem of learning sparse pairwise MRFs from i.i.d. samples when the underlying random variables are continuous.
%consider a general sub-class of graphical models where the node-wise con- ditional distributions arise from exponential families
%conditional distribution of each node conditioned on all the other nodes has an exponential family form
The classical Gaussian graphical model is an example of this. There has been a long history of learning Gaussian MRFs, e.g. Graphical Lasso \cite{FriedmanHT2008} and associated recent developments e.g. \cite{MisraVL2017, KelnerKMM2019}. Despite this, 
%Another example is the following extension of the Ising model to the continuous case i.e.,
%\begin{align}
% \Density \propto \exp\bigg(  \sum_{i \in [p]} \thetaI x_i +  \sum_{ i \neq j \in [p]}  \thetaIJ x_i x_j  \bigg)  \label{eq:continuous-ising-density}.
%\end{align}
%where $\rvbx = (\rvx_1 , \cdots, \rvx_p)$ is a $p$-dimensional vector of continuous variables, $\svbx = (x_1, \cdots, x_p)$ is a realization of $\rvbx$, and $\thetaI$ $\forall i \in [p],  \thetaIJ$ $\forall i \neq j \in [p]$ are the parameters associated with the distribution. 
the overall progress for the generic continuous setting (including \eqref{eq:continuous-ising-density}) has been limited.  In particular, the existing works for efficient learning require somewhat abstract, 
involved conditions that are hard to verify for e.g. incoherence \cite{YangRAL2015, TanseyPSR2015, SunKX2015}, dependency \cite{YangRAL2015, TanseyPSR2015, SunKX2015}, sparse eigenvalue \cite{YangNL2018}, restricted strong convexity ~\cite{YuanLZLL2016, SuggalaKR2017}. The incoherence condition ensures that irrelevant variables do not exert an overly strong effect on the true neighboring variables, the dependency condition ensures that variables do not become overly dependent, the sparse eigenvalue condition and the restricted strong convexity imposes strong curvature condition on the objective function. 
%Our work does not require any such conditions and unlike most of the previous works on pairwise continuous MRFs, our method is applicable to a large class of distributions beyond \eqref{eq:continuous-ising-density}. 
Table \ref{table:comp_exp_family} compares with the previous works on pairwise continuous 
MRFs with distribution of the form \eqref{eq:continuous-ising-density}. 

In summary, the key challenge that remains for continuous pairwise MRFs is finding a learning algorithm requiring (a) numbers of samples scaling as $\exp(\Theta(d))$ (in accordance with lower bound of Santhanam et al. (2012) \cite{SanthanamW2012}) and $\log p$, (b) computation scaling as $O(p^2)$, and (c) the underlying distribution
to satisfy as few conditions as in the discrete setting. 

%The effect of the aforementioned stringent conditions on the sample complexity is polynomial scaling in terms of the maximum node degree $d$ i.e., $\text{poly}(d)$. We believe that the appropriate scaling for the continuous setup (without these stringent conditions) should be $\text{exp}(d)$ instead of $\text{poly}(d)$ analogous to the result of \cite{SanthanamW2012} for the binary setting. This is also evident in the existing works on discrete MRFs that do not assume these conditions (See Table \ref{table:1}). Further, our method has desirable asymptotic properties (consistency and normality) unlike most of the existing works on discrete or continuous MRFs.
\begin{table*}
	\centering
	\caption{Comparison with prior works on discrete MRFs in terms of asymptotic properties (consistency and normality), computational and sample complexities: $p$ is \# of variables, $d$ is maximum node degree.} 
	\begin{tabular}{llllll}
		\hline
		Result (pairwise)  &  Alphabet  & Consistency &  Normality &\#computations&  \#samples  \\
		&  ~~  & (i.e. SLLN) &  (i.e. CLT) &~~&  ~~  \\
		\hline
		\hline
		%\textbf{Bresler $et$ $al$}  
		Bresler et al. (2013) \cite{BreslerMS2013} &  Discrete & $\checkmark$ & $\times$ & $\bcO(p^{d+2})$& $O(\exp(d)\log p)$ \\
		\hline
		%\textbf{Bresler} 
		Bresler (2015) \cite{Bresler2015}  &  Binary & $\checkmark$ & $\times$ & $\tcO(p^2)$  & $O(\exp(\exp(d))\log p)$ \\
		\hline
		%\textbf{Klivans $et$ $al$} 
		Klivans et al. (2017) \cite{KlivansM2017}&  Discrete & $\checkmark$ & $\times$ & $\bcO(p^2)$ & $O(\exp(d)\log p)$ \\
		\hline
		%\textbf{Vuffray $et$ $al$} 
		Vuffray et al. (2019) \cite{VuffrayML2019} &  Discrete & $\checkmark$ & $\times$ & $\bcO(p^2)$& $O(\exp(d)\log p)$ \\
		\hline
		%		%\textbf{Kelner $et$ $al$} 
		%		\cite{KelnerKMM2019} &  Gaussian & $\checkmark$ & $\times$ & $O(p^{d+1})$ & $O(d\log p )$\\
		%		\hline
		{\color{blue}
			{This Work}}&  {\color{blue} Continuous} & {\color{blue} $\checkmark$  } & {\color{blue} $\checkmark$ }  & {\color{blue} $\bcO(p^{2} )$} & {\color{blue} $O(\exp(d)\log p)$}  \\
		%		\midrule
	& ~~ & {\color{blue}  (Thm. \ref{theorem:GRISE-consistency-efficiency}) } & {\color{blue}  (Thm. \ref{theorem:GRISE-consistency-efficiency})}  & {\color{blue} (Thm. \ref{theorem:structure}-\ref{theorem:parameter})}  & {\color{blue} (Thm. \ref{theorem:structure}-\ref{theorem:parameter})} \\
		\hline
	\end{tabular}
	%consistency and efficiency of the estimate as well as sample complexity and computational complexity for different MRF recovery methods.}
	\label{table:1}
%	\vspace{-.1in}
\end{table*}

% contributions 
%we adapt the algorithm of \cite{VuffrayML2019} for discrete MRFs to the setting of continuous MRFs. For this computationally efficient algorithm,  we establish asymptotic consistency and normality under mild conditions on the MRF and provide desirable finite sample guarantees when the underlying distribution further satisfies simple, easy to verify conditions (examples in Section \ref{ssec:examples}). 
\subsection{Contributions} 
%\vspace{-3mm}
As the primary contribution of this work, we make progress towards the aforementioned challenge. Specifically, we provide desirable finite sample guarantees for learning continuous MRFs when the underlying distribution satisfies simple, easy to verify conditions (examples in Section \ref{ssec:examples}). We summarize our contributions in the following two categories.\\

\label{subsec_contributions}
{\bf Finite Sample Guarantees.} We provide rigorous finite sample analysis for learning structure and parameters of continuous MRFs without the abstract conditions common in literature (incoherence, dependency, sparse eigenvalue or restricted strong convexity). We require $\bcO(p^{2})$ computations and $O(\exp(d)\log p)$ samples, in-line with the prior works on discrete / Gaussian MRFs.
We formally extend the approach of Vuffray et al. (2019) \cite{VuffrayML2019} to the continuous setting to recover the `edge' parameters and propose a novel algorithm for learning `node' parameters through a robust variation of sparse linear regression (Lasso). Technically, this robust Lasso shows that even in the presence of arbitrary bounded additive noise, the Lasso estimator is ‘prediction consistent’ under mild assumptions (see Appendix \ref{sec:robust lasso_appendix}). Further, we simplify the sufficient conditions for learning in Vuffray et al. (2019) \cite{VuffrayML2019} from node-neighborhood-based to edge-based (see Condition \ref{condition:1}). This is achieved through a novel argument that utilizes the structure	of the weighted independent set induced by the MRF (see within Appendix \ref{subsec:proof of proposition rsc_giso}). 
%Technically, this provides a generic method to obtain an explicit lower bound on the variance of linear functional of a basis associated with the distribution. 
We show that the new, easy-to-verify, sufficient condition is naturally satisfied by various settings including polynomial and harmonic sufficient statistics (see Section \ref{ssec:examples} for concrete examples). Thus, while most of the existing works focus on distributions of the form \eqref{eq:continuous-ising-density}, our method is applicable to a large class of distributions beyond that.\\

\textbf{Understanding GISO.} We establish that minimizing the population version of GISO of Vuffray et al. (2019) \cite{VuffrayML2019} is identical to minimizing an appropriate Kullback-Leibler (KL) divergence. This is true for MRFs with discrete as well as continuous-valued random variables. Using the equivalence of KL divergence and maximum likelihood, we can interpret minimizing the population version of GISO as ``local'' MLE. By observing that minimizing the GISO is equivalent to M-estimation, we obtain asymptotic consistency and normality for this method with mild conditions. Finally, we also draw connections between the GISO and the surrogate likelihood proposed by Jeon et al. (2006) \cite{JeonL2006} for log-density ANOVA model estimation (see Section \ref{subsec:penalized} and Appendix \ref{connections to the penalized surrogate likelihood}).

%\vspace{-.1in}
\subsection{Other related work} 
\label{subsec_related_work}
Having mentioned some of the relevant work for discrete MRFs, we briefly review a few other approaches. We then focus extensively on the literature pertaining to the continuous setting.
See table \ref{table:comp_exp_family} and \ref{table:1} for a succinct comparision with prior works in the pairwise setting for continuous MRFs and discrete MRFs respectively.\\

{\bf Discrete MRFs.} After Bresler (2015) \cite{Bresler2015} removed the dependence of maximum degree, $d$, from the polynomial factor in the runtime (with sub-optimal sample complexity), Vuffray et al. (2016) \cite{VuffrayMLC2016} achieved optimal sample complexity of $O(\exp(\Theta(d))\log p)$ for Ising models on $p$ nodes but with runtime $\bcO(p^4)$. Their work was the first to propose and analyze the interaction screening objective function. Hamilton et al. (2017) \cite{HamiltonKM2017} generalized the approach of Bresler (2015) \cite{Bresler2015} for $t$-wise MRFs over general discrete alphabets but had non-optimal double-exponential dependence on $d^{t-1}$. Klivans et al. (2017) \cite{KlivansM2017} provided a multiplicative weight update algorithm (called the Sparsitron) for learning pairwise models over general discrete alphabets in time $\bcO(p^2)$ with optimal sample complexity ($O(\exp(\Theta(d))\log p)$) and $t$-wise MRFs over binary alphabets in time $\bcO(p^t)$ with optimal sample complexity ($O(\exp(\Theta(d^{t-1}))\log p)$). Wu et al. (2018) \cite{WuSD2018} considered an $\ell_{2,1}$-constrained logistic regression and improved the sample complexity of Klivans et al. (2017) \cite{KlivansM2017} for pairwise models over general discrete alphabets in terms of dependence on alphabet size. \\

{\bf Gaussian MRFs.} The problem of learning Gaussian MRFs is closely related to the problem of learning the sparsity pattern of the precision matrix of the underlying Gaussian distribution. Consider Gaussian MRFs on $p$ nodes of maximum degree $d$ and the minimum normalized edge strength $\bkappa$ (see Misra et al. (2017) \cite{MisraVL2017}). A popular approach, the Graphical Lasso \cite{FriedmanHT2008}, recovers the sparsity pattern under the restricted eigenvalue and incoherence assumptions from $O((d^2 + \bkappa^{-2})\log p)$ samples \cite{RavikumarWRY2011} by $\ell_1$ regularized log-likelihood estimator. The minimum required sample complexity was shown to be $O(\log p /\bkappa^2)$ by Wang et al. (2010) \cite{WangWR2010} via an information-theoretic lower bound. Misra et al. (2017) \cite{MisraVL2017}  provided a multi-stage algorithm that learns the Gaussian MRFs with  $O(d\log p /\bkappa^2)$ samples and takes time $O(p^{2d+1})$. A recent work by Kelner et al. (2019) \cite{KelnerKMM2019} proposes an algorithm with runtime $O(p^{d+1})$ that learns the sparsity pattern in $O(d \log p /\bkappa^2)$ samples. However, when the variables are positively associated, this algorithm achieves the optimal sample complexity of $O(\log p /\bkappa^2)$.\\

{\bf Continuous MRFs.} Realizing that the normality assumption is restrictive, some researchers have recently proposed extensions to Gaussian MRFs that either learns transformations of the variables or learn the sufficient statistics functions. The non-paranormal \cite{LiuLW2009} and the copula-based \cite{DobraL2011} methods assumed that a monotone transformation Gaussianize the data. Rank-based estimators in \cite{XueZ2012} and \cite{LiuHYLW2012} used non-parametric approximations to the correlation matrix and then fit a Gaussian MRF. 

There have been some recent works on learning exponential family MRFs for the pairwise setting. The subclass where the node-conditional distributions arise from exponential families was looked at by Yang et al. (2015) \cite{YangRAL2015}  and the necessary conditions for consistent joint distribution were derived. However, they consider only linear sufficient statistics and they need incoherence and dependency conditions similar to the discrete setting analyzed in \cite{WainwrightRL2006, JalaliRVS2011}. Yang et al. (2018) \cite{YangNL2018} study the subclass with linear sufficient statistics for edge-wise functions and non-parametric node-wise functions with the requirement of sparse eigenvalue conditions on their loss function. Tansey et al. (2015) \cite{TanseyPSR2015} extend the approach in Yang et al. (2015) \cite{YangRAL2015} to vector-space MRFs and non-linear sufficient statistics but still need the incoherence and dependency conditions similar to \cite{WainwrightRL2006, JalaliRVS2011}. Sun et al. (2015) \cite{SunKX2015} investigate infinite dimensional exponential family graphical models based on score matching loss. They assume that the node and edge potentials lie in a reproducing kernel Hilbert space and need incoherence and dependency conditions similar to \cite{WainwrightRL2006, JalaliRVS2011}. Yuan et al. (2016) \cite{YuanLZLL2016} explore the subclass where the node-wise and edge-wise statistics are linear combinations of two sets of pre-fixed basis functions. They propose two maximum likelihood estimators under the restricted strong convexity assumption. Suggala et al. (2017) \cite{SuggalaKR2017} consider a semi-parametric version of the subclass where the node-conditional distributions arise from exponential families. However, they require restricted strong convexity and hard to verify assumptions on the gradient of the population loss.

%{\bf Useful notations.} 
\subsection{Useful notations}
For any positive integer $n$, let $[n] \coloneqq \{1,\cdots, n\}$.
For a deterministic sequence $v_1, \cdots , v_n$, we let $\svbv \coloneqq (v_1, \cdots, v_n)$. 
For a random sequence $\rvv_1, \cdots , \rvv_n$, we let $\rvbv \coloneqq (\rvv_1, \cdots, \rvv_n)$. 
Let $\mathds{1}$ denote the indicator function. 
For a vector $\svbv \in \Reals^n$, we use $v_i$ to denote its $i^{th}$ coordinate and $v_{-i} \in \Reals^{n-1}$ to denote the vector after deleting the $i^{th}$ coordinate. 
We denote the $\ell_p$ norm $(p \geq 1)$ of a vector $\svbv \in \Reals^n$ by $\| \svbv\|_{p} \coloneqq (\sum_{i=1}^{n}|v_i|^p)^{1/p}$ and its $\ell_{\infty}$ norm by $\|\svbv\|_{\infty} \coloneqq \max_i |v_i|$. 
For a vector $\svbv \in \Reals^n$, we use $\| \svbv\|_{0} $ to denote the number of non-zero elements ($\ell_0$ norm) of $\svbv$. 
We denote the minimum of the absolute values of non-zero elements of a vector $\svbv \in \Reals^n$ by $\| \svbv\|_{\min_+} \coloneqq \min_{i : v_i \neq 0} |v_i|$.
For a matrix $\bV \in \Reals^{m \times n}$, we denote the element in $i^{th}$ row and $j^{th}$ column by $V_{ij}$ and the max norm by $\|\bV\|_{\max} \coloneqq \max_{ij} |V_{ij}|$. All logarithms are in base $e$.

\section{Problem Formulation}
\label{sec:problem setup} 
In this work, our interest is in the parametric pairwise Markov Random Fields with continuous variables.\\

%\subsection{Pairwise MRF}
%\label{subsec:pairwise mrf} 
{\bf Pairwise MRF.}
Let $\rvbx = (\rvx_1 , \cdots, \rvx_p)$ be a $p-$dimensional vector of continuous random variables such that each $\rvx_i$ takes value in a real interval $ \cX_i$ and let $\cX = \prod_{i=1}^p \cX_i$.  Let $\svbx = (x_1, \cdots, x_p) \in \cX$ be a realization of $\rvbx$. 
For any $i \in [p]$, let the length of the interval $\cX_i$ be upper (lower) bounded by a known constant $\Xupper$ ($\Xlower$). 
Consider an undirected graph $G = ([p], E)$ where the nodes correspond to the random variables in $\rvbx$, and $E$ denotes the edge set. 
The MRF corresponding to the graph $G$ is the family of distributions that satisfy the global Markov property with respect to $G$. 
According to the Hammersley-Clifford theorem \cite{Hammersley1971}, any strictly positive distribution 
factorizes with respect to its cliques. Here, we consider the setting where the functions associated with cliques are non-trivial only for the nodes and the edges. This leads to the pairwise MRFs with respect to graph $G$ with density as follows: with node potentials $\NodePotential : \cX_i \rightarrow \Reals$, edge potentials 
$\EdgePotential : \cX_i \times \cX_j \rightarrow \Reals$,
%\vspace{-1mm}
\begin{align}
\Density \propto \exp\Big(  \sum_{i \in [p]} \NodePotential(x_i) +  \sum_{(i,j) \in E} \EdgePotential(x_i,x_j)  \Big). \label{eq:pairwise-density}
\end{align}
%\vspace{-1mm}
%where $\NodePotential : \cX_i \rightarrow \Reals$ are the node potentials and $\EdgePotential : \cX_i \times \cX_j \rightarrow \Reals$ are the edge potentials.
%\subsection{Parametric form}
%\label{subsec:parametric form} 

{\bf Parametric Form.} We consider potentials in parametric form. Specifically, let 
%In this work, we consider a parametric form of the potentials and wish to learn the associated parameters. More specifically, we let
\begin{align}
\NodePotential(\cdot) = \bthetaIT \bphi(\cdot) \qquad \text{and} \qquad \EdgePotential(\cdot, \cdot) = \bthetaIJT \bpsi(\cdot, \cdot)
\end{align}
where $\bthetaI \in \Reals^k$ is the vector of parameters associated with the node $i$, $\bthetaIJ \in \Reals^{k^2}$ is the vector of parameters associated with the edge $(i,j)$, the map $\bphi : \Reals \rightarrow \Reals^k $ is a basis of the vector space of node potentials, and the map $\bpsi : \Reals^2 \rightarrow \Reals^{k^2} $ is a basis of the vector space of edge potentials. 
We assume that the basis $\bpsi(x , y)$ can be written as the Kronecker product of $\bphi(x)$ and $\bphi(y)$ i.e., $\bpsi(x , y) = \bphi(x) \otimes \bphi(y)$. 
This is equivalent to the function space assumption common in the literature \cite{YangRAL2015, YangNL2018, SuggalaKR2017, TanseyPSR2015} that the conditional distribution of each node conditioned on all the other nodes has an exponential family form (see \cite{YangRAL2015} for details). Further, let the basis functions be such that the resulting exponential family is minimal. 

A few examples of basis functions in-line with these assumptions are: (1) Polynomial basis with 
$\bphi(x) = (x^r : r \in [k])$, $\bpsi(x , y) = (x^ry^s : r,s \in [k])$; (2) Harmonic basis with 
$\bphi(x) = (\sin(rx); \cos(rx): r \in [k])$, $\bpsi(x , y) = (\sin(rx+sy) ; \cos(rx+sy) : r,s \in [k])$.\footnote[4]{$\bpsi(x , y)$ can be written as $\bphi(x) \otimes \bphi(y)$ using the sum formulae for sine and cosine.}

For any $r \in [k]$, let $\phi_r(x)$ denote the $r^{th}$ element of $\bphi(x)$ and let $\thetaIr$ be the corresponding element of $\bthetaI$.
For any $r,s \in [k]$, let $\psi_{rs}(x,y)$ denote that element of $\bpsi(x , y)$ which is the product of $\phi_r(x)$ and $\phi_s(y)$ i.e., $\psi_{rs}(x,y) = \phi_r(x)  \phi_s(y) $. 
Let $\thetaIJrs$ be element of $\bthetaIJ$ corresponding to $\psi_{rs}(x,y)$. 
We also assume that $\forall r \in [k], \forall x \in \cup_{i \in [p]} \cX_i$, $|\phi_r(x)| \leq \phiMax $ and $|d\phi_r(x)/dx| \leq \hphiMax$. 
Summarizing, the distribution of focus is
\begin{align}
\DensityParametrized \propto \exp\bigg(  \sum_{i \in [p]} \bthetaIT \bphi(x_i) + \hspace{-2mm} \sum_{\substack{ i \in [p], j > i}} \hspace{-2mm} \bthetaIJT \bpsi(x_i,x_j)  \bigg) \label{eq:pairwise-parametric-density}
\end{align}
where $\btheta \coloneqq \big(\bthetaI \in \Reals^k : i \in [p] ; \bthetaIJ \in \Reals^{k^2} : i \in [p], j > i \big) \in \Reals^{kp + \frac{k^2p(p-1)}{2}}$ is the parameter vector associated with the distribution. For any $i \in [p], i > j$, define $\bthetaIJ = \bthetaJI$ i.e., both $\bthetaIJ$ and $\bthetaJI$ denote the parameter vector associated with the edge $(i,j)$. 

Let the true parameter vector and the true distribution of interest be denoted by $\bthetaStar$ and $\DensityParametrizedTrue$ respectively. We assume a known upper (lower) bound on the maximum (minimum) absolute value of all non-zero parameter in $\bthetaStar$, i.e., $\|\bthetaStar\|_{\infty} \leq \thetaMax, \|\bthetaStar\|_{\min_+} \geq \thetaMin$.
%Let $\parameterSet$ denote the set of all possible weight vectors satisfying  our assumption i.e., $\parameterSet \coloneqq \{ \tbtheta \in  \Reals^{kp + \frac{k^2p(p-1)}{2}}: \|\tbtheta\|_{\min_+} > \thetaMin, \|\tbtheta\|_{\infty} < \thetaMax\}$.

Suppose we are given additional structure. Define 
\begin{align}
E(\bthetaStar) & = \{ (i,j): i < j \in [p], \| \bthetaIJStar \|_0 > 0 \}. \label{eq:edgeSet}
\end{align}
Consider the graph $G(\bthetaStar) = ([p], E(\bthetaStar))$ such that $\DensityParametrizedTrue$ is Markov with respect to $G(\bthetaStar) $. 
Let the max-degree of any node of $G(\bthetaStar)$ be at-most $d$. 
For any node $i \in [p]$, let the neighborhood of node $i$ be denoted as $\cN(i) = \{ j : (i,j) \in E(\bthetaStar)\} \cup \{ j : (j,i) \in E(\bthetaStar)\}$.

The learning tasks of interest are as follows:
\begin{problem}
	(Structure Recovery). Given $n$ independent samples of $\rvbx$ i.e., $\svbx^{(1)} \cdots , \svbx^{(n)}$ obtained from $\DensityParametrizedTrue$, produce a graph $\hG$, such that $\hG = G(\bthetaStar)$.
\end{problem}
\begin{problem}
	(Parameter Recovery). Given $n$ independent samples of $\rvbx$ i.e., $\svbx^{(1)} \cdots , \svbx^{(n)}$ obtained from $\DensityParametrizedTrue$  and $\alpha >0$, compute an estimate $\hbtheta$ of $\bthetaStar$ such that
	\begin{align}
	\|\bthetaStar - \hbtheta\|_{\infty} \leq \alpha.
	\end{align}
\end{problem}
%We emphasize that learning pairwise MRFs (both \textit{parameter recovery} and \textit{structure recovery}) is equivalent to learning the true weight vector $\btheta$. 
%And ideally, we want to achieve \textit{parameter recovery} and \textit{structure recovery} with number of samples scaling logarithmically in $p$ and number of computations scaling polynomially in $p$. 
%In the next section, we show that this is possible with some additional structure on the distribution in \eqref{eq:pairwise-parametric-density}. 
%We will require knowing  $d, k$, $\thetaMax$ and $\thetaMin$ for both \textit{parameter recovery} and \textit{structure recovery}.

%\subsection{Notations and definitions}
%\label{subsec:notations and definitions} 
{\bf Additional Notations.}
For every node $i \in [p]$, define $\bvthetaIStar \coloneqq \big(\bthetaIStar \in \Reals^{k} ; \bthetaIJStar \in \Reals^{k^2}: j \in [p], j \neq i \big) \in \Reals^{k + k^2(p-1)}$ to be the weight vector associated with node $i$ that consists of all the true parameters involving node $i$. 
Define $\parameterSet = \{ \bvtheta\in  \Reals^{k + k^2(p-1)}: \|\bvtheta\|_{\min_+} \geq \thetaMin, \|\bvtheta\|_{\infty} \leq \thetaMax\}$. Then
under our formulation, $\bvthetaIStar \in \parameterSet$ for any $i \in [p]$.
Define $\bvthetaIEdgeStar \coloneqq \big(\bthetaIJStar \in \Reals^{k^2}: j \in [p], j \neq i \big) \in \Reals^{k^2(p-1)}$ to be the component of $\bvthetaIStar$ associated with the edge parameters.
\begin{definition}
	(Locally centered basis functions). For $i \in [p], j \in [p] \backslash \{i\}$, define locally centered basis functions as follows: for $x \in \cX_i$, $x' \in \cX_j$
	\begin{align}
	\bphiI(x) & \coloneqq \bphi(x) - \int_{y \in \cX_i} \bphi(y) dy, \label{eq:centeredBasisFunctions1}\\
	\bpsiI(x, x') & \coloneqq \bpsi(x, x') - \int_{y \in \cX_i} \bpsi(y,x') dy.
	~ \label{eq:centeredBasisFunctions2}
	\end{align}
	%	These locally centered basis functions integrate to zero with respect to the given variable $x_i$, i.e., $\int_{x_i \in \cX_i} \tphi(x_i) dx_i = \int_{x_i \in \cX_i} \tpsi_{X_i}(x_i , x_j) dx_i = 0 $.
\end{definition}
%\vspace{-2mm}
For any $i \in [p], j \in [p] \backslash \{i\}$, the locally centered basis functions $\bphiI(\cdot)$ and $\bpsiI(\cdot, \cdot)$ integrate to zero with respect to the uniform density on $\rvx_i$. This is motivated by the connection of the GISO to the penalized surrogate likelihood (See Appendix \ref{connections to the penalized surrogate likelihood}).

Define 
$\bvphiI(\rvbx) \coloneqq \big(\bphiI(\rvx_i) \in \Reals^{k} ; \bpsiI(\rvx_i,\rvx_j)  \in \Reals^{k^2}: j \in [p], j \neq i \big) \in  \Reals^{k+k^2(p-1)}$ 
to be the vector of all locally centered basis functions involving node $i$. We may also utilize notation $\bvphiI(\rvbx) = \bvphiI(\rvx_i; \rvx_{-i})$. Similary, we define $\bvphiI(\svbx)$ when $\rvbx = \svbx$. Define
\begin{align}\label{eq:defs}
\g & \coloneqq \thetaMax(k+k^2d),\\
\varphiMax & \coloneqq  (1+\Xupper)  \max\{\phiMax,\phiMax^2\}.
\end{align}
Let $\qs \coloneqq \qs(k, \Xlower, \Xupper, \thetaMax, \thetaMin, \phiMax,\hphiMax)$ denote the smallest possible eigenvalue of the Fisher information matrix of any single-variable exponential family distribution with sufficient statistics $\bphi(\cdot)$, with length of the support upper (lower) bounded by $b_u (b_l)$ and with absolute value of all non-zero parameters bounded above (below) by $\thetaMax (\thetaMin)$.  Let
\begin{align}
%c_1(k,d,\Xupper,\thetaMax,\varphiMax,\kappa,\alpha) 
c_1(\alpha) & = \frac{2^4 \pi^2 e^2 (d+1)^2 \g^2 \varphiMax^{2}(1+ \g \varphiMax)^2  \exp(4\g \varphiMax)}{\kappa^2 \alpha^4}\\
%c_2(k,d,\Xupper,\thetaMax,\hphiMax,\qs,\alpha) 
c_2(\alpha)  & = \frac{2^{37d+73} \Xupper^{2d} k^{12d+16} d^{6d+9} \thetaMax^{6d+8}  \phiMax^{8d+12} \hphiMax^{2d}}{\alpha^{8d + 16} (\qs)^{4d+8}}
\end{align}
Observe that 
\begin{align}
c_1(\alpha) = O\Bigg(\frac{\exp(\Theta(k^2 d))}{\kappa^2 \alpha^4}\Bigg), c_2(\alpha) = O\Bigg( \Big(\frac{kd}{\alpha \qs}\Big)^{\Theta(d)}\Bigg).
\end{align}
%\textRed{$\qs$ is a function of only $k, \Xlower, \Xupper, \phiMax,$ and $\hphiMax$.}
Let $A(\bvthetaIStar)$ be the covariance matrix of $\bvphiI(\rvbx)\exp\big( -\bvthetaIStarT \bvphiI(\rvbx) \big)$ and $B(\bvthetaIStar)$ be the cross-covariance matrix of $\bvphiI(\rvbx)$ and $\bvphiI(\rvbx)\exp\big( -\bvthetaIStarT \bvphiI(\rvbx) \big)$, where $\rvbx$ is distributed as per $\DensityParametrizedTrue$.
\section{Algorithm}
\label{sec:algorithm} 

Our algorithm, `Augmented GISO' has two parts: First, it recovers graph structure, i.e. edges $E(\bthetaStar)$ and associated edge 
parameters, $\bthetaIJStar, i \neq j \in [p]$. This is achieved through the Generalized Regularized Interaction Screening Estimator (GRISE) of Vuffray et al. (2019) \cite{VuffrayML2019} 
by extending the definition of GISO for continuous variables in a straightforward manner. This, 
however, does not recover node parameters $\bthetaIStar, i \in [p]$. Second, we transform the problem 
of learning node parameters as solving a sparse linear regression. Subsequently, using a robust variation of the classical Lasso \cite{Tibshirani1996, Efron2004} and knowledge of the learned edge parameters, we recover node parameters. \\

%\vspace{-4mm}
%\subsection{Learning Edge Parameters using GISO}
%\label{subsec:giso for continuous variables}
{\bf Learning Edge Parameters.}
%In this section, we extend the GISO (introduced in \cite{VuffrayML2019}) to the case where the underlying variables are continuous. 
Given $\DensityParametrizedTrue$, for any $i \in [p]$, the conditional density of $\rvx_i$ reduces to 
	\begin{align}
\hspace{-3mm}	\ConditionalDensityNodeI & \propto \exp \Big( \bvthetaIStarT \bvphiI(x_i; x_{-i})  \Big). \label{eq:form1}
	\end{align}
This inspired an unusual local or node $i \in [p]$ specific objective GISO \cite{VuffrayML2019}. 
\begin{definition}[GISO] Given $n$ samples $\svbx^{(1)} \cdots , \svbx^{(n)}$ of $\rvbx$ and $i \in [p]$, the GISO
	maps $\bvtheta \in \Reals^{k+k^2(p-1)}$ to $\cS_{n}^{(i)}(\bvtheta) \in \mathbb{R}$ defined as 
	\begin{align}
	\cS_{n}^{(i)}(\bvtheta)  = \frac{1}{n} \sum_{t = 1}^{n} \exp\Big( -\bvtheta^T \bvphiI(\svbx^{(t)}) \Big). \label{eq:GISO-main}
	\end{align}
	%	where $\ux^{(1)} \cdots , \ux^{(n)}$ are $n$ independent samples of $\uX$ obtained from the joint distribution in \eqref{eq:pairwise-parametric-density}.
\end{definition}
Since the maximum node degree in $G(\bthetaStar)$ is $d$ and $\| \bthetaStar \|_{\infty} \leq \thetaMax$, we have 
$\|\bvthetaIStar\|_1 \leq \g = \thetaMax(k+k^2d)$ for any $i \in [p]$. 
The GRISE produces an estimate of $\bvthetaIStar$ for each 
$i \in [p]$ by solving a separate optimization problem as 
\begin{align}
\hbvthetaIn \in \argmin_{\bvtheta \in \parameterSet : \|\bvtheta\|_1 \leq \g} \cS_{n}^{(i)}(\bvtheta). \label{eq:GRISE}
\end{align}
For $\epsilon > 0$, $\hbvthetaIEps$ is an $\epsilon$-optimal solution of GRISE for $i \in [p]$ if
\begin{align}
\cS_{n}^{(i)}(\hbvthetaIEps) \leq \cS_{n}^{(i)}(\hbvthetaIn) + \epsilon. \label{eq:eps-opt-GRISE}
\end{align}
The \eqref{eq:GRISE} is a convex minimization problem and has an efficient implementation for finding an $\epsilon$-optimal solution. 
Appendix \ref{sec:the generalized interaction screening algorithm} describes such an implementation for completeness 
borrowing from Vuffray et al. (2019) \cite{VuffrayML2019}. 

Now, given such an $\epsilon$-optimal solution $\hbvthetaIEps$ for GRISE corresponding to $i \in [p]$, let 
$\hbvthetaIEpsEdge= (\hbthetaIJ, j \neq i, j \in [p])$ be its components corresponding to all possible $p-1$
edges associated with node $i$. Then, we declare $\hbvthetaIEpsEdge$ as the edge parameters associated with $i$ for each $i \in [p]$. These edge parameters can be used to recover the graph structure as shown in Theorem \ref{theorem:structure}.\\

%\subsection{Algorithm for recovering node-wise parameters}
%\label{subsec:algorithm for recovering node-wise parameters}
%
{\bf Learning Node Parameters.} As we shall argue in Theorems \ref{theorem:GRISE-KLD}-\ref{theorem:GRISE-consistency-efficiency}, for 
each $i \in [p]$, the exact solution of GRISE, $\hbvthetaIn$, is consistent, i.e. $\hbvthetaIn \stackrel{p}{\to} \bvthetaIStar$ in large sample limit, as well as 
it is normal, i.e. appropriately normalized $(\hbvthetaIn - \bvthetaIStar)$ obeys Central Limit Theorem in the large sample limit. While these
are remarkable {\em asymptotic} results, they do not provide {\em non-asymptotic} or finite sample error bounds. We will be able to provide finite sample error bounds for edge parameters learned from an $\epsilon$-optimal solution of GRISE, i.e. $\| \hbvthetaIEpsEdge- \bvthetaIEdgeStar\|_{\infty}$ is small. But to achieve the same for node parameters, we need additional processing. This is the purpose of the method described next. 

To that end, let us consider any $i \in [p]$. Given access to $\bvthetaIEdgeStar$ 
(precisely, access to $\hbvthetaIEpsEdge\approx \bvthetaIEdgeStar$), we wish to identify node 
parameters $\bthetaIStar = \big(\thetaIrStar: r \in [k]\big)$. Now the conditional density of 
$\rvx_i \in \cX_i$ when given $\rvx_{-i} = x_{-i} \in \prod_{j \neq i} \cX_j$, can be written as
%
%
%
%In this section, we introduce the algorithm that learns the parameters associated with the node potentials given estimates of the edge-wise parameters and an estimate of the graph structure. 
%The conditional density of node $X_i$ given the values taken by all the other nodes can also be written as
\begin{align}
\ConditionalDensityNodeI & \propto  \exp\Big(  \blambdaStarT (x_{-i}) \bphi(x_i)\Big)  \label{eq:form2}
\end{align}
where $\blambdaStar(x_{-i}) \coloneqq (\thetaIrStar+ \sum_{ j \neq i} \sum_{ s \in [k] }\thetaIJrsStar  \phi_s(x_j) : r \in [k])$ is the canonical parameter vector of the density in \eqref{eq:form2}. Let $\bmuStar(x_{-i}) = \mathbb{E}[\bphi(\rvx_i) | \rvx_{-i}  = x_{-i}] \in \mathbb{R}^k$. 

Now if we know $\blambdaStar(x_{-i})$, and since we know $\hbvthetaIEpsEdge\approx \bvthetaIEdgeStar$, we can recover 
$(\thetaIrStar: r \in [k])$. However, learning $\blambdaStar(x_{-i})$ from samples is not straightforward. By duality of exponential family, 
in principle, if we know $\bmuStar(x_{-i})$, we can recover $\blambdaStar(x_{-i})$. Now learning $\bmuStar(x_{-i})$ can be viewed as a traditional 
regression problem: features $Z = \rvx_{-i} $, label $Y = \bphi(\rvx_i)$, regression function $\mathbb{E}[Y | Z] = \bmuStar(\rvx_{-i} )$ and
indeed samples $\svbx^{(1)},\dots, \svbx^{(n)}$ of $\rvbx$ provides samples of $Y, Z$ as defined here. Therefore, in principle, we can learn
the regression function. As it turns out, the regression function $\bmuStar(\cdot): \mathbb{R}^{p-1} \to \mathbb{R}^k$ 
is Lipschitz and hence we can approximately {\em linearize} it leading to a sparse linear
regression problem. Therefore, by utilizing Lasso on appropriately linearized problem, we can (approximately) learn $\bmuStar(x_{-i})$,
which in turn leads to $\blambdaStar(x_{-i})$ and hence learning $(\thetaIrStar: r \in [k])$ as desired. This is summarized as a 
three-step procedure: 

Consider $x_{-i}^{(z)}$ where $z$ is chosen uniformly at random from $[n]$.
%\vspace{-2mm}
\begin{enumerate}
	\item Express learning $\bmuStar(\cdot)$ as a sparse linear regression problem (Details in Appendix \ref{subsec:learning conditional mean parameters as a sparse linear regression}). Use robust variation of Lasso (Details in Appendix \ref{sec:robust lasso_appendix}) to obtain an estimate ($\hbmu( x_{-i}^{(z)}) $) of $\bmuStar(x_{-i}^{(z)})$ (Details in Appendix \ref{subsec:learning conditional mean parameter vector}).
	% x_{-i} as one of the n samples at random
	% randomly chosen sample
%	\vspace{-2mm}
	\item Use $\hbmu( x_{-i}^{(z)}) $, and the conjugate duality between the canonical parameters and the mean parameters to learn an estimate ($\hblambda (x_{-i}^{(z)})$) of $\blambdaStar (x_{-i}^{(z)})$ (Details in Appendix \ref{subsec:learning canonical parameter vector}).
%	\vspace{-2mm}
	\item Use the estimates of the edge parameters i.e., $\hbvthetaIEpsEdge$ and $\hblambda (x_{-i}^{(z)})$ to learn an estimate ($\hbthetaI$) of the node parameters ($\bthetaIStar$) (Summarized in Appendix \ref{subsec:error bound on node-wise parameter estimation}).
\end{enumerate}
\section{Analysis and Main results}
\label{sec:main results}

\subsection{Understanding GRISE: ``Local'' MLE, M-estimation, Consistency, Normality}
\label{subsec:equivalence of giso and kl-divergence}

For a given $i \in [p]$, we establish a surprising connection between the population version of GRISE and 
Maximum Likelihood Estimate (MLE) for a specific parametric distribution in an exponential family 
which varies across $i$. That is, for each $i \in [p]$, GRISE is a ``local'' MLE at the population level.
Further, observing that minimzing the GISO is equivalent to M-estimation allows us to import asymptotic theory of M-estimation to establish consistency and normality of GRISE under mild conditions.

Consider $i \in [p]$. For any $\bvtheta \in \parameterSet$, the population version of GISO as defined in \eqref{eq:GISO-main} is given by
\begin{align}\label{eq:GISO-pop}
\cS^{(i)}(\bvtheta)  \coloneqq \Expectation  \bigg[\exp\Big( -\bvtheta^T \bvphiI(\rvbx) \Big) \bigg]. 
\end{align}
Consider the distribution over $\cX$ with density given by 
\begin{align}
\UniformProxyDensity \propto \DensityParametrizedTrue \times \exp \Big( -\bvthetaIStarT \bvphiI(\svbx) \Big)
\end{align}
Define a parametric distribution over $\cX$ parameterized by $\bvtheta \in \parameterSet$ with density given by
\begin{align}\label{eq:localCond}
\DifferenceDensity \propto \DensityParametrizedTrue \times \exp \Big( -\bvthetaT \bvphiI(\svbx) \Big)
\end{align}
The following result argues that the MLE for parametric class induced by \eqref{eq:localCond} coincides 
with the minimizer of the population version of GISO as defined in \eqref{eq:GISO-pop}. This provides an intuitively pleasing connection of the GISO in terms of the KL-divergence. Proof can be found in 
Appendix \ref{sec:proof of theorem-GRISE-KLD}.
\begin{theorem} \label{theorem:GRISE-KLD}
	Consider $i \in [p]$. Then, 
	with $\infdiv{\cdot}{\cdot}$ representing KL-divergence, 
	\begin{align}\label{eq:thm.1}
	\argmin_{\bvtheta \in \parameterSet: \|\bvtheta\|_1 \leq \g}  
	\infdiv{\UniformProxyDensityfun}{\DifferenceDensityfun} %\\
	& = \argmin_{\bvtheta \in \parameterSet: \|\bvtheta\|_1 \leq \g} \cS^{(i)}(\bvtheta). 
	\end{align}
	Further, the true parameter $\bvthetaIStar$ for $i \in [p]$ is a unique minimizer of $\cS^{(i)}(\bvtheta)$.
\end{theorem}
%See Appendix \ref{sec:proof of theorem-GRISE-KLD} for proof of Theorem \ref{theorem:GRISE-KLD}.\\
Even though at the population level, GRISE is equivalent to MLE for parametric class induced by \eqref{eq:localCond}, the link between the finite-sample GRISE and the finite-sample MLE is missing. However, observe that minimizing the finite-sample GISO as defined in \eqref{eq:GISO-main} is equivalent to M-estimation. This results in the following consistency and normality property of GRISE. Proof can be found in Appendix \ref{sec:proof of theorem:GRISE-consistency-efficiency}. 
\begin{theorem}\label{theorem:GRISE-consistency-efficiency}
	Given $i \in [p]$ and $n$ independent samples $\svbx^{(1)},\dots, \svbx^{(n)}$ of $\rvbx$, 
	let $\hbvthetaIn$ be a solution of \eqref{eq:GRISE}. Then, as $n\to \infty$, $\hbvthetaIn \stackrel{p}{\to} \bvthetaIStar$. Further, under the assumptions that $B(\bvthetaIStar)$ is invertible, and that none of the true parameter is equal to  the boundary values of $\thetaMax$ or $\thetaMin$, we have $\sqrt{n} ( \hbvthetaIn - \bvthetaIStar ) \stackrel{d}{\to} {\cal N}({\bf 0},B(\bvthetaIStar)^{-1} A(\bvthetaIStar) B(\bvthetaIStar)^{-1})$ where ${\cal N}(\bm{\mu}, \bm{\Sigma})$ represents multi-variate Gaussian with mean $\bm{\mu}$ and covariance $\bm{\Sigma}$.
\end{theorem}
We emphasize that $B(\bvthetaIStar)^{-1} A(\bvthetaIStar) B(\bvthetaIStar)^{-1}$ need not be equal to the inverse of the corresponding Fisher information matrix. Thus, $\hbvthetaIn$ is asymptotically only normal and not efficient. See Appendix \ref{sec: discussion on theorem} for more details on this and on invertibility of $B(\bvthetaIStar)$.
%{\bf GISO: Special Instance of Penalized Surrogate Likelihood}
%Consider nonparametric density estimation where densities are of the form $f_{\uX}(\ux) = e^{\eta(\ux)} / \int e^{\eta(\ux)} d\ux$ from i.i.d samples $\ux^{(1)}, \cdots, \ux^{(n)}$. 
%To circumvent the computational limitation of the exact likelihood based functionals, \cite{JeonL2006} proposed to minimize penalized surrogate likelihood as follows:
%\begin{align}
%\cl_n(\eta)  = \frac{1}{n} \sum_{t = 1}^{n} \exp\Big( -\eta(\ux^{(t)}) \Big) + \int_{\ux} \rho(\ux) \times  \eta(\ux)d\ux \label{eq:penalized surrogate likelihood}
%\end{align}
%where $\rho(\cdot)$ is some known probability density function.
%The following proposition shows that the GISO is a special case of the penalized surrogate likelihood.
%\begin{proposition} \label{prop:GISO_penalized_likelihood_connection}
%	For any $i \in [p]$, the GISO is equivalent to the penalized surrogate likelihood associated with the conditional density of $X_i$ when $\rho(\cdot)$ is the uniform density on $\cX_i$.
%\end{proposition}

\subsection{Finite Sample Guarantees}
\label{subsec:learning continuous mrfs}

While Theorem \ref{theorem:GRISE-consistency-efficiency} talks about asymptotic consistency and normality, it does not provide finite-sample error bounds. In this section, we provide the finite-sample error bounds which require the following additional condition. 
%In this section, we provide finite sample analysis for learning continuous MRFs. 
%We first describe a condition on the distribution in \eqref{eq:pairwise-parametric-density} that is sufficient to ensure that \textit{parameter recovery} and \textit{structure recovery} can be achieved with number of samples (computations) scaling logarithmically (polynomially) in $p$.
\begin{condition} \label{condition:1}
	%(Lower bound on expected conditional entropy power). 
	Let $\cbtheta , \tbtheta \in \Reals^{kp + \frac{k^2p(p-1)}{2}}$ be feasible weight vectors associated with the distribution in \eqref{eq:pairwise-parametric-density} i.e., they have an upper (lower) bound on the maximum (minimum) absolute value of all non-zero parameters. There exists a constant $\kappa > 0$ such that for any $i \neq j \in [p]$
	\begin{align}
	\Expectation & \bigg[ \exp\bigg\{2\Entropy \bigg( (\cbthetaIJ - \tbthetaIJ)^{T} \bpsiI(\rvx_i,\rvx_j) \bigg| \rvx_{-j}  \bigg) \bigg\} \bigg] \\
	& \qquad \geq \kappa \|\cbthetaIJ - \tbthetaIJ\|_2^2.  \label{eq:condition}
	\end{align}
	Here $\Entropy(\cdot | \rvx_{-j})$ represents conditional differential entropy conditioned on $\rvx_{-j}$.  
\end{condition}
%%\textRed{\subsubsection{Discussion}
%%\label{subsubsec:discussion}
%%Suppose in the simplest of the cases when $k=1$, the distribution of the random variable $\tpsi_{X_i}(X_i,X_j) $ is uniform on $\cX_j$. Then, the LHS of \eqref{eq:condition} is at least $\Xlower^2 \times (\ctheta_{ij} - \ttheta_{ij})^2$. This supports the quadratic behavior on the RHS of \eqref{eq:condition}.}
%%\subsubsection{Guarantees}
%%\label{subsubsec:guarantees}
Under condition \ref{condition:1}, we obtain the following structural recovery result whose proof is in Appendix  
\ref{sec:proof of theorem:structure}.
%The following theorem shows that structure recovery can be achieved with number of samples scaling logarithmically in $p$.
\begin{theorem} \label{theorem:structure}
	Let Condition \ref{condition:1} be satisfied. Given $n$ independent samples $\svbx^{(1)},\dots, \svbx^{(n)}$ of $\rvbx$, for each $i \in [p]$, 
	let $\hbvthetaIEps$ be an $\epsilon$-optimal solution of \eqref{eq:GRISE} and $\hbvthetaIEpsEdge$ be the associated edge parameters. Let 
	\begin{align}
	\hat{E} & \hspace{-0.25mm} = \hspace{-0.25mm} \bigg\{ (i,j): i < j \in [p], \Big( \hspace{-1.5mm} \sum_{ r,s \in [k] } \hspace{-1.5mm}  \Indicator \{| \hthetaIJrs | > \thetaMin/3 \} \Big) \hspace{-0.5mm}  > \hspace{-0.5mm}  0 \bigg\}.
%	\vspace{-5mm}
	\end{align}
	Let $\hG = ([p], \hat{E})$. Then for any $\delta \in (0,1)$, $G(\bthetaStar) = \hG$ with probability at least $1 - \delta$ as long as 
	\begin{align}
	n & \geq  c_1\Big(\frac{\thetaMin}{3}\Big)  \log\bigg(\hspace{-0.25mm}\frac{2pk}{\sqrt{\delta}}\hspace{-0.25mm}\bigg) \hspace{-1mm}=\hspace{-0.5mm}
	\Omega\Bigg(\hspace{-0.5mm}\frac{\exp(\Theta(k^2 d))}{\kappa^2 } \log\bigg(\frac{pk}{\sqrt{\delta}}\bigg) \hspace{-0.5mm}\Bigg).	
%	\vspace{-5mm}
	\end{align}
	%	Let $\hbvthetaIEpsEdge$ $\forall i \in [p]$ be estimates of the edge-wise parameters from Section \ref{subsec:giso for continuous variables}. 
	%	Consider the graph $\hG = ([p], \hat{E})$ where $\hat{E}$ is defined as follows:
	%	\begin{align}
	%		\hat{E} = \bigg\{(i,j) \in [p] \times [p] \setminus \{i\} : \Big( \sum_{ r,s \in [k] } \Indicator\{| \hTheta_{ij}^{(r,s)} | > \thetaMin/3 \} \Big) > 0 \bigg\}
	%	\end{align}
	%	Then, with probability at least $1 - \delta $, we have $\hG = G(\btheta)$. 
	The number of computations required scale as $\bcO(p^{2})$.
\end{theorem}
%See Appendix \ref{sec:proof of theorem:structure} for proof of Theorem \ref{theorem:structure}.\\
Now we state our result about parameter recovery whose proof can be found in Appendix \ref{sec:proof of theorem:parameter}.
\begin{theorem} \label{theorem:parameter}
	Let Condition \ref{condition:1} be satisfied. Given $n$ independent samples $\svbx^{(1)},\dots, \svbx^{(n)}$ of $\rvbx$, for each $i \in [p]$, 
	let $\hbvthetaIEps$ be an $\epsilon$-optimal solution of \eqref{eq:GRISE} and $\hbvthetaIEpsEdge \in \Reals^{k^2(p-1)}$ be the associated edge parameters.
	Let $\hbthetaI \in \Reals^k, i \in [p]$ be estimates of node parameters obtained through the three-step procedure involving robust Lasso. Let 
	$\hbtheta = (\hbthetaI; \hspace{1mm} \hbvthetaIEpsEdge : i \in [p]) \in \Reals^{kp + \frac{k^2p(p-1)}{2}}$ be their appropriate concatenation. Then, for any $\alpha \in (0,1)$ 
	\begin{align}
	\| \hbtheta - \bthetaStar \|_{\infty} & \leq \alpha, 
	\end{align}
	with probability at least $1 - \alpha^4$ as long as 
	\begin{align}
	n & \geq \hspace{-0.5mm} \max \hspace{-0.75mm}\bigg[ \hspace{-0.5mm} c_1 \hspace{-0.75mm}\bigg(\hspace{-1.5mm} \min\hspace{-0.75mm}\bigg\{ \hspace{-0.75mm} \frac{\thetaMin}{3}\hspace{-0.5mm}, \hspace{-0.5mm} \alpha, \hspace{-0.5mm} \frac{\alpha}{2^{\frac54} d k \phiMax } \hspace{-0.75mm} \bigg\} \hspace{-1mm}\bigg) \hspace{-0.5mm} \log \hspace{-0.75mm} \bigg( \hspace{-0.75mm} \frac{2^{\frac52} p k }{\alpha^2} \hspace{-0.75mm} \bigg)\hspace{-0.5mm}, \hspace{-0.5mm}  c_2\Big(\hspace{-0.5mm}\frac{\alpha}{2^{\frac14}}\hspace{-0.5mm}\Big) \hspace{-0.5mm} \bigg] \\
	& =~ \Omega\Bigg( \frac{\exp\bigg(\Theta\Big(k^2 d + d \log \big(\frac{dk}{\alpha q^*}\big)\Big)\bigg)}{\kappa^2 \alpha^4}   \times \log \bigg(\frac{pk}{\alpha^2} \bigg) \Bigg).
	\end{align}
	The number of computations required scale as $\bcO(p^{2} )$.
\end{theorem}
\subsection{ Connections to surrogate likelihood.}
	\label{subsec:penalized}
To circumvent the computational limitation of exact likelihood-based functionals in nonparametric density estimation, Jeon et al. (2006) \cite{JeonL2006} proposed to minimize the surrogate likelihood. Let $\svbx^{(1)},\dots, \svbx^{(n)}$ be $n$ independent samples of $\rvbx$ where $\rvbx \in \cX$. For densities of the form $f_{\rvbx}(\svbx) \propto e^{\eta(\svbx)}$, the surrogate likelihood is as follows:
\begin{align}
\cl_n(\eta)  = \frac{1}{n} \sum_{t = 1}^{n} \exp\Big( -\eta(\svbx^{(t)}) \Big) + \int_{\svbx} \rho(\svbx) \times  \eta(\svbx)d\svbx \label{eq:penalized surrogate likelihood}
\end{align}
where $\rho(\cdot)$ is some known probability density function on $\cX$.
The following proposition shows that the GISO is a special case of the surrogate likelihood. Proof can be found in Appendix \ref{connections to the penalized surrogate likelihood}.
\begin{proposition} \label{prop:GISO_penalized_likelihood_connection}
	For any $i \in [p]$, the GISO is equivalent to the surrogate likelihood associated with the conditional density of $\rvx_i$ when $\rho(\cdot)$ is the uniform density on $\cX_i$.
\end{proposition}
\subsection{Examples}\label{ssec:examples}
%{\bf Examples of distributions.}
The following are a few examples where the Condition \ref{condition:1} is naturally satisfied (subject to problem setup) as explained in Appendix \ref{sec: examples of distributions appendix}.
Therefore, these distributions are learnable consistently, have asymptotic Gaussian-like behavior (under the assumptions in Theorem \ref{theorem:GRISE-consistency-efficiency}), and have finite sample guarantees.
This is in contrast to most prior works for the continuous setting where there are difficult to verify conditions (even for these examples) such as incoherence, dependency, sparse eigenvalue, and restricted strong convexity.

A. Polynomial (linear) sufficient statistics i.e., $\bphi(x) = x$ and $k = 1$.
\begin{align}
\hspace{-3mm}\DensityParametrizedTrue \propto \exp\bigg(  \sum_{i \in [p]} \thetaIStar x_i +  \sum_{ i \in [p]} \sum_{ j > i} \thetaIJ x_i x_j  \bigg). \label{eq:continuous-ising-density}
\end{align}
B. Harmonic sufficient statistics i.e., $\bphi(x) = \Big(\sin\big(\pi x/b\big), \cos\big(\pi x/b\big)\Big)$ and $k = 2$.
\begin{align}
\DensityParametrizedTrue \propto \exp\bigg(  \sum_{i \in [p]}  \Big[ \thetaIoneStar \sin \frac{\pi x_i}{b} + \thetaItwoStar \cos \frac{\pi x_i}{b} \Big]  + \\
%\vspace{-3mm}
\sum_{\substack{ i \in [p], j > i}}  \Big[ \thetaIJoneStar \sin \frac{\pi(x_i + x_j)}{b} + \thetaIJtwoStar \cos \frac{\pi (x_i + x_j)}{b} \Big]  \bigg)
\end{align}
%\vspace{-10mm}
%See Appendix \ref{sec: examples of distributions appendix} for more examples and details.
\section{Conclusion}
We provide rigorous finite sample analysis for learning structure and parameters of continuous MRFs without the abstract conditions of incoherence, dependency, sparse eigenvalue or restricted strong convexity that are common in literature. We provide easy-to-verify sufficient condition for learning that is naturally satisfied for polynomial and harmonic sufficient statistics. Our methodology requires $\bcO(p^{2})$ computations and $O(\exp(d)\log p)$ samples similar to the discrete and Gaussian settings. Additionally, we propose a robust variation of Lasso by showing that even in the presence of bounded additive noise, the Lasso estimator is ‘prediction consistent’ under mild assumptions.

We also establish that minimizing the population version of GISO \cite{VuffrayML2019} is equivalent to finding MLE of a certain related parametric distribution. We provide asymptotic consistency and normality of the estimator under mild conditions. Further, we show that the GISO is equivalent to the surrogate likelihood proposed by Jeon et al. (2006) \cite{JeonL2006}.

A natural extension of the pairwise setup is the $t$-wise MRF with continuous variables. The approach and the objective function introduced in Vuffray et al. (2019) \cite{VuffrayML2019} naturally extend for such a setting allowing to learn $t$-wise MRFs with general discrete variables as explained in that work. We believe that our results for continuous setting, in a similar vein, extend for $t$-wise MRFs as well and it is an important direction for immediate future work. We also believe that the connection of the GISO to KL-divergence could be used to remove the bounded random variables assumption of our work. Another important direction is to leverage the asymptotic normality of the estimator established in our work to construct data-driven explicit confidence intervals for learned parameters of MRF.
\section*{Acknowledgements}
We would like to thank Andrey Y. Lokhov, Marc Vuffray, and Sidhant Misra for pointing to us the possibility of using GRISE for finite sample analysis of learning continuous graphical models
during the MIFODS Workshop on Graphical models, Exchangeable models and Graphons organized at MIT in summer of 2019. We would also like to thank the anonymous referees of NeurIPS 2020 for pointing out a bug in the earlier version of Theorem \ref{theorem:GRISE-consistency-efficiency}.
%%\bibliographystyle{imsart-nameyear}
%%\bibliography{Mybib_papers}
\clearpage
{\small
	\bibliography{Mybib_papers}
	\bibliographystyle{abbrv}
}
\newpage
%\appendix
\section{Conditional density}
\label{sec:conditional density}%LABEL
In this section, we derive the two forms of the conditional density of $\rvx_i$ for $i \in [p]$ i.e., $\ConditionalDensityNodeI$ used in Section \ref{sec:algorithm}. 
We further obtain lower and upper bounds on this conditional density.

\subsection{Forms of conditional density}
\label{subsec:forms of conditional density}
We will first derive the form of conditional density in \eqref{eq:form1}. 
For any $i \in [p]$, the conditional density of node $\rvx_i$ given the values taken by all other nodes is obtained by applying Bayes' theorem to $\DensityParametrizedTrue$ and is given by
\begin{align}
\ConditionalDensityNodeI & =  \frac{\exp \Big( \bthetaIStarT \bphi(x_i) + \sum_{ j \in [p], j \neq i} \bthetaIJStarT \bpsi(x_i,x_j)  \Big)}{\int_{x_i \in \cX_i} \exp \Big( \bthetaIStarT \bphi(x_i) + \sum_{ j \in [p], j \neq i} \bthetaIJStarT \bpsi(x_i,x_j)  \Big) dx_i}   \label{eq:conditionalDensity1}
\end{align}
where $\rvx_{-i} \coloneqq \rvbx \setminus \rvx_i $ and $ x_{-i} \coloneqq \svbx \setminus x_i$. 
Recall definition of locally centered basis functions in \eqref{eq:centeredBasisFunctions1} and \eqref{eq:centeredBasisFunctions2} from perspective of $i \in [p], j \in [p] \backslash \{i\}$. For $x \in \cX_i$, $x' \in \cX_j$
\begin{align*}
\bphiI(x) & \coloneqq \bphi(x) - \int_{y \in \cX_i} \bphi(y) dy \\
\bpsiI(x, x') & \coloneqq \bpsi(x, x') - \int_{y \in \cX_i} \bpsi(y,x') dy.
\end{align*}
We can 
%Observe that the terms subtracted from the original basis functions to obtain the locally centered basis functions do not involve $x_i$. 
%We can multiply both the numerator and the denominator by these terms that do not involve $x_i$ and 
rewrite \eqref{eq:conditionalDensity1} as
\begin{align}
\ConditionalDensityNodeI & = \frac{\exp \Big( \bthetaIStarT \bphiI(x_i) + \sum_{ j \in [p], j \neq i} \bthetaIJStarT \bpsiI(x_i,x_j)  \Big)}{\int_{x_i \in \cX_i} \exp \Big( \bthetaIStarT \bphiI(x_i) + \sum_{ j \in [p], j \neq i} \bthetaIJStarT \bpsiI(x_i, x_j)  \Big) dx_i}  \label{eq:conditionalDensity2}
\end{align}
Recalling notation of $\bvthetaIStar$ and $\bvphiI(x_i ; x_{-i})$, this results in 
\begin{align}
\ConditionalDensityNodeI & = \frac{\exp \Big( \bvthetaIStarT \bvphiI(x_i ; x_{-i})  \Big)}{\int_{x_i \in \cX_i} \exp \Big( \bvthetaIStarT \bvphiI(x_i ; x_{-i})  \Big) dx_i}.  \label{eq:conditionalDensity3}
\end{align}
We will now derive the form of conditional density in \eqref{eq:form2}. 
Using the definition of Kronecker product, the conditional density in \eqref{eq:conditionalDensity1} can also be written as:
\begin{align}
\ConditionalDensityNodeI & = \frac{\exp \Big( \sum_{ r \in [k] } \thetaIrStar \phi_r(x_i) + \sum_{ j \neq i} \sum_{ r,s \in [k] } \thetaIJrsStar \phi_r(x_i) \phi_s(x_j)  \Big)}{\int_{x_i \in \cX_i} \exp  \Big( \sum_{ r \in [k] } \thetaIrStar \phi_r(x_i) + \sum_{ j \neq i} \sum_{ r,s \in [k] } \thetaIJrsStar \phi_r(x_i) \phi_s(x_j)  \Big) dx_i}.  \label{eq:conditionalDensity4}
\end{align}
Recalling notation of $\blambdaStar(x_{-i})$, this results in 
\begin{align}
\ConditionalDensityNodeI & = \frac{\exp\Big(  \blambdaStarT (x_{-i}) \bphi(x_i)\Big) }{\int_{x_i \in \cX_i} \exp\Big(  \blambdaStarT (x_{-i}) \bphi(x_i)\Big) dx_i}.  \label{eq:conditionalDensity5}
\end{align}

\subsection{Bounds on conditional density}
\label{subsec:bounds on conditional density}
Let us first bound the locally centered basis function. For any $i \in [p], r \in [k]$, let $\bphiI_r(\cdot)$ denote the $r^{th}$ element of $\bphiI(\cdot)$.
We have $\forall i \in [p], \forall r \in [k]$
\begin{align}
\Big| \bphiI_r(x_i) \Big| & \stackrel{(a)}{\leq}  \Big| \phi_{r}(x_i) \Big| + \Big| \int_{y_i \in \cX_i} \phi_{r}(y_i) dy_i \Big|  \stackrel{(b)}{\leq}  | \phi_r(x_i) | +  \int_{y_i \in \cX_i} | \phi_r(y_i) | dy_i  \stackrel{(c)}{\leq}  \phiMax (1+\Xupper).
\end{align}
where $(a)$ follows by applying the triangle inequality, $(b)$ follows because the absolute value of an integral is smaller than or equal to the integral of an absolute value, and $(c)$ follows because $|\phi_r(x)| \leq \phiMax$ $\forall r \in [k], x \in \cup_{i \in [p]} \cX_i$ and the length of the interval $\cX_i$ is upper bounded by $\Xupper$. Therefore, 
\begin{align}
\| \bphiI(\cdot)\|_{\infty} \leq  (1+\Xupper) \phiMax.
\end{align}
Similary,
\begin{align}
\| \bpsiI(\cdot)\|_{\infty} \leq  (1+\Xupper) \phiMax^2.
\end{align}
Recall the definition of $\varphiMax$. We now have
\begin{align}
\| \bvphiI(\svbx) \|_{\infty} \leq \varphiMax.  \label{eq:basisUpperBound}
\end{align}
Also, recall that $\| \bvthetaIStar \|_1 \leq \g$. Using this and \eqref{eq:basisUpperBound}, we have
\begin{align}
\exp \Big( -\g \varphiMax \Big) \leq \exp \Big( \bvthetaIStarT \bvphiI(\svbx)  \Big) \leq \exp \Big( \g \varphiMax \Big).  \label{eq:boundsExp}
\end{align}
As a result, we can lower and upper bound the conditional density in \eqref{eq:conditionalDensity3} as,
\begin{align}
f_L \coloneqq \frac{\exp \Big( -2\g \varphiMax \Big)}{\Xupper} \leq \ConditionalDensityNodeI  \leq f_U \coloneqq \frac{\exp \Big( 2\g \varphiMax \Big)}{\Xlower}.  \label{eq:boundsDensity}
\end{align}

\section{Proof of Theorem \ref{theorem:GRISE-KLD}}
\label{sec:proof of theorem-GRISE-KLD}%LABEL
In this section, we prove Theorem \ref{theorem:GRISE-KLD}. 
Consider $i \in [p]$. For any $\bvtheta \in \parameterSet$, recall that the population version of GISO is given by
\begin{align}
\cS^{(i)}(\bvtheta)  = \Expectation  \bigg[\exp\Big( -\bvthetaT \bvphiI(\rvbx) \Big) \bigg]. 
\end{align}
Also, recall that the parametric distribution $\DifferenceDensity$ under consideration has the following density:
\begin{align}
\DifferenceDensity \propto \DensityParametrizedTrue \times \exp \Big( -\bvthetaT \bvphiI(\svbx) \Big)
\end{align}
and the density $\UniformProxyDensity$ is given by:
\begin{align}
\UniformProxyDensity \propto \DensityParametrizedTrue \times \exp \Big( -\bvthetaIStarT \bvphiI(\svbx) \Big)
\end{align}
We show that minimizing $\cS^{(i)}(\bvtheta)$ is equivalent to minimizing the KL-divergence between the distribution with density $\UniformProxyDensityfun$ and the distribution with density $\DifferenceDensityfun$. 
In other words, we show that, at the population level, the GRISE is a ``local'' maximum likelihood estimate. 
We further show that the true parameter vector $\bvthetaIStar$ for $i \in [p]$ is a unique minimizer of $\cS^{(i)}(\bvtheta)$.

\begin{proof}[Proof of Theorem \ref{theorem:GRISE-KLD}]
	We will first write $\DifferenceDensityfun$ in terms of $\cS^{(i)}(\bvtheta) $. We have
	\begin{align}
	\DifferenceDensity & = \frac{\DensityParametrizedTrue \exp\Big( -\bvthetaT \bvphiI(\svbx) \Big) }{\int_{\svbx \in \cX}   \DensityParametrizedTrue \exp\Big( -\bvthetaT \bvphiI(\svbx) \Big) d\svbx}\\
	& \stackrel{(a)}{=} \frac{\DensityParametrizedTrue \exp\Big( -\bvthetaT \bvphiI(\svbx) \Big) }{\cS^{(i)}(\bvtheta)  }
	\end{align}
	where $(a)$ follows from definition of $\cS^{(i)}(\bvtheta)$.
	
	Now let us write an alternative expression for $\UniformProxyDensity$ which does not depend on $x_i$ functionally. We have
	\begin{align}
	\UniformProxyDensity & \stackrel{(a)}{\propto}  \DensityParametrizedNotITrue \times \ConditionalDensityNodeI \times \exp \Big( -\bvthetaIStarT \bvphiI(\svbx) \Big) \\
	& \stackrel{(b)}{\propto}  \frac{\DensityParametrizedNotITrue}{\int_{x_i \in \cX_i} \exp \Big( \bvthetaIStarT \bvphiI(\svbx) \Big) dx_i}
	\end{align}
	where $(a)$ follows from $\DensityParametrizedfunTrue = \DensityParametrizedNotIfunTrue \times \ConditionalDensityNodeIfun$ and $(b)$ follows from \eqref{eq:conditionalDensity3}.\\
	
	We will now simplify the KL-divergence between $\UniformProxyDensityfun$ and $\DifferenceDensityfun$. For any $l \in [k+k^2(p-1)]$, let $\bvtheta_{l}$ denote the $l^{th}$ component of $\bvtheta$ and $\bvphiI_{l}(\svbx)$ denote the $l^{th}$ component of $\bvphiI(\svbx)$.
	\begin{align}
	& \infdiv{\UniformProxyDensity}{\DifferenceDensity}  \\
	&  = \int_{\svbx \in \cX} \UniformProxyDensity \log\bigg( \frac{\UniformProxyDensity \cS^{(i)}(\bvtheta) }{\DensityParametrizedTrue \exp\Big( -\bvthetaT \bvphiI(\svbx) \Big)}\bigg) d\svbx \\
	&  \stackrel{(a)}{=} \int_{\svbx \in \cX} \UniformProxyDensity \log\bigg( \frac{\UniformProxyDensity}{\DensityParametrizedTrue} \bigg) d\svbx + \int_{\svbx \in \cX} \UniformProxyDensity \times  \bvthetaT \bvphiI(\svbx) d\svbx +  \log \cS^{(i)}(\bvtheta)  \\
	&  = \int_{\svbx \in \cX} \UniformProxyDensity \log\bigg( \frac{\UniformProxyDensity}{\DensityParametrizedTrue} \bigg) d\svbx + \sum_{l} \Big[\bvtheta_{l} \int_{\svbx \in \cX} \UniformProxyDensity \times  \bvphiI_{l}(\svbx) d\svbx\Big] +  \log \cS^{(i)}(\bvtheta)  \\
	&  \stackrel{(b)}{=} \int_{\svbx \in \cX}  \UniformProxyDensity \log\bigg( \frac{\UniformProxyDensity}{\DensityParametrizedTrue} \bigg) d\svbx + \log \cS^{(i)}(\bvtheta)
	\end{align}
	where $(a)$ follows because $\log(ab) = \log a + \log b$ and $\cS^{(i)}(\bvtheta) $ is a constant and $(b)$ follows because $\UniformProxyDensityfun$ does not functionally depend on $x_i \in \cX_i$ and the basis functions are locally centered from perspective of $i$. 
	Observing that the first term in the above equation is independent on $\bvtheta$, we can write
	\begin{align}
	\argmin_{\bvtheta \in \parameterSet: \|\bvtheta\|_1 \leq \g}  
	\infdiv{\UniformProxyDensityfun}{\DifferenceDensityfun} %\\
	= \argmin_{\bvtheta \in \parameterSet: \|\bvtheta\|_1 \leq \g} \log \cS^{(i)}(\bvtheta)
	= \argmin_{\bvtheta \in \parameterSet: \|\bvtheta\|_1 \leq \g} \cS^{(i)}(\bvtheta). 
	\end{align}
	Further, the KL-divergence between $\UniformProxyDensityfun$ and $\DifferenceDensityfun$ is minimized when $\UniformProxyDensityfun = \DifferenceDensityfun$. Recall that the basis functions are such that the exponential family is minimal. Therefore, $\UniformProxyDensityfun = \DifferenceDensityfun$ only when $\bvtheta = \bvthetaIStar$. Thus,
	\begin{align}
	\bvthetaIStar \in \argmin_{\bvtheta \in \parameterSet: \|\bvtheta\|_1 \leq \g}  \cS^{(i)}(\bvtheta)
	\end{align}
	and it is a unique minimizer of $\cS^{(i)}(\bvtheta)$.\\
	
	Similar analysis works for MRFs with discrete variables as well i.e., the setting considered in Vuffray et al. (2019) \cite{VuffrayML2019}.
\end{proof}

\section{Proof of Theorem \ref{theorem:GRISE-consistency-efficiency}}
\label{sec:proof of theorem:GRISE-consistency-efficiency}%LABEL
In this section, we prove Theorem \ref{theorem:GRISE-consistency-efficiency}.
We will use the theory of M-estimation. In particular, we observe that $\hbvthetaIn$ is an M-estimator and invoke Theorem 4.1.1 and Theorem 4.1.3 of \cite{Amemiya1985} for consistency and normality of M-estimators respectively.

\begin{proof}[Proof of Theorem \ref{theorem:GRISE-consistency-efficiency}] We divide the proof in two parts.\\
	
	{\bf Consistency. } We will first show that the GRISE is a consistent estimator.
	
	Recall \cite[Theorem~4.1.1]{Amemiya1985}: Let $y_1, \cdots, y_n$ be i.i.d. samples of a random variable $\rvy$. Let $q(\rvy ; \vtheta)$ be some function of $\rvy$ parameterized by $\vtheta \in \Theta$. Let $\vthetaStar$ be the true underlying parameter. Define
	\begin{align}
	Q_n(\vtheta) = \frac{1}{n} \sum_{i = 1}^{n} q(y_i ; \vtheta) \label{eq:m-est-con1}
	\end{align}
	and
	\begin{align}
	\hvthetan \in \argmin_{\vtheta \in \Theta} Q_n(\vtheta) \label{eq:m-est-con2}
	\end{align} 
	The M-estimator $\hvthetan$ is consistent for $\vthetaStar$ i.e., $\hvthetan \stackrel{p}{\to} \vthetaStar$ as $n\to \infty$ if,
	\begin{enumerate}
		\item[(a)] $\Theta$ is compact,
		\item[(b)] $Q_n(\vtheta)$ converges uniformly in probability to a non-stochastic function $Q(\vtheta)$, 
		\item[(c)] $Q(\vtheta)$ is continuous, and
		\item[(d)] $Q(\vtheta)$ is uniquely minimzed at $\vthetaStar$.
	\end{enumerate}
	
	Comparing \eqref{eq:GISO-main} and \eqref{eq:GRISE} with \eqref{eq:m-est-con1} and \eqref{eq:m-est-con2}, we only need to show that the above regularity conditions (a)-(d) hold for $Q_n(\vtheta) \coloneqq \cS_{n}^{(i)}(\bvtheta)$ in order to prove that $\hbvthetaIn \xrightarrow{p} \bvthetaIStar$ as $n\to \infty$. We have the following:
	
	\begin{enumerate}
		\item[(a)] The parameter space $\parameterSet$ is bounded and closed. Therefore, we have compactness.
		\item[(b)] Recall \cite[Theorem 2]{Jennrich1969}: Let $y_1, \cdots, y_n$ be i.i.d. samples of a random variable $\rvy$. Let $g(\rvy ; \vtheta)$ be a  function of $\vtheta$ parameterized by $\vtheta \in \Theta$. Suppose (a) $\Theta$ is compact, (b) $g(\rvy , \vtheta)$ is continuous at each $\vtheta \in \Theta$ with probability one, (c) $g(\rvy , \vtheta)$ is dominated by a function $G(\rvy)$ i.e., $| g(\rvy , \vtheta) | \leq G(\rvy)$, and (d) $\Expectation[G(\rvy)] < \infty$. Then, $n^{-1} \sum_t g(y_t , \vtheta)$ converges uniformly in probability to $\Expectation [ g(\rvy, \vtheta)]$.\\
		
		Using this theorem with $\rvy \coloneqq \rvbx$, $y_t \coloneqq \svbx^{(t)}$, $\Theta \coloneqq \parameterSet$, $g(\rvy, \vtheta) \coloneqq \exp\Big( -\bvthetaT \bvphiI(\svbx) \Big)$, $G(\rvy) \coloneqq \exp(\g \varphiMax)$, we conclude that $\cS_{n}^{(i)}(\bvtheta)$ converges to $\cS^{(i)}(\bvtheta)$ uniformly in probability.
		\item[(c)] From the continuity of $\exp\Big( -\bvthetaT \bvphiI(\svbx) \Big)$ we have continuity of $\cS^{(i)}(\bvtheta)$ and $\cS_{n}^{(i)}(\bvtheta)$ for all $\bvtheta \in \parameterSet$
		\item[(d)] From Theorem \ref{theorem:GRISE-KLD}, $\bvthetaIStar$ is a unique minimizer of $\cS^{(i)}(\bvtheta)$.
	\end{enumerate}
	Therefore, we have asymptotic consistency for GRISE.
	\\
	
	{\bf Normality. }
	We will now show that the GRISE is asymptotically normal.
	
	Recall \cite[Theorem~4.1.3]{Amemiya1985}: Let $y_1, \cdots, y_n$ be i.i.d. samples of a random variable $\rvy$. Let $q(\rvy ; \vtheta)$ be some function of $\rvy$ parameterized by $\vtheta \in \Theta$. Let $\vthetaStar$ be the true underlying parameter. Define
	\begin{align}
	Q_n(\vtheta) = \frac{1}{n} \sum_{i = 1}^{n} q(y_i ; \vtheta) \label{eq:m-est-eff1}
	\end{align}
	and
	\begin{align}
	\hvthetan \in \argmin_{\vtheta} Q_n(\vtheta) \label{eq:m-est-eff2}
	\end{align} 
	The M-estimator $\hvthetan$ is normal for $\vthetaStar$ i.e., $\sqrt{n}( \hvthetan - \vthetaStar)\stackrel{d}{\to} {\cal N}({\bf 0}, B^{-1}(\vthetaStar)A(\vthetaStar)B^{-1}(\vthetaStar))$ if
	\begin{enumerate}
		\item[(a)] $\hvthetan$, the minimzer of $Q_n(\cdot)$, is consistent for $\vthetaStar$, 
		\item[(b)] $\vthetaStar$ lies in the interior of the parameter space $\Theta$, 
		\item[(c)] $Q_n$ is twice continuously differentiable in an open and convex neighbourhood of $\vthetaStar$, 
		\item[(d)] $\sqrt{n}\nabla Q_n(\vtheta)|_{\vtheta = \vthetaStar} \stackrel{d}{\to} {\cal N}({\bf 0}, A(\vthetaStar))$, and 
		\item[(e)] $\nabla^2 Q_n(\vtheta)|_{\vtheta = \hvthetan} \stackrel{p}{\to} B(\vthetaStar)$ with $B(\vtheta)$ finite, non-singular, and continuous at $\vthetaStar$, 
	\end{enumerate}
	%then 
	%\begin{align}
	%\sqrt{n}( \hvthetan - \btheta)\stackrel{d}{\to} {\cal N}({\bf 0}, B^{-1}(\vtheta)A(\btheta)B^{-1}(\vtheta))
	%\end{align}
	
	Comparing \eqref{eq:GISO-main} and \eqref{eq:GRISE} with \eqref{eq:m-est-eff1} and \eqref{eq:m-est-eff2}, we only need to show that the above regularity conditions (a)-(e) hold for $Q_n(\vtheta) \coloneqq \cS_{n}^{(i)}(\bvtheta)$ in order to prove that the GRISE is asymptotically normal. We have the following:
	
	\begin{enumerate}
		\item[(a)] We have already established that $\hbvthetaIn$ is consistent for $\bvthetaIStar$.
		\item[(b)] We assume that none of the parameter is equal to the boundary values of $\thetaMin$ or $\thetaMax$. Therefore, $\bvthetaIStar$ lies in the interior of $\parameterSet$.
		\item[(c)] From \eqref{eq:GISO-main}, we have
		\begin{align}
		\cS_{n}^{(i)}(\bvtheta)  = \frac{1}{n} \sum_{t = 1}^{n} \exp\Big( -\bvthetaT \bvphiI(\svbx^{(t)}) \Big).
		\end{align}
		For any $l \in [k+k^2(p-1)]$, let $\bvtheta_{l}$ denote the $l^{th}$ component of $\bvtheta$ and $\bvphiI_{l}(\svbx^{(t)})$ denote the $l^{th}$ component of $\bvphiI(\svbx^{(t)})$. For any $l_1,l_2 \in [k+k^2(p-1)]$, we have
		\begin{align}
		\frac{\partial^2 \cS_{n}^{(i)}(\bvtheta)}{\partial \bvtheta_{l_1}\partial \bvtheta_{l_2}}  = \frac{1}{n} \sum_{t = 1}^{n} \bvphiI_{l_1}(\svbx^{(t)}) \bvphiI_{l_2}(\svbx^{(t)}) \exp\Big( -\bvthetaT \bvphiI(\svbx^{(t)}) \Big)  \label{eq:hessian-GISO}
		\end{align}
		Thus, $\partial^2 \cS_{n}^{(i)}(\bvtheta)/\partial \bvtheta_{l_1}\partial \bvtheta_{l_2}$ exists. Using the continuity of $\bvphiI(\cdot)$ and 
		$\exp\Big( -\bvthetaT \bvphiI(\cdot) \Big)$, we see that $\partial^2 \cS_{n}^{(i)}(\bvtheta)/\partial \bvtheta_{l_1}\partial \bvtheta_{l_2}$ is continuous in an open and convex neighborhood of $\bvthetaIStar$.
		\item[(d)] For any $l \in [k+k^2(p-1)]$, define the following random variable:
		\begin{align}
		\rvx_{i,l} \coloneqq - \bvphiI_{l}(\rvbx)\exp\Big( -\bvthetaIStarT \bvphiI(\rvbx) \Big) 
		\end{align}
		The $l^{th}$ component of the gradient of the GISO evaluated at $\bvthetaIStar$ is given by
		\begin{align}
		\left.\frac{\partial \cS_{n}^{(i)}(\bvtheta)}{\partial \bvtheta_{l}}\right\vert_{\bvtheta = \bvthetaIStar} = \frac{1}{n} \sum_{t = 1}^{n} - \bvphiI_{l}(\svbx^{(t)})\exp\Big( -\bvthetaIStarT \bvphiI(\svbx^{(t)}) \Big)
		\end{align}
		Each term in the above summation is distributed as the random variable $\rvx_{i,l} $.
		The random variable $\rvx_{i,l}$ has zero mean (see Lemma \ref{lemma:expectation-gradient-GISO}). Using this and the multivariate central limit theorem \cite{Vaart2000}, we have
		\begin{align}
		\sqrt{n} \nabla \cS_{n}^{(i)}(\bvtheta) |_{\bvtheta = \bvthetaIStar} \xrightarrow{d} {\cal N}({\bf 0}, A(\bvthetaIStar))
		\end{align} 
		where $A(\bvthetaIStar)$ is the covariance matrix of $\bvphiI(\rvbx)\exp\big( -\bvthetaIStarT \bvphiI(\rvbx) \big)$.
		\item [(e)] We will first show that the following is true.
		\begin{align}
		\nabla^2 \cS_{n}^{(i)}(\bvtheta) |_{\bvtheta = \hbvthetaIn} \xrightarrow{p} \nabla^2 \cS^{(i)}(\bvtheta) |_{\bvtheta = \bvthetaIStar}  \label{eq:ulln+cmt}
		\end{align}
		To begin with, using the uniform law of large numbers \cite[Theorem 2]{Jennrich1969} for any $\bvtheta \in \parameterSet$ results in
		\begin{align}
		\nabla^2 \cS_{n}^{(i)}(\bvtheta)  \xrightarrow{p} \nabla^2 \cS^{(i)}(\bvtheta) \label{eq:ulln} 
		\end{align}
		Using the consistency of $\hbvthetaIn$ and the continuous mapping theorem, we have
		\begin{align}
		\nabla^2 \cS^{(i)}(\bvtheta) |_{\bvtheta = \hbvthetaIn} \xrightarrow{p} \nabla^2 \cS^{(i)}(\bvtheta) |_{\bvtheta = \bvthetaIStar}  \label{eq:cmt}
		\end{align}
		Let $l_1, l_2 \in [k+k^2(p-1)]$. From \eqref{eq:ulln} and \eqref{eq:cmt}, for any $\epsilon > 0$, for any $\delta > 0$, there exists integers $n_1 , n_2$ such that
		\begin{align}
		\Prob( | \big[\nabla^2 \cS_{n}^{(i)} (\hbvthetaIn)  \big]_{l_1,l_2}  - \big[\nabla^2 \cS^{(i)} (\hbvthetaIn)  \big]_{l_1,l_2}  | > \epsilon / 2 ) \leq \delta / 2 \qquad \text{ if } n \geq n_1 \\
		\Prob( | \big[\nabla^2 \cS^{(i)} (\hbvthetaIn)  \big]_{l_1,l_2}  - \big[\nabla^2 \cS^{(i)} (\bvthetaIStar)  \big]_{l_1,l_2}  | > \epsilon / 2 ) \leq \delta / 2 \qquad \text{ if } n \geq n_2
		\end{align} 
		Now for $n \geq \max\{n_1,n_2\}$, we have
		\begin{align}
		\Prob( | \big[\nabla^2 \cS_{n}^{(i)} (\hbvthetaIn)  \big]_{l_1,l_2}  - \big[\nabla^2 \cS^{(i)} (\bvthetaIStar)  \big]_{l_1,l_2}  | > \epsilon) & \leq \Prob( | \big[\nabla^2 \cS_{n}^{(i)} (\hbvthetaIn)  \big]_{l_1,l_2}  - \big[\nabla^2 \cS^{(i)} (\hbvthetaIn)  \big]_{l_1,l_2}  | > \epsilon / 2 ) \\
		 &  + \Prob( | \big[\nabla^2 \cS^{(i)} (\hbvthetaIn)  \big]_{l_1,l_2}  - \big[\nabla^2 \cS^{(i)} (\bvthetaIStar)  \big]_{l_1,l_2}  | > \epsilon / 2 ) \\
		 & \leq \delta / 2 + \delta / 2 = \delta
		\end{align}
		Thus, we have \eqref{eq:ulln+cmt}. Using \eqref{eq:GISO-pop}, we have
		\begin{align}
		\big[\nabla^2 \cS^{(i)} (\bvthetaIStar)  \big]_{l_1,l_2} & = \Expectation \bigg[ \bvphiI_{l_1}(\rvbx)  \bvphiI_{l_2}(\rvbx)\exp\Big( -\bvthetaIStarT \bvphiI(\rvbx) \Big)  \bigg] \\
		& \stackrel{(b)}{=} \Expectation \bigg[ \bvphiI_{l_1}(\rvbx)  \bvphiI_{l_2}(\rvbx)\exp\Big( -\bvthetaIStarT \bvphiI(\rvbx) \Big)  \bigg]  \\ & \qquad - \Expectation \bigg[ \bvphiI_{l_1}(\rvbx)\bigg]  \Expectation \bigg[ \bvphiI_{l_2}(\rvbx)\exp\Big( -\bvthetaIStarT \bvphiI(\rvbx) \Big)  \bigg]\\
		& = \text{cov} \bigg(  \bvphiI_{l_1}(\rvbx) , \bvphiI_{l_2}(\rvbx)\exp\Big( -\bvthetaIStarT \bvphiI(\rvbx) \Big)\bigg)
		\end{align}	
		where (b) follows from Lemma \ref{lemma:expectation-gradient-GISO}.	Therefore, we have
		\begin{align}
		\nabla^2 \cS_{n}^{(i)}(\bvtheta) |_{\bvtheta = \hbvthetaIn} \xrightarrow{p} B(\bvthetaIStar)
		\end{align}
		where $B(\bvthetaIStar)$ is the cross-covariance matrix of $\bvphiI(\rvbx)$ and $\bvphiI(\rvbx)\exp\big( -\bvthetaIStarT \bvphiI(\rvbx) \big)$. Finiteness and continuity of $\bvphiI(\rvbx)$ and $\bvphiI(\rvbx)\exp\big( -\bvthetaIStarT \bvphiI(\rvbx) \big)$ implies the finiteness and continuity of $B(\bvthetaIStar)$.
	\end{enumerate}
	Therefore, under the assumption that the cross-covariance matrix of $\bvphiI(\rvbx)$ and $\bvphiI(\rvbx)\exp\big( -\bvthetaIStarT \bvphiI(\rvbx) \big)$ is invertible, and that none of the parameter is equal to the boundary values of $\thetaMax$ or $\thetaMin$, we have the asymptotic normality of GRISE i.e.,
	\begin{align}
	\sqrt{n} ( \hbvthetaIn - \bvthetaIStar ) \xrightarrow{d}  {\cal N}({\bf 0},B(\bvthetaIStar)^{-1} A(\bvthetaIStar) B(\bvthetaIStar)^{-1})
	\end{align}
\end{proof}

\section{Supporting lemmas for Theorem \ref{theorem:structure} and \ref{theorem:parameter}}
\label{sec:supporting lemmas for theorem}%LABEL
In this section, we will state the two key lemmas required in the proof of Theorem \ref{theorem:structure} and \ref{theorem:parameter}. 
The proof of Theorem \ref{theorem:structure} is given in Appendix \ref{sec:proof of theorem:structure} and the proof of Theorem \ref{theorem:parameter} is given in Appendix \ref{sec:proof of theorem:parameter}. 
Recall the definitions of $\g = \thetaMax(k+k^2d)$, $\varphiMax = (1+\Xupper)  \max\{\phiMax,\phiMax^2\}$, and $c_1 (\alpha)$ from Section \ref{sec:problem setup}.
Also, define 
\begin{align}
c_3 (\alpha) & = \frac{k^2 d^4  \g^8\varphiMax^8\exp(8\g\varphiMax) }{ \kappa^4 \alpha^8} ~=~O\Bigg(\frac{\exp(\Theta(k^2 d))}{\kappa^4 \alpha^8}\Bigg)
\end{align}
\subsection{Error Bound on Edge Parameter Estimation with GRISE}
\label{subsec:error bound on edge-wise parameter estimation with grise}
The following lemma shows that, with enough samples, the parameters associated with the edge potentials can be recovered, within small error, with high probability using the GISO for continuous variables from Section \ref{sec:algorithm}. 
\begin{lemma} \label{lemma:recover-edge-terms}
	Let Condition \ref{condition:1} be satisfied. Given $n$ independent samples $\svbx^{(1)},\dots, \svbx^{(n)}$ of $\rvbx$, for each $i \in [p]$, 
	let $\hbvthetaIEps$ be an $\epsilon$-optimal solution of \eqref{eq:GRISE}. Let $\hbvthetaIEpsEdge = (\hat{\theta}_{ij}, j \neq i, j \in [p])$ be its components corresponding to all possible $p-1$
	edges associated with node $i$. Let $\alpha_1 > 0$ be the prescribed accuracy level. 
	Then, for any $\delta \in (0,1)$, 
	\begin{align}
	\| \bvthetaIEdgeStar - \hbvthetaIEpsEdge \|_{2} \leq \alpha_1, \hspace{1cm} \forall i \in [p]
	\end{align}
	with probability at least $1 - \delta$ as long as 
	\begin{align}
	n & \geq  c_1\big(\alpha_1\big)  \log\bigg(\frac{2pk}{\sqrt{\delta}}\bigg)~=~ 
	\Omega\Bigg(\frac{\exp(\Theta(k^2 d))}{\kappa^2 \alpha_1^4} \log\bigg(\frac{pk}{\sqrt{\delta}}\bigg) \Bigg).	
	\end{align}
	The number of computations required scale as
	\begin{align}
	c_3(\alpha_1)  \times \log\bigg(\frac{2pk}{\sqrt{\delta}}\bigg) \times \log{(2k^2p)}  \times p^2 ~=~ 
	\Omega\Bigg(\frac{\exp(\Theta(k^2 d))}{\kappa^4 \alpha_1^8} \log^2\bigg(\frac{pk}{\sqrt{\delta}}\bigg)  p^2  \Bigg).	
	\end{align}
\end{lemma}
The proof of Lemma \ref{lemma:recover-edge-terms} is given in Appendix \ref{sec:proof of lemma:recover-edge-terms}.

\subsection{Error Bound on Node Parameter Estimation}
\label{subsec:error bound on node-wise parameter estimation}
The following lemma shows that, with enough samples, the parameters associated with the node potentials can be recovered, within small error, with high probability using the three-step procedure from Section \ref{sec:algorithm}. 
\begin{lemma} \label{lemma:recover-node-terms}
	Let Condition \ref{condition:1} be satisfied. Given $n$ independent samples $\svbx^{(1)},\dots, \svbx^{(n)}$ of $\rvbx$, for each $i \in [p]$, 
	let $\hbthetaI$ be an estimate of $\bthetaIStar$ obtained using the three-step procedure from Section \ref{sec:algorithm}. 
	Then, for any $\alpha_2 \in (0,1)$, 
	\begin{align}
	\|  \bthetaIStar - \hbthetaI\|_{\infty} \leq \alpha_2, \hspace{1cm} \forall i \in [p]
	\end{align}
	with probability at least $1 - \alpha_2^4$ as long as 
	\begin{align}
	n & \geq \max\Big[ c_1\bigg(\min\bigg\{ \frac{\thetaMin}{3}, \frac{\alpha_2}{2 d k \phiMax } \bigg\} \bigg) \log \bigg( \frac{4 p k }{\alpha_2^2} \bigg), c_2( \alpha_2) \Big] \\
	& =~ \Omega\Bigg( \frac{\exp(\Theta\Big(k^2 d + d \log \big(\frac{dk}{\alpha_2 q^*}\big)\Big))}{\kappa^2 \alpha_2^4}   \times \log \bigg(\frac{pk}{\alpha_2^2} \bigg) \Bigg).
	\end{align}
	The number of computations required scale as
	\begin{align}
	c_3\Big(\min\Big\{\frac{\thetaMin}{3}, \frac{\alpha_2}{2dk\phiMax} \Big\}\Big)  \times \log\bigg(\frac{4pk}{\alpha_2^2}\bigg) \times \log{(2k^2p)}  \times p^2 ~=~ 
	\Omega\Bigg(\frac{\exp(\Theta(k^2 d))}{\kappa^4 \alpha_2^8} \log^2\bigg(\frac{pk}{\alpha_2^2}\bigg)  p^2  \Bigg).	
	\end{align}
\end{lemma}
The proof of Lemma \ref{lemma:recover-node-terms} is given in Appendix \ref{sec:proof of lemma:recover-node-terms}.

\section{Proof of Theorem \ref{theorem:structure}}
\label{sec:proof of theorem:structure}
In this section, we prove Theorem \ref{theorem:structure}. 
See Appendix \ref{subsec:error bound on edge-wise parameter estimation with grise} for the key lemma required in the proof.\\

Recall the definitions of $\g = \thetaMax(k+k^2d)$, $\varphiMax = (1+\Xupper)  \max\{\phiMax,\phiMax^2\}$ and $c_1(\alpha)$ from Section \ref{sec:problem setup}.

\begin{proof}[Proof of Theorem \ref{theorem:structure}]
	The graph $\hG = ([p], \hat{E})$ is such that:
	\begin{align}
	\hat{E} & = \bigg\{ (i,j): i < j \in [p], \Big( \sum_{ r,s \in [k] } \Indicator\{| \hthetaIJrs | > \thetaMin/3 \} \Big) > 0 \bigg\}.
	\end{align}
	The graph $G(\bthetaStar) = ([p], E(\bthetaStar))$ is such that $E(\bthetaStar) = \{ (i,j): i < j \in [p], \| \bthetaIJStar \|_0 > 0 \}$.
	
	Let the number of samples satisfy 
	\begin{align}
	n \geq  c_1\Big(\frac{\thetaMin}{3}\Big)  \log\bigg(\frac{2pk}{\sqrt{\delta}}\bigg)
	\end{align}
	Recall that $\hbvthetaIEps \in \parameterSet$ is an $\epsilon$-optimal solution of GRISE and $\hbvthetaIEpsEdge$ is the component of $\hbvthetaIEps$ associated with the edge potentials. 
	Using Lemma \ref{lemma:recover-edge-terms} with $\alpha_1 = \thetaMin/3$ and any $\delta \in (0,1)$, we have with probability at least $1-\delta$,
	\begin{align}
	\| \bvthetaIEdgeStar - \hbvthetaIEpsEdge \|_{2} \leq & \frac{\thetaMin}{3}, \hspace{1cm} \forall i \in [p] \\
	\implies \| \bvthetaIEdgeStar - \hbvthetaIEpsEdge \|_{\infty} \stackrel{(a)}{\leq} & \frac{\thetaMin}{3}, \hspace{1cm} \forall i \in [p]  \label{eq:infinity-error-bound-edge-wise}
	\end{align}
	where $(a)$ follows because $\| \svbv \|_{\infty} \leq \| \svbv \|_{2}$ for any vector $\svbv$. 
	
	From Section \ref{sec:problem setup}, we have $\| \bvthetaIStar  \|_{\min_+} \geq \thetaMin$. 
	This implies that $\| \bvthetaIEdgeStar  \|_{\min_+} \geq \thetaMin$. 
	Combining this with \eqref{eq:infinity-error-bound-edge-wise}, we have with probability at least $1-\delta$,
	\begin{align}
	\thetaIJrsStar = 0 \iff | \hthetaIJrs | \leq \thetaMin/3, \hspace{1cm} \forall i \in [p], \forall j \in [p] \setminus \{i\}, \forall r,s \in [k].
	\end{align}
	Therefore, with probability at least $1-\delta$, $E(\bthetaStar) = \hat{E}$. 
	
	Further, from Lemma \ref{lemma:recover-edge-terms}, the number of computations required for generating $ \hbvthetaIEpsEdge$ scale as $\bcO(p^2)$. Also, the number of computations required for generating $\hat{E}$ scale as $O(p^2)$. Therefore, the overall computational complexity is $\bcO(p^2)$.
\end{proof}

\section{Proof of Theorem \ref{theorem:parameter}}
\label{sec:proof of theorem:parameter}
In this section, we prove Theorem \ref{theorem:parameter}. 
See Appendix \ref{subsec:error bound on edge-wise parameter estimation with grise} and Appendix \ref{subsec:error bound on node-wise parameter estimation} for two key lemmas required in the proof.\\

Recall the definitions of $\g = \thetaMax(k+k^2d)$, $\varphiMax = (1+\Xupper)  \max\{\phiMax,\phiMax^2\}$ and $c_1(\alpha)$ from Section \ref{sec:problem setup}.
\begin{proof}[Proof of Theorem \ref{theorem:parameter}]
	Let the number of samples satisfy
	\begin{align}
	n \geq \max\Big[ c_1\bigg(\min\bigg\{ \frac{\thetaMin}{3}, \alpha, \frac{\alpha}{2^{\frac54} d k \phiMax } \bigg\} \bigg) \log \bigg( \frac{2^{5/2} p k }{\alpha^2} \bigg), c_2(2^{-\frac14} \alpha) \Big]
	\end{align}
	For each $i \in [p], \hbthetaI$ is the estimate of node parameters obtained through robust Lasso.
	Using Lemma \ref{lemma:recover-node-terms} with $\alpha_2 = 2^{-\frac14} \alpha$, the following holds with probability at least $1 - \alpha^4/2 $,
	\begin{align}
	\|  \bthetaIStar - \hbthetaI\|_{\infty} & \leq 2^{-\frac14} \alpha, \hspace{1cm} \forall i \in [p] \\
	\implies \|  \bthetaIStar - \hbthetaI\|_{\infty} & \leq \alpha, \hspace{1cm} \forall i \in [p]	 \label{eq:error-bound-node-wise}
	\end{align}
	
	For each $i \in [p]$, $\hbvthetaIEps$ is an $\epsilon$-optimal solution of \eqref{eq:GRISE} and $\hbvthetaIEpsEdge = (\hat{\theta}_{ij}, j \neq i, j \in [p])$ is the estimate of edge parameters associated with node $i$. 
	Using Lemma \ref{lemma:recover-edge-terms} with $\alpha_1 = \alpha$ and $\delta = \alpha^4 / 2$, the following holds with probability at least $1 - \alpha^4 / 2 $,
	\begin{align}
	\| \bvthetaIEdgeStar - \hbvthetaIEpsEdge \|_{2} & \leq \alpha, \hspace{1cm} \forall i \in [p] \\
	\implies \| \bvthetaIEdgeStar - \hbvthetaIEpsEdge \|_{\infty} & \stackrel{(a)}{\leq} \alpha, \hspace{1cm} \forall i \in [p]  \label{eq:error-bound-edge-wise}
	\end{align}
	where $(a)$ follows because $\| \svbv \|_{\infty} \leq \| \svbv \|_{2}$ for any vector $\svbv$.
	
	Now $\hbtheta$ is the estimate of $\bthetaStar$ obtained after appropriately concatenating $\hbthetaI$ and $\hbvthetaIEpsEdge$ $\forall i \in [p]$. 
	Combining \eqref{eq:error-bound-node-wise} and \eqref{eq:error-bound-edge-wise}, we have
	\begin{align}
	\| \hbtheta - \bthetaStar \|_{\infty} \leq \alpha
	\end{align}
	with probability at least $1 - \alpha^4 $.
	Further, combining the computations from Lemma \ref{lemma:recover-edge-terms} and Lemma \ref{lemma:recover-node-terms}, the total number of computations scale as $\bcO(p^{2})$.
\end{proof}

\section{GISO: Special instance of the penalized surrogate likelihood}
\label{connections to the penalized surrogate likelihood}
In this section, we show that the GISO is a special case of the penalized surrogate likelihood introduced by Jeon et al. (2006) \cite{JeonL2006}. In other words, we provide the proof of Proposition \ref{prop:GISO_penalized_likelihood_connection}.

Consider nonparametric density estimation where densities are of the form $f_{\rvbx}(\svbx) = e^{\eta(\svbx)} / \int e^{\eta(\svbx)} d\svbx$ from i.i.d samples $\svbx^{(1)}, \cdots, \svbx^{(n)}$. 
To circumvent the computational limitation of the exact likelihood-based functionals, Jeon et al. (2006) \cite{JeonL2006} proposed to minimize penalized surrogate likelihood. The surrogate likelihood is defined as follows:
\begin{align}
\cl_n(\eta)  = \frac{1}{n} \sum_{t = 1}^{n} \exp\Big( -\eta(\svbx^{(t)}) \Big) + \int_{\svbx} \rho(\svbx) \times  \eta(\svbx)d\svbx
\end{align}
where $\rho(\cdot)$ is some known probability density function. As Proposition \ref{prop:GISO_penalized_likelihood_connection} establishes, GISO is a special case of the surrogate likelihood.
\begin{proof}[Proof of Proposition \ref{prop:GISO_penalized_likelihood_connection}]
	Recall that the conditional density of $\rvx_i$ given $\rvx_{-i}  = x_{-i}$ is as follows:
	\begin{align}
	\ConditionalDensityNodeI & \propto \exp \Big( \bvthetaIStarT \bvphiI(x_i; x_{-i})  \Big).
	\end{align}
	For a given $\rvx_{-i}  = x_{-i}$, estimation of the conditional density of $\rvx_i$ is equivalent to estimating $\bvthetaIStar$.
	
	For any $\bvtheta \in \Reals^{k+k^2(p-1)}$, let us denote the surrogate likelihood associated with the conditional density of $\rvx_i$ by $\cl_{n}^{(i)}(\bvtheta)$. We have
	\begin{align}
	\cl_{n}^{(i)}(\bvtheta)  = \frac{1}{n} \sum_{t = 1}^{n} \exp\Big( -\bvthetaT \bvphiI(\svbx^{(t)}) \Big) + \int_{x_i \in \cX_i} \rho(x_i)  \times \Big(\bvthetaT \bvphiI(x_i ; x_{-i}) \Big) dx_i, \label{eq:penalized surrogate likelihood_conditional_density}
	\end{align}
	Let $\rho(\cdot)$ be the uniform density over $\cX_i$. Recall that the basis functions, $\bvphiI(x_i ; x_{-i})$, are locally centered and their integral with respect to $x_i$ is 0. Therefore, \eqref{eq:penalized surrogate likelihood_conditional_density} can be written as
	\begin{align}
	\cl_{n}^{(i)}(\bvtheta)  = \frac{1}{n} \sum_{t = 1}^{n} \exp\Big( -\bvthetaT \bvphiI(\svbx^{(t)}) \Big) = \cS_{n}^{(i)}(\bvtheta). \label{eq:penalized surrogate likelihood GISO connection}
	\end{align}
\end{proof} 
As we see in the proof above, the equivalence between the GISO and the surrogate likelihood occurs only the integral in \eqref{eq:penalized surrogate likelihood_conditional_density} is zero. As stated in Jeon et al. (2006) \cite{JeonL2006}, $\rho(\cdot)$ can be chosen to be equal to any known density and the choice typically depends on mathematical simplicity. Therefore, this provides a motivation to locally center the basis functions to simplify the exposition.

\section{Supporting propositions for Lemma \ref{lemma:recover-edge-terms}}
\label{sec:supporting propositions for lemma:recover-edge-terms}%LABEL
In this section, we will state the two key propositions required in the proof of Lemma \ref{lemma:recover-edge-terms}.
The proof of Lemma \ref{lemma:recover-edge-terms} is given in Appendix \ref{sec:proof of lemma:recover-edge-terms}.\\

Recall the definitions of $\g = \thetaMax(k+k^2d)$ and $\varphiMax = (1+\Xupper)  \max\{\phiMax,\phiMax^2\}$ from Section \ref{sec:problem setup}. 
For any $i \in [p]$, let $\nabla \cS_{n}^{(i)}(\bvthetaIStar) $ denote the gradient of the GISO for node $i$ evaluated at $\bvthetaIStar$.

\subsection{Bounds on the gradient of the GISO} 
\label{subsec:bounds on the gradient of the giso}
The following proposition shows that, with enough samples, the $\ell_{\infty}$-norm of the gradient of the GISO is bounded with high probability.
\begin{proposition} \label{prop:gradient-concentration-GISO}
	Consider any $i \in [p]$. For any $\deltaOne \in (0,1)$, any $\epsOne > 0$, the components of the gradient of the GISO are bounded from above as 
	\begin{align}
	\| \nabla \cS_{n}^{(i)}(\bvthetaIStar) \|_{\infty} \leq \epsOne
	\end{align}
	with probability at least $1 - \deltaOne$ as long as
	\begin{align}
	n > \frac{2\varphiMax^2 \exp(2\g\varphiMax)}{\epsOne^2} \log\bigg(\frac{2p^2k^2}{\deltaOne}\bigg) ~=~ 
	\Omega\Bigg(\frac{\exp(\Theta(k^2 d))}{\epsOne^2} \log\bigg(\frac{pk}{\sqrt{\deltaOne}}\bigg) \Bigg).	
	\end{align}
\end{proposition}
The proof of proposition \ref{prop:gradient-concentration-GISO} is given in Appendix \ref{sec:proof of proposition:gradient-concentration-giso}.

\subsection{Restricted Strong Convexity for GISO} 
\label{subsec:restricted strong convexity for giso}
Consider any $\bvtheta \in \parameterSet$. 
Let $\Delta = \bvtheta - \bvthetaIStar$. 
Define the residual of the first-order Taylor expansion as 
\begin{align}
\delta \cS_{n}^{(i)}(\Delta, \bvthetaIStar) = \cS_{n}^{(i)}(\bvthetaIStar + \Delta) - \cS_{n}^{(i)}(\bvthetaIStar)  - \langle\nabla \cS_{n}^{(i)}(\bvthetaIStar),\Delta \rangle.  \label{eq:residual}
\end{align}
Recall that $\bvthetaIEdgeStar$ denote the component of $\bvthetaIStar$ associated with the edge potentials. Let $\bvthetaEdge$ denote the component of $\bvtheta$ associated with the edge potentials and let $\Delta_{E}$ denote the component of $\Delta$ associated with the edge potentials i.e., $\Delta_{E} = \bvthetaEdge - \bvthetaIEdgeStar$.

The following proposition shows that, with enough samples, the GISO obeys a property analogous to the restricted strong convexity with high probability.
\begin{proposition}\label{prop:rsc_giso}
	Consider any $i \in [p]$. For any $\deltaTwo \in (0,1)$, any $\epsTwo > 0$, the residual of the first-order Taylor expansion of the GISO satisfies
	\begin{align}
	\delta \cS_{n}^{(i)}(\Delta, \bvthetaIStar) \geq \exp(-\g \varphiMax)\frac{ \frac{\kappa}{2\pi e (d+1)}  \|\Delta_{E}\|_2^2 - \epsTwo \|\Delta\|_1^2}{2 +\varphiMax \|\Delta\|_1}.
	\end{align}
	with probability at least $1-\deltaTwo$ as long as 
	\begin{align}
	n > \frac{2 \varphiMax^2}{\epsTwo^2}\log\Big(\frac{2p^3 k^4}{\deltaTwo}\Big) ~=~ 
	\Omega\Bigg(\frac{1}{\epsTwo^2} \log\bigg(\frac{p^3k^4}{\deltaTwo}\bigg) \Bigg).	
	\end{align}
\end{proposition}
The proof of proposition \ref{prop:rsc_giso} is given in Appendix \ref{sec:proof of proposition:rsc_giso}.

\section{Proof of Lemma \ref{lemma:recover-edge-terms}}
\label{sec:proof of lemma:recover-edge-terms}%LABEL
In this section, we prove Lemma \ref{lemma:recover-edge-terms}.
See Appendix \ref{subsec:bounds on the gradient of the giso} and Appendix \ref{subsec:restricted strong convexity for giso} for two key propositions required in the proof.\\

Recall the definitions of $\g = \thetaMax(k+k^2d)$, $\varphiMax = (1+\Xupper)  \max\{\phiMax,\phiMax^2\}$ and $c_1(\alpha)$ from Section \ref{sec:problem setup} and the definition of $c_3(\alpha)$ from Section \ref{sec:supporting lemmas for theorem}.
Recall that $\hbvthetaIEps$ is an $\epsilon$-optimal solution of the GISO.

For any $i \in [p]$, let $\nabla \cS_{n}^{(i)}(\bvthetaIStar) $ denote the gradient of the GISO for node $i$ evaluated at $\bvthetaIStar$.
Define $\Delta = \hbvthetaIEps - \bvthetaIStar$ and let $\Delta_{E}$ denote the component of $\Delta$ associated with the edge potentials i.e., $\Delta_{E} = \hbvthetaIEpsEdge - \bvthetaIEdgeStar$.
Recall from \eqref{eq:residual} that $\delta \cS_{n}^{(i)}(\Delta, \bvthetaIStar)$ denotes the residual of the first-order Taylor expansion.
\begin{proof}[Proof of Lemma~\ref{lemma:recover-edge-terms}]
	Consider any $i \in [p]$. Let the number of samples satisfy
	\begin{align}
	n \geq  c_1(\alpha_1) \times \log\bigg(\frac{2pk}{\sqrt{\delta}}\bigg)
	\end{align}
	We have from \eqref{eq:eps-opt-GRISE}
	\begin{align}
	\epsilon &\geq \cS_{n}^{(i)}(\hbvthetaIEps) - \min_{\bvtheta \in \parameterSet: \|\bvtheta\| \leq \g} \cS_{n}^{(i)}(\bvtheta)  \\
	&\stackrel{(a)}{\geq} \cS_{n}^{(i)}(\hbvthetaIEps) - \cS_{n}^{(i)}(\bvthetaIStar) \\
	&\stackrel{(b)}{=} \langle\nabla \cS_{n}^{(i)}(\bvthetaIStar),\Delta \rangle +  \delta \cS_{n}^{(i)}(\Delta, \bvthetaIStar)\\
	&\geq - \|\nabla \cS_{n}^{(i)}(\bvthetaIStar) \|_{\infty} \|\Delta\|_{1} + \delta \cS_{n}^{(i)}(\Delta, \bvthetaIStar).
	\end{align}
	where $(a)$ follows because $\bvthetaIStar \in \parameterSet$ and $\|\bvthetaIStar\| \leq \g$ and $(b)$ follows from \eqref{eq:residual}.
	Using the union bound on Proposition \ref{prop:gradient-concentration-GISO} and Proposition \ref{prop:rsc_giso} with $\deltaOne=\frac{\delta}{2}$ and $\deltaTwo=\frac{\delta}{2}$ respectively, we have with probability at least $1-\delta$,
	\begin{align}
	\epsilon \geq - \epsOne \|\Delta\|_{1} + \exp(-\g \varphiMax)\frac{ \frac{\kappa}{2\pi e (d+1)}  \|\Delta_{E}\|_2^2 - \epsTwo \|\Delta\|_1^2}{2 +\varphiMax \|\Delta\|_1}
	\end{align}
	This can be rearranged as
	\begin{align}
	\|\Delta_{E}\|_2^2 \leq \frac{2\pi e (d+1)}{\kappa} \bigg[ \exp(-\g \varphiMax) \times \Big( \epsilon + \epsOne\|\Delta\|_1 \Big) \times \Big( 2 + \varphiMax\|\Delta\|_1  \Big) + \epsTwo\|\Delta\|_1^2 \bigg]
	\end{align}
	Using $\|\bvthetaIStar\|_1 \leq \g$, $\|\hbvthetaIEps\|_1 \leq \g$ and the triangle inequality, we see that $\|\Delta\|_1$ is bounded by $2\g$. 
	By choosing 
	\begin{align}
	\epsilon \leq \frac{\kappa \alpha_1^{2} \exp(-\g \varphiMax)}{ 16 \pi e (d+1) (1+\varphiMax \g)}, \epsOne \leq \frac{\kappa \alpha_1^{2} \exp(-\g \varphiMax)}{ 32 \pi e (d+1) \g (1+\varphiMax \g)}, \epsTwo \leq \frac{\kappa \alpha_1^{2}}{16\pi e (d+1) \g^2},
	\end{align}
	and after some algebra, we obtain that
	\begin{align}
	\|\Delta_{E}\|_2 \leq \alpha_1.
	\end{align}
	Using Proposition \ref{prop:computational_complexity_entropic_descent}, the number of computations required to compute $\hbvthetaIEps$ scale as
	\begin{align}
	\frac{k^2 \g^2\varphiMax^2\exp(2\g\varphiMax) np}{\epsilon^2} \times \log{(2k^2p)}
	\end{align}
	Substituting for $\epsilon$, $n$ and observing that we need to compute the $\epsilon$-optimal estimate for every node, the total number of computations scale as
	\begin{align}
	c_3(\alpha_1)  \times \log\bigg(\frac{2pk}{\sqrt{\delta}}\bigg) \times \log{(2k^2p)}  \times p^2
	\end{align}
\end{proof}

\section{Proof of Proposition \ref{prop:gradient-concentration-GISO}}
\label{sec:proof of proposition:gradient-concentration-giso}%LABEL
In this section, we prove Proposition \ref{prop:gradient-concentration-GISO}.

Recall the definitions of $\g = \thetaMax(k+k^2d)$ and $\varphiMax = (1+\Xupper)  \max\{\phiMax,\phiMax^2\}$ from Section \ref{sec:problem setup}.
Also, recall the definition of GISO from \eqref{eq:GISO-main}.

For any $l \in [k+k^2(p-1)]$, let $\bvthetaIStar_{l}$ denote the $l^{th}$ component of $\bvthetaIStar$ and $\bvphiI_{l}(x_i^{(t)} ; x_{-i}^{(t)})$ denote the $l^{th}$ component of $\bvphiI(x_i^{(t)} ; x_{-i}^{(t)})$. Define the following random variable:
\begin{align}
\rvx_{i,l} \coloneqq - \bvphiI_{l}(\rvx_i ; \rvx_{-i})\exp\Big( -\bvthetaIStarT \bvphiI(\rvx_i ; \rvx_{-i}) \Big)  \label{eq:distribution-gradient-GISO}
\end{align}

\subsection{Supporting Lemma for Proposition \ref{prop:gradient-concentration-GISO}}
\label{supporting lemma for proposition:gradient-concentration-GISO}
The following Lemma shows that the expectation of the random variable $\rvx_{i,l}$ defined above is zero.

\begin{lemma} \label{lemma:expectation-gradient-GISO}
	For any $i \in [p]$ and $l \in [k+k^2(p-1)]$, we have
	\begin{align}
	\Expectation[\rvx_{i,l} ] = 0
	\end{align}
	where the expectation is with respect to $\DensityParametrizedTrue$.
\end{lemma}
\begin{proof}[Proof of Lemma \ref{lemma:expectation-gradient-GISO}]
	Fix $i \in [p]$ and $l \in [k+k^2(p-1)]$. Using $\eqref{eq:distribution-gradient-GISO}$ and Bayes theorem, we have
	\begin{align}
	\Expectation[\rvx_{i,l} ] & = - \int_{\svbx \in \cX} \bvphiI_{l}(x_i ; x_{-i})\exp\Big( -\bvthetaIStarT \bvphiI(x_i ; x_{-i}) \Big)  \ConditionalDensityNodeI \DensityParametrizedNotITrue d\svbx
	\end{align}
	Using \eqref{eq:conditionalDensity3} results in
	\begin{align}
	\Expectation[\rvx_{i,l} ] & = \frac{- \int_{\svbx \in \cX} \bvphiI_{l}(x_i ; x_{-i})  \DensityParametrizedNotITrue d\svbx}{\int_{x_i \in \cX_i} \exp \Big( \bvthetaIStarT \bvphiI(x_i ; x_{-i})  \Big)  dx_i}
	\end{align}
	Recall the fact that the basis functions are locally centered with respect to $\rvx_i$ and their integral is zero. Therefore, $\Expectation[\rvx_{i,l} ] = 0$.
\end{proof}

\subsection{Proof of Proposition \ref{prop:gradient-concentration-GISO}}
\label{subsec:proof of proposition gradient-concentration-GISO}
\begin{proof}[Proof of Proposition \ref{prop:gradient-concentration-GISO}]
	Fix $i \in [p]$ and $l \in [k+k^2(p-1)]$. 
	We start by simplifying the gradient of the GISO evaluated at $\bvthetaIStar$. 
	The $l^{th}$ component of the gradient of the GISO evaluated at $\bvthetaIStar$ is given by
	\begin{align}
	\frac{\partial \cS_{n}^{(i)}(\bvthetaIStar)}{\partial \bvthetaIStar_{l}}  = \frac{1}{n} \sum_{t = 1}^{n} - \bvphiI_{l}(x_i^{(t)} ; x_{-i}^{(t)})\exp\Big( -\bvthetaIStarT \bvphiI(x_i^{(t)} ; x_{-i}^{(t)}) \Big)  \label{eq:gradient-GISO}
	\end{align}
	Each term in the above summation is distributed as the random variable $\rvx_{i,l} $.
	The random variable $\rvx_{i,l}$ has zero mean (Lemma \ref{lemma:expectation-gradient-GISO}) and is bounded as follows:
	\begin{align}
	\Big|\rvx_{i,l}\Big| =  \Big|\bvphiI_{l}(\rvx_i ; \rvx_{-i}) \Big| \times \exp\Big( -\bvthetaIStarT \bvphiI(\rvx_i ; \rvx_{-i}) \Big) \stackrel{(a)}{\leq} \varphiMax \exp(\g \varphiMax)
	\end{align}
	where $(a)$ follows from \eqref{eq:basisUpperBound} and \eqref{eq:boundsExp}. 
	Using the Hoeffding's inequality, we have
	\begin{align}
	\Prob\bigg( \bigg| \frac{\partial \cS_{n}^{(i)}(\bvthetaIStar)}{\partial \bvthetaIStar_{l}}  \bigg| > \epsOne \bigg) < 2\exp\bigg(-\frac{n\epsOne^2}{2\varphiMax^2 \exp(2\g \varphiMax)}\bigg)  \label{eq:Hoeffding-gradient-GISO}
	\end{align}
	The proof follows by using \eqref{eq:Hoeffding-gradient-GISO}, the union bound over all $i \in [p]$ and $l \in [k+k^2(p-1)]$, and the fact that $k + k^2(p-1) \leq k^2p$.
\end{proof}

\section{Proof of Proposition \ref{prop:rsc_giso}}
\label{sec:proof of proposition:rsc_giso}%LABEL
In this section, we prove Proposition \ref{prop:rsc_giso}.

Recall the definitions of $\g = \thetaMax(k+k^2d)$ and $\varphiMax = (1+\Xupper)  \max\{\phiMax,\phiMax^2\}$ from Section \ref{sec:problem setup}.

For any $l \in [k+k^2(p-1)]$, let $\bvphiI_{l}(\rvx_i ; \rvx_{-i})$ denotes the $l^{th}$ component of $\bvphiI(\rvx_i ; \rvx_{-i})$. For any $\bvtheta \in \parameterSet$, let $\Delta = \bvtheta - \bvthetaIStar$. 
Let $\Delta_{E}$ denote the component of $\Delta$ associated with the edge potentials. 
Recall from \eqref{eq:residual} that $\delta \cS_{n}^{(i)}(\Delta, \bvthetaIStar)$ denotes the residual of the first-order Taylor expansion.

\subsection{Functional inequality}
\label{subsec:functional-inequality}
We start by stating the following deterministic functional inequality derived in Vuffray et al. (2016) \cite{VuffrayMLC2016}.
\begin{lemma}\label{lemma:functional_inequality}
	The following inequality holds for all $z \in \Reals$.
	\begin{align}
	e^{-z} - 1 + z \geq \frac{z^2}{2 + |z|}
	\end{align}
\end{lemma}
See Lemma 5 in Vuffray et al. (2016) \cite{VuffrayMLC2016} for proof.

\subsection{Correlation between locally centered basis functions}
\label{subsec:correlation between locally centered basis functions}
For any $l_1,l_2 \in [k+k^2(p-1)]$ let $H_{l_1l_2}$ denote the correlation between $\bvphiI_{l_1}(\rvbx)$ and $\bvphiI_{l_2}(\rvbx)$ defined as
\begin{align}
H_{l_1l_2} = \Expectation \Big[\bvphiI_{l_1}(\rvbx)\bvphiI_{l_2}(\rvbx)\Big],  \label{eq:correlation}
\end{align}
and let $\bH = [H_{l_1l_2}] \in \Reals^{[k+k^2(p-1)] \times [k+k^2(p-1)]}$ be the corresponding correlation matrix. 
Similarly, we define $\hbH$  based on the empirical estimates of the correlation i.e., $\hH_{l_1l_2} = \frac{1}{n} \sum_{t=1}^{n} \bvphiI_{l_1}(\svbx^{(t)})\bvphiI_{l_2}(\svbx^{(t)})$. 

The following lemma bounds the deviation between the true correlation and the empirical correlation.
\begin{lemma} \label{lemma:correlation_concentration}
	Consider any $i \in [p]$ and $l_1, l_2 \in [k+k^2(p-1)]$. Then, we have for any $\epsTwo > 0$,
	\begin{align}
	|\hH_{l_1l_2} - H_{l_1l_2}| < \epsTwo,
	\end{align}
	with probability at least $1 - 2p^3 k^4\exp \left(-\frac{n \epsTwo^2}{2\varphiMax^2}\right)$.
\end{lemma}

\begin{proof}[Proof of Lemma~\ref{lemma:correlation_concentration}]
	Fix $i \in [p]$ and $l_1, l_2 \in [k+k^2(p-1)]$. 
	The random variable defined as $Y_{l_1l_2} \coloneqq \bvphiI_{l_1}(\rvbx)\bvphiI_{l_2}(\rvbx)$ satisfies $|Y_{l_1l_2}| \leq \varphiMax^2$. 
	Using the Hoeffding inequality we get
	\begin{align}
	\Prob \left( |\hH_{l_1l_2} - H_{l_1l_2}| > \epsTwo \right) < 2\exp \left(-\frac{n \epsTwo^2}{2\varphiMax^2}\right).
	\end{align}
	The proof follows by using the union bound over all $i \in [p]$ and $l_1,l_2 \in [k+k^2(p-1)]$, and the fact that $k + k^2(p-1) \leq k^2p$.
\end{proof}

\subsection{Supporting Lemma for Proposition \ref{prop:rsc_giso}}
\label{subsec:lower bounds on the residual}
The following Lemma provides a lower bound on the residual defined in \eqref{eq:residual} i.e., $\delta \cS_{n}^{(i)}(\Delta, \bvthetaIStar)$.
\begin{lemma} \label{lemma:rsc_interim}
	Consider any $i \in [p]$. The residual of the first-order Taylor expansion of the \textsc{GISO} satisfies
	\begin{align}
	\delta \cS_{n}^{(i)}(\Delta, \bvthetaIStar) \geq \exp(-\g \varphiMax) 
	\frac{\Delta^T \hbH \Delta}{2 +\varphiMax \|\Delta\|_1}. 
	\end{align}
\end{lemma}

\begin{proof}[Proof of Lemma~\ref{lemma:rsc_interim}]
	Fix any $i \in [p]$. Substituting \eqref{eq:GISO-main} and \eqref{eq:gradient-GISO} in \eqref{eq:residual}, we have
	\begin{align}
	\delta \cS_{n}^{(i)}(\Delta, \bvthetaIStar) &= \frac{1}{n}\sum_{t=1}^n \exp\Big( -\bvthetaIStarT \bvphiI(x_i^{(t)} ; x_{-i}^{(t)}) \Big) \\
	& \qquad \times \Big(\exp\Big( -\Delta^T \bvphiI(x_i^{(t)} ; x_{-i}^{(t)}) \Big)  - 1 + \Delta^T \bvphiI(x_i^{(t)} ; x_{-i}^{(t)})\Big)  \\
	& \stackrel{(a)}{\geq} \exp(-\g \varphiMax) \frac{1}{n}\sum_{t=1}^n 
	\frac{\Big(\Delta^T \bvphiI(x_i^{(t)} ; x_{-i}^{(t)})\Big)^2}{2 + |\Delta^T \bvphiI(x_i^{(t)} ; x_{-i}^{(t)}) |}\\
	& \stackrel{(b)}{\geq} \exp(-\g \varphiMax) 
	\frac{\Delta^T \hbH \Delta}{2 +\varphiMax \|\Delta\|_1}
	\end{align}
	where $(a)$ follows by using \eqref{eq:boundsExp} and Lemma~\ref{lemma:functional_inequality} with $z = \Delta^T \bvphiI(x_i^{(t)} ; x_{-i}^{(t)})$, and $(b)$ follows by using \eqref{eq:basisUpperBound}, the defintion of $\hbH$, and observing that $\forall$ $t \in [n], |\Delta^T \bvphiI(x_i^{(t)} ; x_{-i}^{(t)}) | \leq \varphiMax \|\Delta\|_1$.
\end{proof}

\subsection{Proof of Proposition \ref{prop:rsc_giso}}
\label{subsec:proof of proposition rsc_giso}
\begin{proof}[Proof of Proposition~\ref{prop:rsc_giso}]
	Consider any $i \in [p]$. Using Lemma~\ref{lemma:rsc_interim} we have
	\begin{align}
	\delta \cS_{n}^{(i)}(\Delta, \bvthetaIStar)  &\geq \exp(-\g \varphiMax)\frac{\Delta^T \hbH \Delta}{2 +\varphiMax \|\Delta\|_1} \\
	&= \exp(-\g \varphiMax)\frac{\Delta^T {\bH} \Delta + \Delta^T (\hbH - \bH) \Delta}{2 +\varphiMax \|\Delta\|_1}
	\end{align}
	Let the number of samples satisfy
	\begin{align}
	n > \frac{2 \varphiMax^2}{\epsTwo^2}\log\Big(\frac{2p^3 k^4}{\deltaTwo}\Big)
	\end{align}
	Using Lemma~\ref{lemma:correlation_concentration} we have
	\begin{align}
	\delta \cS_{n}^{(i)}(\Delta, \bvthetaIStar)  & \geq \exp(-\g \varphiMax)\frac{\Delta^T {\bH} \Delta - \epsTwo \|\Delta\|_1^2}{2 +\varphiMax \|\Delta\|_1}  \label{eq:rsc_grise}
	\end{align}
	with probability at least $1 - \deltaTwo$.
	
	Now we will lower bound $\Delta^T {\bH} \Delta$. 
	First, let us unroll the vector $\Delta$ such that $\Delta^{(i)} \in \Reals^k$ is associated with $\bphiI(x_i)$ and $\forall j \in [p] \setminus \{i\},$ $\Delta^{(ij)} \in \Reals^{k^2}$ is associated with $\bpsiI(x_i, x_j)$. Recall that $\Delta_{E}$ is the component of $\Delta$ associated with the edge potentials i.e., 
	\begin{align}
	\Delta_{E} = [\Delta^{(ij)} \in \Reals^{k^2} : j \in [p], j \neq i ]  \label{eq:delta_edge-wise}
	\end{align}
	Using \eqref{eq:correlation} we have
	\begin{align}
	\Delta^T {\bH} \Delta = \Expectation \bigg[\Big( \Delta^T \bvphiI(\rvbx) \Big)^2\bigg] \stackrel{(a)}{\geq} \Variance\big[ \Delta^T \bvphiI(\rvbx) \big]  \label{eq:var_lower_bound}
	\end{align}
	where $(a)$ follows from the fact that for any random variable $Z, \Expectation[Z^2] \geq \Variance[Z]$.
	
	Now consider the graph $G_{-i} (\bthetaStar)$ obtained from the graph $G(\bthetaStar)$ by removing the node $i$ and all the edges associated with it. 
	We will next choose an independent set of the graph $G_{-i} (\bthetaStar)$ with a special property. 
	Let $r_1 \in [p] \setminus \{i\}$  be such that $\| \Delta^{(ir_1)} \|_2 \geq \| \Delta^{(ij)} \|_2$ $ \forall j \in [p] \setminus \{i , r_1\}$. 
	Let $r_2 \in [p] \setminus \{i, r_1, \cN(r_1)\}$ be such that $\| \Delta^{(ir_2)} \|_2 \geq \| \Delta^{(ij)} \|_2$ $ \forall j \in [p] \setminus \{i , r_1, \cN(r_1), r_2\}$, and so on. 
	Denote by $m \geq p / (d+1)$ the total number of nodes selected in this manner, and let $\cR = \{ r_1, \cdots , r_m\}$. 
	It is easy to see that $\cR$ is independent set of the graph $G_{-i} (\bthetaStar)$ with the following property:
	\begin{align}
	\sum_{ j \in \cR } \| \Delta^{(ij)} \|_2^2 \geq \frac{1}{d+1}\sum_{ j \in [p] , j\neq i } \| \Delta^{(ij)} \|_2^2  \label{eq:independentSetProperty}
	\end{align} 
	Let $\cR^c = [p] \setminus \{i , \cR\}$. Using the law of total variance and conditioning on $\cR^c$, we can rewrite \eqref{eq:var_lower_bound} as
	\begin{align}
	\Delta^T {\bH} \Delta & \geq \Expectation \bigg[ \Variance\big[ \Delta^T \bvphiI(\rvbx) \big|  \rvx_i,  \rvx_{\cR^c} \big] \bigg] \\
	& \stackrel{(a)}{=} \Expectation \bigg[ \Variance\Big[ \Delta^{(i)^T} \bphiI(\rvx_i) + \sum_{ j \in [p], j \neq i} \Delta^{(ij)^T} \bpsiI(\rvx_i,\rvx_j)	\big| \rvx_i , \rvx_{\cR^c} \Big] \bigg] \\
	& \stackrel{(b)}{=} \Expectation \bigg[ \Variance\Big[ \sum_{ j \in \cR} \Delta^{(ij)^T} \bpsiI(\rvx_i,\rvx_j)	\big|  \rvx_i , \rvx_{\cR^c} \Big] \bigg] \\
	& \stackrel{(c)}{=} \Expectation \bigg[ \sum_{ j \in \cR}   \Variance\Big[ \Delta^{(ij)^T} \bpsiI(\rvx_i,\rvx_j)	\big|  \rvx_i , \rvx_{\cR^c} \Big] \bigg] \\
	& \stackrel{(d)}{=} \sum_{ j \in \cR}  \Expectation \bigg[   \Variance\Big[ \Delta^{(ij)^T} \bpsiI(\rvx_i,\rvx_j)	\big|  \rvx_i , \rvx_{\cR^c} \Big] \bigg] \\
	& \stackrel{(e)}{=} \sum_{ j \in \cR}  \Expectation \bigg[   \Variance\Big[ \Delta^{(ij)^T} \bpsiI(\rvx_i,\rvx_j)	\big| \rvx_{\cN(j)} \Big] \bigg] \\
	& \stackrel{(f)}{=} \sum_{ j \in \cR}  \Expectation \bigg[   \Variance\Big[ \Delta^{(ij)^T} \bpsiI(\rvx_i,\rvx_j)	\big| \rvx_{-j} \Big] \bigg] \\
	&  \stackrel{(g)}{\geq} \frac{1}{2\pi e}\sum_{ j \in \cR}  \Expectation \bigg[ \exp\bigg\{2 \Entropy
	\Big[ \Delta^{(ij)^T} \bpsiI(\rvx_i,\rvx_j)	\big| \rvx_{-j}  \Big] \bigg\} \bigg] \\
	& \stackrel{(h)}{\geq}  \frac{\kappa}{2\pi e} \sum_{ j \in \cR} \| \Delta^{(ij)} \|_2^2\\
	& \stackrel{(i)}{\geq}  \frac{\kappa}{2\pi e (d+1)} \sum_{ j \in [p] , j\neq i } \| \Delta^{(ij)} \|_2^2\\
	& \stackrel{(j)}{=}  \frac{\kappa}{2\pi e (d+1)}  \| \Delta_{E} \|_2^2  \label{eq:lower_bound_using_condition}
	\end{align}
	where $(a)$ follows from the definition of $\bvphiI(\rvbx)$ from Section \ref{sec:problem setup}, $(b)$ follows because we have conditioned on $\rvx_{i}$ and $\rvx_{\cR^c}$ (note $(\rvx_j)_{j \in \cR^c}$ are constant given $\rvx_{\cR^c}$), $(c)$ follows because $(\rvx_j)_{j \in \cR}$ are conditionally independent given $\rvx_{\cR^c}$ (note that $\cR$ is an independent set in $G_{-i} (\bthetaStar)$, i.e. there is no edge connecting two vertices in $\cR$), $(d)$ follows from linearity of expectation, $(e)$ follows because $\rvx_{\cN(j)} \subseteq \rvx_{\cR^c} \cup \rvx_i$ $\forall j \in \cR$, $(f)$ follows from the global Markov property, $(g)$ follows from monotonicity of expectation and Shannon's entropy inequality $(\Entropy(\cdot) \leq \log\sqrt{2\pi e \Variance(\cdot)} )$, $(h)$ follows from \eqref{eq:condition}, $(i)$ follows from \eqref{eq:independentSetProperty}  and $(j)$ follows from \eqref{eq:delta_edge-wise}. 
	
	Plugging this back in \eqref{eq:rsc_grise} we have
	\begin{align}
	\delta \cS_{n}^{(i)}(\Delta, \bvthetaIStar) & \geq \exp(-\g \varphiMax)\frac{ \frac{\kappa}{2\pi e (d+1)}  \|\Delta_{E}\|_2^2 - \epsTwo \|\Delta\|_1^2}{2 +\varphiMax \|\Delta\|_1}.
	\end{align}
\end{proof}

\section{The Generalized Interaction Screening algorithm}
\label{sec:the generalized interaction screening algorithm}
In this section, we describe the \textit{Generalized Interaction Screening} algorithm for the setup in Section \ref{sec:problem setup} and also provide its computational complexity. 

Recall the definitions of $\g = \thetaMax(k+k^2d)$ and $\varphiMax = (1+\Xupper)  \max\{\phiMax,\phiMax^2\}$ from Section \ref{sec:problem setup}.

\subsection{The \textit{Generalized Interaction Screening} algorithm}
\label{the generalized interaction screening algorithm}
Vuffray et al. (2019) \cite{VuffrayML2019} showed that an $\epsilon$-optimal solution of GRISE could be obtained by first finding an $\epsilon$-optimal solution of the unconstrained GRISE using a variation of the Entropic Descent Algorithm and then projecting the solution onto $\parameterSet$. 
See Lemma 4 of Vuffray et al. (2019) \cite{VuffrayML2019} for more details.\\

For $\epsilon > 0$, $\hbvthetaIEpsUnc$ is an $\epsilon$-optimal solution of the unconstrained GRISE for $i \in [p]$ if
\begin{align}
\cS_{n}^{(i)}(\hbvthetaIEpsUnc) \leq \min_{\bvtheta: \|\bvtheta\|_1 \leq \g} \cS_{n}^{(i)}(\bvtheta) + \epsilon  \label{eq:eps-opt-unconstrained-GRISE}
\end{align}
The iterative Algorithm \ref{alg:entropic_descent} outputs an $\epsilon$-optimal solution of GRISE without constraints in \eqref{eq:eps-opt-unconstrained-GRISE}. 
This algorithm is an application of the Entropic Descent Algorithm introduced in \cite{beck2003} to a reformulation of \eqref{eq:GRISE} as a minimization over the probability simplex.

\begin{algorithm}
	\caption{Entropic Descent for unconstrained GRISE}\label{alg:entropic_descent}
	\begin{algorithmic}[1]
		\State \textbf{Input:} $k,p,\g,\varphiMax,\cS_{n}^{(i)}(\cdot),T$
		\State \textbf{Output:} $\hbvthetaIEpsUnc$ 
		\State \textbf{Initialization:}  
		\State \hskip2.0em $w^{(1)}_{l,+} \leftarrow e/(2k^2(p-1) + 2k + 1)$, $\forall l \in [k^2(p-1) + k]$
		\State \hskip2.0em $w^{(1)}_{l,-} \leftarrow e/(2k^2(p-1) + 2k + 1)$, $\forall l \in [k^2(p-1) + k]$
		\State \hskip2.0em $y^{(1)} \leftarrow e/(2k^2(p-1) + 2k + 1)$
		\State \hskip2.0em  $\eta^{(1)} \leftarrow \sqrt{\log{(2k^2(p-1) + 2k + 1)}} / 2\g\varphiMax\exp(\g\varphiMax)$
		\For{$t = 1,\cdots , T$}
		\State $\svbw^{(t)}_{+} = (w^{(t)}_{l,+} : l \in [k^2(p-1) + k])$
		\State $\svbw^{(t)}_{-} = (w^{(t)}_{l,-} : l \in [k^2(p-1) + k])$
		\State $v_{l} = \g \dfrac{\partial \cS_{n}^{(i)}(\g(\svbw^{(t)}_{+} - \svbw^{(t)}_{-}))}{\partial \bvtheta_{l}} $, $\forall l \in [k^2(p-1) + k]$
		\State $x_{l,+} = w^{(t)}_{l,+} \exp(- \eta^t v_l)$, $\forall l \in [k^2(p-1) + k]$
		\State $x_{l,-} = w^{(t)}_{l,-} \exp( \eta^t v_l)$, $\forall l \in [k^2(p-1) + k]$
		\State $z = y^{(t)} + \sum\limits_{l \in [k^2(p-1) + k]} (x_{l,+} + x_{l,-})$
		\State $w^{(t+1)}_{l,+} \leftarrow x_{l,+} / z$, $\forall l \in [k^2(p-1) + k]$
		\State $w^{(t+1)}_{l,-} \leftarrow x_{l,-} / z$, $\forall l \in [k^2(p-1) + k]$
		\State $y^{(t+1)} \leftarrow y^{(t)} / z$
		\State $\eta^{(t+1)} \leftarrow \eta^t \sqrt{t / t+1}$
		\EndFor
		\State $s = \argmin_{s=1,\dots,T} \cS_{n}^{(i)}(\g(\svbw^{(s)}_{+} - \svbw^{(s)}_{-}))$
		\State $\hbvthetaIEpsUnc \leftarrow \g(\svbw^{(s)}_{+} - \svbw^{(s)}_{-})$
	\end{algorithmic}
\end{algorithm}

\subsection{Computational Complexity of Algorithm \ref{alg:entropic_descent}}
\label{computational complexity of algorithm:entropic_descent}
The following proposition provides guarantees on the computational complexity of unconstrained GRISE.
\begin{proposition} \label{prop:computational_complexity_entropic_descent}
	Let $\epsilon > 0$ be the optimality gap. Let the number of iterations satisfy
	\begin{align}
	T \geq \frac{\g^2\varphiMax^2\exp(2\g\varphiMax)}{\epsilon^2} \times \log{(2k^2(p-1) + 2k + 1)} ~=~ 
	\Omega\Bigg(\frac{\exp(\Theta(k^2 d))}{\epsilon^2} \log (k^2 p) \Bigg).	
	\end{align}
	Then, Algorithm~\ref{alg:entropic_descent} is guaranteed to produce an $\epsilon$-optimal solution of GRISE without constraints in \eqref{eq:eps-opt-unconstrained-GRISE} with number of computations of the order
	\begin{align}
	\frac{k^2 \g^2\varphiMax^2\exp(2\g\varphiMax) np}{\epsilon^2} \times \log{(2k^2(p-1) + 2k + 1)} ~=~ 
	\Omega\Bigg(\frac{\exp(\Theta(k^2 d))}{\epsilon^2}  n  p  \log (k^2 p) \Bigg).	
	\end{align}
\end{proposition}
\begin{proof}[Proof of Proposition \ref{prop:computational_complexity_entropic_descent}]
	We first show that the minimization of GRISE when $\parameterSet = \Reals^{k^2(p-1) + k}$ (the unconstrained case) is equivalent to the following lifted minimization,
	\begin{align}
	\min_{\bvtheta, \svbw_{+} , \svbw_{-}, y} \quad & \cS_{n}^{(i)}(\bvtheta) \label{eq:obj} \\
	\text{s.t.} \quad & \bvtheta = \g (\svbw_{+} - \svbw_{-})  \label{eq:equality} \\  
	& y + \sum_{l \in [k^2(p-1) + k]} (w_{l,+} + w_{l,-}) = 1  \label{eq:sum_simplex} \\ 
	& y \geq 0, w_{l,+} \geq 0, w_{l,-} \geq 0, \forall l \in [k^2(p-1) + k].  \label{eq:positive_simplex}
	\end{align}
	where $\svbw_{+} = (w_{l,+} : l \in [k^2(p-1) + k])$ and $\svbw_{-} = (w_{l,-} : l \in [k^2(p-1) + k])$.
	
	We start by showing that for all $\bvtheta \in \Reals^{k^2(p-1) + k}$ such that $\| \bvtheta \|_1 \leq \g$, there exists $\svbw_{+} , \svbw_{-}, y$ satisfying constraints \eqref{eq:equality}, \eqref{eq:sum_simplex}, \eqref{eq:positive_simplex}. This is easily done by choosing $\forall l \in [k^2(p-1) + k]$, $w_{l,+}  = \max(\bvtheta_{l} / \g ,0)$ , $w_{l,-}  = \max(-\bvtheta_{l}/ \g ,0)$ and $y= 1 - \|\bvtheta\|_1 / \g$.
	
	Next, we trivially see that for all $\bvtheta, \svbw_{+} , \svbw_{-}, y$ satisfying constraints \eqref{eq:equality}, \eqref{eq:sum_simplex}, \eqref{eq:positive_simplex}, it implies that $\bvtheta$ also satisfies $\| \bvtheta \|_1 \leq \g$. 
	Therefore, any $\bvtheta$ that is an $\epsilon$-minimizer of \eqref{eq:obj} is also an $\epsilon$-minimizer of \eqref{eq:GRISE} without constraints. The remainder of the proof is a straightforward application of the analysis of the Entropic Descent Algorithm in \cite{beck2003} to the above minimization where $\bvtheta$ has been replaced by $\svbw_{+} , \svbw_{-}, y$ using \eqref{eq:equality}.
\end{proof}
The computational complexity of the projection step is usually insignificant compared to the computational complexity of Algorithm \ref{alg:entropic_descent} provided in Proposition \ref{prop:computational_complexity_entropic_descent}.

\section{Robust LASSO}
\label{sec:robust lasso_appendix}%LABEL
In this section, we present a robust variation of the sparse linear regression. 
More specifically, we show that even in the presence of bounded additive noise, the Lasso estimator is `prediction consistent' under almost no assumptions at all.
\subsection{Setup}
\label{subsec:setup}
Suppose that $\rvv_1, \cdots, \rvv_{\tp}$ (where $\tp \geq 1)$ are (possibly dependent) random variables, and suppose $\tcOne$ is a constant such that $|\rvv_r| \leq \tcOne$ almost surely for each $r \in [\tp]$. Let
\begin{align}
\rvy = \sum_{r = 1}^{\tp} \betaStar_r \rvv_r + \tBoundedNoise + \tepsilon  \label{eq:Lasso1}
\end{align}
where $\tBoundedNoise$ is bounded noise with $|\tBoundedNoise| \leq \tBoundedNoise_0$,  $\tepsilon$ is \textit{sub-Gaussian} noise with mean 0 and variance proxy $\tsigma^2$, and $\tepsilon$ is independent of the $\rvv_r$'s and $\tBoundedNoise$. 
Define $\bbetaStar \coloneqq (\betaStar_1, \cdots, \betaStar_{\tp})$. 
We also have the `sparsity' condition that $\|\bbetaStar\|_1 \leq \tcTwo$. 
Here $\betaStar_1, \cdots, \betaStar_{\tp} ,\tcTwo$, and $\tsigma$ are unknown constants.
\subsection{Data}
\label{subsec:data}
Let $\rvbv$ denote the random vector $(\rvv_1,\cdots,\rvv_{\tp})$. 
Let $\svbv_1,\cdots, \svbv_n$ be $n$ i.i.d copies of $\rvbv$ and let $\svby \coloneqq (y_1, \cdots, y_n)$ denote the corresponding true values of $\rvy$. 
Let $\bV$ be a $n \times \tp$ matrix such that the $j^{th}$ row is $\svbv_j$.

Suppose that our task is to predict $\rvy$ given the value of $\rvbv$. 
If the parameter vector $\bbetaStar$ was known, then the predictor of $\rvy$, of interest, based on $\rvbv$ would be $\hrvy \coloneqq \sum_{r = 1}^{\tp} \betaStar_r \rvv_r $. 
However, $\bbetaStar$ is unknown, and we need to estimate it from the data ($\bV$, $\svby$). Let $\tbbeta$ be the output of Algorithm  \ref{alg:LASSO1}. 
Let $\hsvby \coloneqq (\hy_1, \cdots, \hy_n)$ where
\begin{align}
\hy_j = \bbetaStar \cdot \svbv_j  \label{eq:lasso4}
\end{align}
Let $\tsvby \coloneqq (\ty_1, \cdots, \ty_n)$ where
\begin{align}
\ty_j = \tbbeta \cdot \svbv_j  \label{eq:lasso5}
\end{align}
\begin{algorithm}
	\caption{Robust LASSO}\label{alg:LASSO1}
	\begin{algorithmic}[1]
		\State \textbf{Input:} $\bV, \svby, \tcTwo$
		\State \textbf{Output:} $\tbbeta$ 
		\State $\tbbeta \leftarrow \argmin_{\bbeta : \|\bbeta\|_1 \leq \tcTwo} [\svby - \bV \cdot \bbeta]^T [\svby - \bV \cdot \bbeta]$
	\end{algorithmic}
\end{algorithm} 
\subsection{Prediction error}
\label{subsec:prediction error}
\begin{definition}\label{def:MSPE}
	The `mean square prediction error' of any estimator $\tbbeta \coloneqq (\tbeta_1,\cdots, \tbeta_{\tp})$ is defined as the expected squared error in estimating $\hrvy$ using $\tbbeta$, that is, 
	\begin{align}
	\text{MSPE}(\tbbeta) \coloneqq \Expectation_{\rvbv}(\hrvy - \trvy)^2,
	\end{align} 
	where $\trvy \coloneqq \sum_{r = 1}^{\tp} \tbeta_r \rvv_r $.
\end{definition}
\begin{definition}\label{def:estimated-MSPE}
	The `estimated mean square prediction error' of any estimator $\tbbeta \coloneqq (\tbeta_1,\cdots, \tbeta_{\tp})$ is defined
	\begin{align}
	\widehat{\text{MSPE}}(\tbbeta) \coloneqq \frac{1}{n} \sum_{ j \in [n]} (\hy_j - \ty_j)^2
	\end{align} 
\end{definition}
The following Lemma shows that the Lasso estimator of Algorithm \ref{alg:LASSO1} is `prediction consistent' even in presence of bounded noise if $\tcTwo$ is correctly chosen and $n \gg \tp$.
\begin{lemma}\label{lemma:lasso}
	Let $\tbbeta$ be the ouput of Algorithm  \ref{alg:LASSO1}. 
	Then,
	\begin{align}
	\Expectation [\widehat{\text{MSPE}}(\tbbeta)]  \leq 4\tBoundedNoise_0^2 + 4\tcOne\tcTwo\tsigma  \sqrt{\frac{2\log 2\tp}{n}} 
	\end{align}
	\begin{align}
	\text{MSPE}(\tbbeta) \leq 4\tBoundedNoise_0^2 + 4\tcOne\tcTwo\tsigma  \sqrt{\frac{2\log 2\tp}{n}} + 8\tcOne^2\tcTwo^2  \sqrt{\frac{2\log (2\tp^2)}{n}} 
	\end{align}
\end{lemma}

\subsection{Proof of Lemma \ref{lemma:lasso}}
\label{subsec:proof-of-lemma:lasso}
\begin{proof}[Proof of Lemma \ref{lemma:lasso}]
	$\hsvby$ is the vector of the best predictions of $\svby$ based on $\bV$ and $\tsvby$ is the vector of predictions of $\svby$ using Algorithm \ref{alg:LASSO1}. 
	Let $\svbv^{(j)}$ denote the $j^{th}$ column of $\bV$ $\forall j \in [\tp]$. 
	Let $v_{h,r}$ denote the $h^{th}$ element of $\svbv^{(r)}$ $\forall h \in [n], r \in [\tp]$. 
	Let $\tBoundedNoise_h$ $(\tepsilon_h)$ denote the bounded (\text{sub-Gaussian}) noise associated with $y_h$ $\forall h \in [n]$.
	
	Define the set
	\begin{align}
	\cY \coloneqq \big\{\sum_{r = 1}^{\tp} \tbeta_r \svbv^{(r)} : \sum_{r = 1}^{\tp} |\tbeta_r| \leq \tcTwo\big\}
	\end{align}
	Note that $\cY$ is a compact and convex subset of $\Reals^n$. 
	By definition, $\tsvby$ is the projection of $\svby$ on to the set $\cY$. 
	Because $\cY$ is convex and $\hsvby \in \cY$, we have from the Pythagorean theorem for projection onto a convex set,
	\begin{align}
	(\hsvby - \tsvby ) \cdot (\svby -\tsvby ) \leq 0
	\end{align}
	Adding and subtracting $\hsvby$ in the second term we get,
	\begin{align}
	\| \hsvby - \tsvby\|^2_2  & \leq (\svby - \hsvby) \cdot (\tsvby -  \hsvby)\\
	& \stackrel{(a)}{=}\sum_{h = 1}^n \big( \tBoundedNoise_h + \tepsilon_h\big) \big(\sum_{r= 1}^{\tp} (\tbeta_r - \betaStar_r) v_{h,r}  \big)\\
	& = \sum_{h = 1}^n \tBoundedNoise_h  \big(\sum_{r = 1}^{\tp} (\tbeta_r - \betaStar_r) v_{h,r}   \big)  + \sum_{h = 1}^n \tepsilon_h   \big(\sum_{r = 1}^{\tp} (\tbeta_r - \betaStar_r) v_{h,r}   \big) \\
	& = \sum_{h = 1}^n \tBoundedNoise_h  \big(\sum_{r = 1}^{\tp} (\tbeta_r - \betaStar_r) v_{h,r}   \big)  + \sum_{r = 1}^{\tp} (\tbeta_r - \betaStar_r) \big(\sum_{h = 1}^n \tepsilon_h    v_{h,r}   \big)  \label{eq:LASSO:Two-terms}
	\end{align}
	where $(a)$ follows from \eqref{eq:lasso4}, \eqref{eq:lasso5} and definitions of $\rvy$ and $\hrvy$.  
	
	Let us first focus on only the first term in \eqref{eq:LASSO:Two-terms}.
	\begin{align}
	\sum_{h = 1}^n \tBoundedNoise_h  \big(\sum_{r = 1}^{\tp} (\tbeta_r - \betaStar_r) v_{h,r}   \big) & \stackrel{(a)}{\leq}  \sum_{h = 1}^n | \tBoundedNoise_h| \big|\big(\sum_{r = 1}^{\tp} (\tbeta_r- \betaStar_r) v_{h,r}  \big) \big| \\
	& \stackrel{(b)}{\leq} \tBoundedNoise_0 \sum_{h = 1}^n  \bigg|\sum_{r = 1}^{\tp} (\tbeta_r - \betaStar_r) v_{h,r}  \bigg|  \label{eq:LASSO-norm}
	\end{align}
	where $(a)$ follows from the triangle inequality and $(b)$ follows because $\tBoundedNoise_h \leq \tBoundedNoise_0$ $\forall h \in [n]$. 
	Notice that $\sum_{h = 1}^n  \big|\sum_{r = 1}^{\tp} (\tbeta_r - \betaStar_r) v_{h,r}  \big|$ is the $\ell_1$ norm of the vector $\hsvby - \tsvby$.
	
	Let us now focus on the second term in the  \eqref{eq:LASSO:Two-terms}. 
	Using the facts that $\| \bbetaStar \|_1 \leq \tcTwo$ and $\| \tbbeta \|_1 \leq \tcTwo $, we have
	\begin{align}
	\sum_{r = 1}^{\tp} (\tbeta_r - \betaStar_r) \big(\sum_{h = 1}^n \tepsilon_h    v_{h,r}   \big)  & \leq 2\tcTwo \max_{1\leq r \leq \tp} |\rvu_r|  \label{eq:LASSO-max}
	\end{align}
	where 
	\begin{align}
	\rvu_r \coloneqq \sum_{h = 1}^n \tepsilon_h    v_{h,r}
	\end{align}
	Now plugging back the upper bounds from \eqref{eq:LASSO-norm} and \eqref{eq:LASSO-max} in \eqref{eq:LASSO:Two-terms} we get,
	\begin{align}
	\| \hsvby - \tsvby\|_2^2   & \leq \tBoundedNoise_0 \| \hsvby - \tsvby\|_1 + 2\tcTwo \max_{1\leq r \leq \tp} |\rvu_r|\\
	& \stackrel{(a)}{\leq} \tBoundedNoise_0 \sqrt{n} \| \hsvby - \tsvby\|_2 + 2\tcTwo \max_{1\leq r \leq \tp} |\rvu_r| \\
	& \stackrel{(b)}{\leq} 2 \max \big\{\tBoundedNoise_0 \sqrt{n} \| \hsvby - \tsvby\|_2, 2\tcTwo \max_{1\leq r \leq \tp} |\rvu_r| \big\}\\
	& \stackrel{(c)}{\leq} \max\{4\tBoundedNoise_0^2 n , 4\tcTwo \max_{1\leq r \leq \tp} |\rvu_r|\} \\
	& \stackrel{(d)}{\leq} 4\tBoundedNoise_0^2 n  + 4\tcTwo \max_{1\leq r \leq \tp} |\rvu_r|  \label{eq:LASSO:Two-terms-simplified}
	\end{align}
	where $(a)$ follows from the fact that $\| \hsvby - \tsvby\|_1 \leq \sqrt{n} \| \hsvby - \tsvby\|_2$, $(b)$ follows from the fact $a+b \leq 2\max\{a,b\} $ for any $a,b \geq 0$, $(c)$ follows by looking the cases separately, and $(d)$ follows from the fact $\max\{a,b\} \leq a+b$ for any $a,b \geq 0$.
	
	Let $\cF$ be the sigma-algebra generated by $(v_{h,r})_{1\leq h \leq n, 1\leq r \leq \tp}$. 
	Let $\Expectation^{\cF}$ denote the conditional expectation given $\cF$. 
	Conditional on $\cF$, $\rvu_r$ is \textit{sub-Gaussian} with variance proxy $\tsigma^2 \Big(\sum_{h = 1}^n v_{h,r}^2 \Big)$. Since $v_{h,r} \leq \tcOne$ almost surely for all $h,r$, it follows from the \textit{maximal inequality} of the \textit{sub-Gaussian} random variables (see Lemma 4 in \cite{Chatterjee2013})  that
	\begin{align}
	\Expectation^{\cF} (\max_{1\leq r \leq \tp} |\rvu_r|)\leq \tcOne \tsigma \sqrt{2n \log(2\tp)} 
	\end{align}
	Since the right-hand-side is non-random, taking expectation on both sides with respect to $\cF$ result in,
	\begin{align}
	\Expectation (\max_{1\leq r \leq \tp} |\rvu_r|)\leq \tcOne \tsigma \sqrt{2n \log(2\tp)}   \label{eq:lasso6}
	\end{align}
	Taking expectation on both sides in \eqref{eq:LASSO:Two-terms-simplified} and using \eqref{eq:lasso6}, we get
	\begin{align}
	\Expectation (\| \hsvby - \tsvby\|_2^2 )\leq 4\tBoundedNoise_0^2 n + 4\tcOne\tcTwo \tsigma \sqrt{2n \log(2\tp)}   \label{eq:LASSO-insample}
	\end{align}
	Dividing both sides by $n$ results in
	\begin{align}
	\Expectation [\widehat{\text{MSPE}}(\tbbeta)]  \leq 4\tBoundedNoise_0^2 + 4\tcOne\tcTwo\tsigma  \sqrt{\frac{2\log 2\tp}{n}} 
	\end{align}
	Recall that $\tbbeta$ is computed using the data $\bV$ and $\svby$, and is therefore independent of $\rvbv$ and $\rvy$. 
	Using definitions of $\trvy$ and $\hrvy$, we have
	\begin{align}
	\Expectation^{\cF}  (\hrvy - \trvy)^2 = \sum_{r,s = 1}^p ( \betaStar_r - \tbeta_r) ( \betaStar_s - \tbeta_s) \Expectation(\rvv_r \rvv_s)
	\end{align}
	We also have
	\begin{align}
	\frac{1}{n} \| \hsvby - \tsvby\|_2^2 =\frac{1}{n} \sum_{h = 1}^n \sum_{r,s = 1}^p ( \betaStar_r - \tbeta_r) ( \betaStar_s - \tbeta_s)v_{h,r} \rvv_{h,s}
	\end{align}
	Therefore by defining
	\begin{align}
	\rvu_{r,s} = \Expectation(\rvv_r \rvv_s) - \frac{1}{n} \sum_{h = 1}^n v_{h,r} \rvv_{h,s}
	\end{align}
	we have
	\begin{align}
	\Expectation^{\cF}  (\hrvy - \trvy)^2  - \frac{1}{n} \| \hsvby - \tsvby\|_2^2  = \sum_{r,s = 1}^p ( \betaStar_r - \tbeta_r) ( \betaStar_s - \tbeta_s) \rvu_{r,s}  \stackrel{(a)}{\leq}  4\tcTwo^2 \max_{1\leq r,s \leq \tp} |\rvu_{r,s} |  \label{eq:LASSO-combined}
	\end{align}
	where $(a)$ follows from the facts that $\| \bbetaStar \|_1 \leq \tcTwo$ and $\| \tbbeta \|_1 \leq \tcTwo $. Recall that $|\rvv_r | \leq \tcOne$ $\forall r \in [\tp]$. 
	Using the triangle inequality, we have $\Expectation(\rvv_r \rvv_s) - v_{h,r} \rvv_{h,s} \leq 2\tcOne^2$ for all $h,r$, and $s$. 
	It follows by Hoeffding's inequality (see Lemma 5 in \cite{Chatterjee2013})  that for any $\varsigma \in \Reals$,
	\begin{align}
	\Expectation(e^{\varsigma \rvu_{r,s}}) \leq e^{2\varsigma^2 \tcOne^4 / n}.
	\end{align}
	Again by the \textit{maximal inequality} of the \textit{sub-Gaussian} random variables (see Lemma 4 in \cite{Chatterjee2013}) we have,
	\begin{align}
	\Expectation (\max_{1\leq r,s \leq \tp} |\rvu_{r,s} |) \leq 2\tcOne^2 \sqrt{\frac{2\log(2\tp^2)}{n}}  \label{eq:LASSO-maximal2}
	\end{align}
	Taking expectation on both sides in \eqref{eq:LASSO-combined} and plugging in \eqref{eq:LASSO-insample} and \eqref{eq:LASSO-maximal2}, we get
	\begin{align}
	\Expectation  (\hrvy - \trvy)^2	\leq 4\tBoundedNoise_0^2 + 4\tcOne\tcTwo\tsigma  \sqrt{\frac{2\log 2\tp}{n}} + 8\tcOne^2\tcTwo^2  \sqrt{\frac{2\log (2\tp^2)}{n}} 
	\end{align}
	and this completes the proof.
\end{proof}

\section{Supporting propositions for Lemma \ref{lemma:recover-node-terms}}
\label{sec:supporting propositions for lemma:recover-node-terms}%LABEL
In this section, we will state the key propositions required in the proof of Lemma \ref{lemma:recover-node-terms}.
The proof of Lemma \ref{lemma:recover-node-terms} is given in Appendix \ref{sec:proof of lemma:recover-node-terms}.\\
%Recall the definitions of $\g = \thetaMax(k+k^2d)$ and $\varphiMax = (1+\Xupper)  \max\{\phiMax,\phiMax^2\}$ from Section \ref{subsec:notations and definitions}. 

Recall from Section \ref{sec:algorithm} that $\blambdaStar (x_{-i})$ denotes the conditional canonical parameter vector and $\bmuStar(x_{-i})$ denotes the conditional mean parameter vector of the conditional density $\ConditionalDensityNodeIfun$. Recall the definition of $\qs$ from Section \ref{sec:problem setup}.

\subsection{Learning conditional mean parameter vector}
\label{subsec:learning conditional mean parameter vector}
The first step of the algorithm for recovering node parameters from Section \ref{sec:algorithm} provides an estimate of the conditional mean parameter vector. 
The following proposition shows that, with enough samples and an estimate of the graph structure, we can learn the conditional mean parameter vector such that the $\ell_{\infty}$ error is small with high probability.
\begin{proposition} \label{proposition:learning conditional mean}
	Suppose we have an estimate $\hG$ of $G(\bthetaStar)$ such that for any $\deltaFour \in (0,1)$, $\hG = G(\bthetaStar)$ with probability at least $1-\deltaFour$.
	Given $n$ independent samples $\svbx^{(1)} \cdots , \svbx^{(n)}$ of $\rvbx$, consider $x_{-i}^{(z)}$ where $z$ is chosen randomly from $\{1,\cdots,n\}$.
	There exists an alogrithm that produces an estimate $\hbmu( x_{-i}^{(z)})$ of $\bmuStar(x_{-i}^{(z)})$ such that for any $\epsFour \in (0,1)$,
	\begin{align}
	\|\bmuStar(x_{-i}^{(z)})  - \hbmu( x_{-i}^{(z)}) \|_{\infty} \leq \epsFour \hspace{1cm} \forall i \in [p] 
	\end{align}
	with probability at least $1 - \deltaFour - k \epsFour^2/4 $ as long as
	\begin{align}
	n \geq \bigg(\frac{2^{9d+17} \Xupper^{2d} k^{4d} d^{2d+1} \thetaMax^{2d} \phiMax^{4d+4} \hphiMax^{2d} }{\epsFour^{4d + 8}}\bigg)  \log (\frac{2^{5.5}\Xupper k^2d\thetaMax \phiMax^2 \hphiMax}{ \epsFour^2 })
	\end{align}
	The number of computations required scale as 
	\begin{align}
	\frac{2^{18d+17} \Xupper^{4d} k^{8d+1} d^{4d+1} \thetaMax^{4d} \phiMax^{8d+4} \hphiMax^{4d}}{\alpha^{8d + 8}} \times p
	\end{align}
\end{proposition}
The proof of proposition \ref{proposition:learning conditional mean} is given in Appendix \ref{sec:proof of proposition:learning conditional mean}.

\subsection{Learning canonical parameter vector}
\label{subsec:learning canonical parameter vector}
The second step of the algorithm for recovering node parameters from Section \ref{sec:algorithm} is to obtain an estimate of the canonical parameter vector given an estimate of the mean parameter vector.
We exploit the conjugate duality between the canonical and mean parameters and run a projected gradient descent algorithm for this purpose.\\

We will describe the algorithm using a generic setup in this section and then apply it to the current setting in the proof of Lemma \ref{lemma:recover-node-terms} in Appendix \ref{sec:proof of lemma:recover-node-terms}.

\subsubsection{Setup for the projected gradient descent algorithm}
\label{subsubsec:setup for the projected gradient descent algorithm}
Let $\cX_{0}$ be a real interval such that its length is upper (lower) bounded by $\Xupper$ ($\Xlower$). 
Suppose that $\rvw$ is a random variable that takes value in $\cX_{0}$ with probability density function as follows,
\begin{align}
f_{\rvw}(w; \brhoStar) \propto \exp(\brhoStarT \bphi(w) )  \label{eq:true_density_CD}
\end{align}
where the parameter vector $\brhoStar\coloneqq (\rhoStar_1, \cdots, \rhoStar_k)$ is unknown and is such that $\| \brhoStar\|_{\infty} \leq \rhoMax$. 
Let $\cP \coloneqq \{\brho \in \Reals^k :  \| \brho\|_{\infty} \leq \rhoMax\}$.
Let $ \bupStar \coloneqq (\upStar_1 , \cdots, \upStar_k)$ denote the mean parameter vector of $f_{\rvw}(w;\brhoStar)$ and let $\hbup$ be an estimate of $\bupStar$ such that, we have $\| \bupStar - \hbup \|_{\infty} \leq \epsFive$ with probability at least $1 - \deltaFive$ for any $\epsFive > 0$, and any $\deltaFive \in (0,1)$.
The goal is to estimate the parameter vector $\brhoStar$ using the projected gradient descent algorithm \cite{BoydV2004, Bubeck2015}.

\subsubsection{The projected gradient descent algorithm}
\label{subsubsec:the projected gradient descent algorithm}

Let $\Uniform$ denote the uniform distribution on $\cX_{0}$.
Algorithm \ref{alg:MRW} is a subroutine that is used in the projected gradient descent algorithm. 
This subroutine is a Markov chain and it provides an estimate of the mean parameters of an exponential family distribution of the form \eqref{eq:true_density_CD} when the underlying canonical parameters are known.
See Appendix \ref{sec:analysis of algorithm:MRW} for discussion on the theoretical properties of this subroutine.
\begin{algorithm}
	\caption{Metropolized random walk (MRW)}\label{alg:MRW}
	\begin{algorithmic}[1]
		\State \textbf{Input:} $\brho, k , \cX_{0}, \tau_1, \tau_2, w_{(0)}$
		\State \textbf{Output:} $\hunu(\brho)$ 
		\For{$m = 1,\cdots$,$\tau_2$}
		\For{$r = 0,\cdots$,$\tau_1$}
		\State \textbf{Proposal step:} Draw $z_{(m,r+1)} \sim \Uniform$
		\State \textbf{Accept-reject step:}
		\State \hskip2.0em Compute $\alpha_{(m,r+1)} \leftarrow \min\bigg\{1, \frac{\exp( \brho^T \bphi( z_{(m,r+1)}) )}{\exp( \brho^T \bphi(w_{(m,r)}) )}\bigg\}$
		\State \hskip2.0em With probability $\alpha_{(m,r+1)}$ accept the proposal: $w_{(m,r+1)} \leftarrow z_{(m,r+1)}$
		\State \hskip2.0em With probability $1 - \alpha_{(m,r+1)}$ reject the proposal: $w_{(m,r+1)} \leftarrow w_{(m,r)}$
		\EndFor
		\State $\hunu(\brho) \leftarrow \hunu(\brho) + \bphi(w_{(m,\tau_1 + 1)})$
		\EndFor
		\State $\hunu(\brho) \leftarrow \frac{1}{\tau_2}\hunu(\brho)$
	\end{algorithmic}
\end{algorithm}

\begin{algorithm}
	\caption{Projected Gradient Descent}\label{alg:GradientDescent}
	\begin{algorithmic}[1]
		\State \textbf{Input:} $\xi, k, \cX_{0}, \tau_1, \tau_2, \tau_3, w_{(0)}, \brho^{(0)}, \hbup $
		\State \textbf{Output:} $\hbrho$
		\For {$r = 0,\cdots,\tau_3$}
		\State $\hunu(\brho^{(r)} ) \leftarrow MRW(\brho^{(r)} , k, \cX_{0}, \tau_1, \tau_2, w_{(0)})$
		\State $\brho^{(r+1)} \leftarrow \argmin_{\brho \in \cP} \| \brho^{(r)} - \xi [\hunu(\brho^{(r)} ) - \hbup ] - \brho\|$
		\EndFor
		\State	$\hbrho \leftarrow \brho^{(\tau_3+1)}$
	\end{algorithmic}
\end{algorithm}

\subsubsection{Guarantees on the output of the projected gradient descent algorithm}
\label{subsubsec:guarantees on the output of the projected gradient descent algorithm}

The following Proposition shows that running sufficient iterations of the projected gradient descent (Algorithm \ref{alg:GradientDescent}) results in an estimate, $\hbrho$, of the parameter vector, $\brhoStar$, such that the $\ell_{2}$ error is small with high probability.\\

Define $\bcOne \coloneqq \qs, \bcTwo \coloneqq 2k\phiMax^2$, $\bcThree \coloneqq \frac{ 4k\epsFive(\epsFive + 2 \bcTwo \rhoMax + 2\phiMax)}{\bcOne\bcTwo} $.
\begin{proposition}\label{proposition:grad-des}
	Let $\epsSix > 0$. Let $\hbrho$ denote the output of Algorithm \ref{alg:GradientDescent} with $\xi = 1/\bcTwo$, $\tau_1 = 8  k  \Xlower^{-2}  \rhoMax \phiMax\exp(12k \rhoMax \phiMax) \log\frac{4\phiMax \sqrt{\Xupper}}{\epsFive\sqrt{\Xlower}}$, $\tau_2 = \frac{8\phiMax^2}{\epsFive^2} \log\big(\frac{2k\tau_3}{\deltaFive}\big)$, $\tau_3 = \frac{\bcTwo }{\bcOne} \log \bigg(\frac{k \rhoMax^2 }{\epsSix^2 - \bcThree}\bigg)$, $w_{(0)} = 0$, $\brho^{(0)} = (0,\cdots , 0)$ and $\hbup  = ( \hup_1 , \cdots, \hup_k)$. 
	Then,
	\begin{align}
	\| \brhoStar- \hbrho\|_{2} \leq \epsSix
	\end{align}
	with probability at least $1-2\deltaFive$.
\end{proposition}
The proof of proposition \ref{proposition:grad-des} is given in Appendix \ref{sec:Proof of Proposition:grad-des}.

\section{Proof of Lemma \ref{lemma:recover-node-terms}}
\label{sec:proof of lemma:recover-node-terms}%LABEL
In this section, we prove Lemma \ref{lemma:recover-node-terms}.
See Appendix \ref{subsec:learning conditional mean parameter vector} and Appendix \ref{subsec:learning canonical parameter vector} for two key propositions required in the proof.\\

Recall from Section \ref{sec:algorithm} that $\blambdaStar (x_{-i})$ denotes the conditional canonical parameter vector and $\bmuStar(x_{-i})$ denotes the conditional mean parameter vector of the conditional density $\ConditionalDensityNodeIfun$.
Recall the definitions of $\g = \thetaMax(k+k^2d)$, $\varphiMax = (1+\Xupper)  \max\{\phiMax,\phiMax^2\}$, $c_1(\alpha)$, and $c_2(\alpha)$ from Section \ref{sec:problem setup} and the definition of $c_3(\alpha)$ from Section \ref{sec:supporting lemmas for theorem}.
\begin{proof}[Proof of Lemma \ref{lemma:recover-node-terms}]
	Let the number of samples satisfy 
	\begin{align}
	n & \geq \max\Big[ c_1\bigg(\min\bigg\{ \frac{\thetaMin}{3}, \frac{\alpha_2}{2 d k \phiMax } \bigg\} \bigg) \log \bigg( \frac{4 p k }{\alpha_2^2} \bigg), c_2( \alpha_2) \Big]
	\end{align}
	Using Theorem \ref{theorem:structure} with $\delta = \alpha_2^4/4$, we know the true neighborhood $\cN(i)$ $\forall i \in [p]$, with probability at least $1 - \alpha_2^4/4$. 
	Let us condition on the event that we know the true neighborhood for every node.
	
	Define
	\begin{align}
	\epsFour = \frac{\alpha^2_2 \qs}{2^7k^2d\thetaMax\phiMax}, \epsSix = \frac{\alpha_2}{2}, \bcThree = \frac{\alpha_2^2}{8}
	\end{align}
	Consider $x_{-i}^{(z)}$ where $z$ is chosen uniformly at random from $[n]$. Using Proposition \ref{proposition:learning conditional mean} with $\deltaFour = \alpha_2^4/4$, the estimate $\hbmu( x_{-i}^{(z)})$ is such that
	\begin{align}
	\|\bmuStar(x_{-i}^{(z)})  - \hbmu( x_{-i}^{(z)}) \|_{\infty} \leq \epsFour \hspace{1cm} \forall i \in [p]
	\end{align}
	with probability at least $1 - \alpha_2^4/4 - k\epsFour^2 /4 $.
	
	Observe that $\blambdaStar (x_{-i}) $ is such that $\| \blambdaStar (x_{-i}) \|_{\infty} \leq \rhoMax = 2kd\thetaMax\phiMax$ $\forall x_{-i} \in \Pi_{j \in [p] \setminus \{i\}} \cX_j$. 
	Using Proposition \ref{proposition:grad-des} with $\hbup = \hbmu( x_{-i}^{(z)})$, $\epsFive = \epsFour$ and $\deltaFive = \alpha_2^4/4 + k \epsFour^2/4 $, the estimate $\hblambda (x_{-i}^{(z)})$ is such that
	\begin{align}
	\|\blambdaStar (x_{-i}^{(z)})  - \hblambda (x_{-i}^{(z)}) \|_{2} & \leq \epsSix \hspace{1cm} \forall i \in [p]\\
	\implies \|\blambdaStar (x_{-i}^{(z)})  - \hblambda (x_{-i}^{(z)}) \|_{\infty}  & \leq\epsSix \hspace{1cm} \forall i \in [p]  \label{eq:CCP}
	\end{align}
	with probability at least $1 - \alpha_2^4/2 - k \epsFour^2/2$.
	
	Plugging in the value of $\epsFour$ and observing that $(\qs)^2 \leq 2^{13}k^4d^2 \thetaMax^2 \phiMax^2$, it is easy to see that $k \epsFour^2/2 \leq \alpha_2^4/4$.\\
	
	Let $\hbvthetaIEps \in \parameterSet$ be an $\epsilon$-optimal solution of GRISE and let $\hbvthetaIEpsEdge$ be the component of $\hbvthetaIEps$ associated with the edge potentials. 
	Using Lemma \ref{lemma:recover-edge-terms} with $\alpha_1 = \alpha_2/2dk\phiMax$ and $\delta = \alpha_2^4/4$, we have
	\begin{align}
	\| \bvthetaIEdgeStar - \hbvthetaIEpsEdge \|_{2} \leq \frac{\alpha_2}{2dk\phiMax}, \hspace{1cm} \forall i \in [p] \\
	\implies \| \bvthetaIEdgeStar - \hbvthetaIEpsEdge \|_{\infty} \leq \frac{\alpha_2}{2dk\phiMax}, \hspace{1cm} \forall i \in [p]  \label{eq:EWP}
	\end{align}
	with probability at least $1 - \alpha_2^4/4$.
	
	For any $r \in [k]$ and $i \in [p]$, let $\lambdaStar_{r} (x_{-i}^{(z)})$ denote the $r^{th}$ element of $\blambdaStar (x_{-i}^{(z)})$. We have, for any $r \in [k]$ and $i \in [p]$,
	\begin{align}
	\thetaIrStar  & = \lambdaStar_{r} (x_{-i}^{(z)}) - \sum_{ j \neq i} \sum_{ s \in [k] }\thetaIJrsStar \phi_s(x_j^{(z)})  \\
	& \stackrel{(a)}{=} \lambdaStar_{r} (x_{-i}^{(z)}) - \sum_{ j \in \cN(i)} \sum_{ s \in [k] }\thetaIJrsStar \phi_s(x_j^{(z)})  \label{eq:truenodepara}
	\end{align}
	where $(a)$ follows because $\forall r,s \in [k], j \notin \cN(i)$, $\thetaIJrsStar = 0$. 
	Let $\hlambda_{r} (x_{-i}^{(z)})$ denote the $r^{th}$ element of $\hblambda (x_{-i}^{(z)})$. 
	Define the estimate, $\hthetaIr$, as follows:
	\begin{align}
	\hthetaIr  & \coloneqq \hlambda_{r} (x_{-i}^{(z)}) - \sum_{ j \in \cN(i)} \sum_{ s \in [k] }\hthetaIJrs \phi_s(x_j^{(z)})   \label{eq:estimatednodepara}
	\end{align}
	Combining \eqref{eq:truenodepara} and \eqref{eq:estimatednodepara}, the following holds $\forall i \in [p], \forall r \in [k]$ with probability at least $1 - \alpha_2^4$:
	\begin{align}
	\Big| \thetaIrStar  - \hthetaIr  \Big| & = \Big| \lambdaStar_{r} (x_{-i}^{(z)}) - \hlambda_{r} (x_{-i}^{(z)}) - \sum_{ j \in \cN(i)} \sum_{ s \in [k] }\Big( \thetaIJrsStar - \hthetaIJrs \Big) \phi_s(x_j^{(z)}) \Big|\\
	& \stackrel{(a)}{\leq} \Big| \lambdaStar_{r} (x_{-i}^{(z)}) - \hlambda_{r} (x_{-i}^{(z)}) \Big| + \sum_{ j \in \cN(i)} \sum_{ s \in [k] }\Big| \thetaIJrsStar - \hthetaIJrs \Big| \phi_s(x_j^{(z)}) \\
	& \stackrel{(b)}{\leq}  \epsSix + \frac{\alpha_2}{2dk\phiMax} \sum_{ j \in \cN(i)} \sum_{ s \in [k] } \phi_s(x_j^{(z)}) \\
	& \stackrel{(c)}{\leq}  \frac{\alpha_2}{2} + \frac{\alpha_2}{2}  = \alpha_2
	\end{align}
	where $(a)$ follows from the triangle inequality, $(b)$ follows from \eqref{eq:CCP} and \eqref{eq:EWP}, and $(c)$ follows because $\| \bphi(x_j) \|_{\infty} \leq \phiMax$ for any $x_j \in \Pi_{j \in [p]} \cX_j$, $| \cN(i) | \leq d$ and $\epsSix = \alpha_2/2$.
	
	The key computational steps are estimating $\hbvthetaIEps$ and $\cN(i)$ for every node.
	Using Lemma \ref{lemma:recover-edge-terms} with $\alpha_1 = \alpha_2/2dk\phiMax$, $\delta = \alpha_2^4/4$ and Theorem \ref{theorem:structure} with $\delta = \alpha_2^4/4$, the computational complexity scales as
	\begin{align} 
	c_3\Big(\min\Big\{\frac{\thetaMin}{3}, \frac{\alpha_2}{2dk\phiMax} \Big\}\Big)  \times \log\bigg(\frac{4pk}{\alpha_2^2}\bigg) \times \log{(2k^2p)}  \times p^2
	\end{align}
\end{proof}

\section{Proof of Proposition \ref{proposition:learning conditional mean}}
\label{sec:proof of proposition:learning conditional mean}%LABEL
In this section, we prove Proposition \ref{proposition:learning conditional mean}. 

Recall from Section \ref{sec:algorithm} that $\blambdaStar (x_{-i})$ denotes the conditional canonical parameter vector and $\bmuStar(x_{-i})$ denotes the conditional mean parameter vector of the conditional density $\ConditionalDensityNodeIfun$ in \eqref{eq:form2}. 
For any $j \in [k]$, the $j^{th}$ element of the conditional mean parameter vector is given by
\begin{align}
\muStar_{j} (x_{-i}) = \frac{\displaystyle \int_{x_i \in \cX_i} \phi_j(x_i) \exp\Big(  \blambdaStarT (x_{-i}) \bphi(x_i)\Big) dx_i}{ \displaystyle \int_{x_i \in \cX_i} \exp\Big(  \blambdaStarT (x_{-i}) \bphi(x_i)\Big) dx_i}  \label{eq:conditionalmeans}
\end{align}
Define $L_1 \coloneqq 2 k^2 \thetaMax \phiMax^2 \hphiMax$.

We begin by showing Lipschitzness of the conditional mean parameters and then express the problem of learning the conditional mean parameters as a sparse linear regression. 
This will put us in a position to prove Proposition \ref{proposition:learning conditional mean}.

\subsection{Lipschitzness of conditional mean parameters}
\label{subsec:lipschitzness of conditional mean parameters}
The following Lemma shows that $\forall i \in [p]$, the conditional mean parameters associated with the conditional density of node $\rvx_i$ given the values taken by all the other nodes $(\rvx_{-i} = x_{-i})$ are Lipschitz functions of $x_m$ $\forall m \in [p] \setminus \{i\}$
\begin{lemma} \label{lemma:lipschitzness}
	For any $i \in [p]$, $j \in [k]$ and $x_{-i} \in \Pi_{j \in [p] \setminus \{i\}} \cX_j$, $\muStar_{j}(x_{-i}) $ is  $L_1$ Lipschitz.
\end{lemma}

\begin{proof}[Proof of Lemma \ref{lemma:lipschitzness}]	
	Fix any $i \in [p]$ and $j \in [k]$. 
	Consider any $m \in [p] \setminus \{i\}$. 
	Differentiating both sides of \eqref{eq:conditionalmeans} with respect to $x_m$ and applying the quotient rule gives us,
	\begin{align}
	\begin{aligned}
	\dfrac{\partial\muStar_{j} (x_{-i}) }{\partial x_m} &	= \frac{ \frac{\partial}{ \partial x_m}\displaystyle\int_{x_i \in \cX_i} \phi_j(x_i) \exp\Big(  \blambdaStarT (x_{-i}) \bphi(x_i)\Big) dx_i }{\displaystyle\int_{x_i \in \cX_i} \exp\Big(  \blambdaStarT (x_{-i}) \bphi(x_i)\Big) dx_i} \\& -  \frac{\bigg(\displaystyle\int_{x_i \in \cX_i} \phi_j(x_i) \exp\Big(  \blambdaStarT (x_{-i}) \bphi(x_i)\Big) dx_i\bigg) \bigg( \frac{\partial}{\partial x_m} \displaystyle \int_{x_i \in \cX_i} \exp\Big(  \blambdaStarT (x_{-i}) \bphi(x_i)\Big) dx_i\bigg) }{\bigg(\displaystyle \int_{x_i \in \cX_i} \exp\Big(  \blambdaStarT (x_{-i}) \bphi(x_i)\Big) dx_i\bigg)^2}
	\end{aligned}
	\end{align}
	Observe that $\phi_j(x_i) \exp\Big(  \blambdaStarT (x_{-i}) \bphi(x_i) \Big)$ and $\frac{\partial}{\partial x_m}\phi_j(x_i) \exp\Big(  \blambdaStarT (x_{-i}) \bphi(x_i) \Big)$ are analytic functions, and therefore are also continuous functions of $x_i$ and $x_m$. 
	We can apply the Leibniz integral rule to interchange the integral and partial differential operators. 
	This results in
	\begin{align}
	\begin{aligned}
	\dfrac{\partial \muStar_{j} (x_{-i}) }{\partial x_m} 	&  = \frac{ \displaystyle \int_{x_i \in \cX_i } \phi_j(x_i)\frac{\partial \blambdaStarT (x_{-i}) \bphi(x_i)}{\partial x_m} \exp\Big(\blambdaStarT (x_{-i}) \bphi(x_i)\Big)dx_i }{\displaystyle \int_{x_i \in \cX_i} \exp\Big(  \blambdaStarT (x_{-i}) \bphi(x_i)\Big) dx_i} \\ & -  \frac{\bigg(\displaystyle \int_{x_i \in \cX_i} \hspace{-3mm} \phi_j(x_i) \exp\Big(  \blambdaStarT (x_{-i}) \bphi(x_i)\Big) dx_i \bigg) \bigg( \displaystyle \int_{x_i \in \cX_i} \hspace{-3mm} \frac{\partial \blambdaStarT (x_{-i}) \bphi(x_i)}{\partial x_m} \exp\Big(\blambdaStarT (x_{-i}) \bphi(x_i)\Big)dx_i\bigg) }{\bigg(\displaystyle \int_{x_i \in \cX_i} \exp\Big(  \blambdaStarT (x_{-i}) \bphi(x_i)\Big) dx_i\bigg)^2}\\
	& = \Expectation\bigg( \phi_j(\rvx_i) \times  \frac{\partial \blambdaStarT (x_{-i}) \bphi(\rvx_i)}{\partial x_m}\bigg| \rvx_{-i} = x_{-i} ; \bvthetaIStar\bigg) \\ & - \Expectation\bigg( \phi_j(\rvx_i) \bigg| \rvx_{-i} = x_{-i} ; \bvthetaIStar\bigg)  \times  \Expectation\bigg(  \frac{\partial \blambdaStarT (x_{-i}) \bphi(\rvx_i)}{\partial x_m}\bigg| \rvx_{-i} = x_{-i} ; \bvthetaIStar\bigg)
	\end{aligned}
	\end{align}
	Using the triangle inequality we have,
	\begin{align}
	\begin{aligned}
	\bigg|\dfrac{\partial \muStar_{j} (x_{-i}) }{\partial x_m} \bigg| & \leq \bigg| \Expectation\bigg( \phi_j(\rvx_i) \times  \frac{\partial \blambdaStarT (x_{-i}) \bphi(\rvx_i)}{\partial x_m}\bigg| \rvx_{-i} = x_{-i} ; \bvthetaIStar\bigg) \bigg|  \\ & + \bigg|  \Expectation\bigg( \phi_j(\rvx_i) \bigg| \rvx_{-i} = x_{-i} ; \bvthetaIStar\bigg)  \times  \Expectation\bigg(  \frac{\partial \blambdaStarT (x_{-i}) \bphi(\rvx_i)}{\partial x_m}\bigg| \rvx_{-i} = x_{-i} ; \bvthetaIStar\bigg) \bigg|\\
	& \stackrel{(a)}{=} \bigg| \Expectation\bigg( \phi_j(\rvx_i) \times  \frac{\partial \blambdaStarT (x_{-i}) \bphi(\rvx_i)}{\partial x_m}\bigg| \rvx_{-i} = x_{-i} ; \bvthetaIStar\bigg) \bigg|  \\ & + \bigg|  \Expectation\bigg( \phi_j(\rvx_i) \bigg| \rvx_{-i} = x_{-i} ; \bvthetaIStar\bigg) \bigg| \times \bigg|  \Expectation\bigg(  \frac{\partial \blambdaStarT (x_{-i}) \bphi(\rvx_i)}{\partial x_m}\bigg| \rvx_{-i} = x_{-i} ; \bvthetaIStar\bigg) \bigg|\\
	& \stackrel{(b)}{\leq}  \Expectation\bigg( \bigg| \phi_j(\rvx_i) \bigg|  \times \bigg|  \frac{\partial \blambdaStarT (x_{-i}) \bphi(\rvx_i)}{\partial x_m}\bigg|  \bigg| \rvx_{-i} = x_{-i} ; \bvthetaIStar\bigg)\\ & +  \Expectation\bigg( \bigg| \phi_j(\rvx_i) \bigg| \bigg| \rvx_{-i} = x_{-i} ; \bvthetaIStar\bigg) \times \Expectation\bigg( \bigg| \frac{\partial \blambdaStarT (x_{-i}) \bphi(\rvx_i)}{\partial x_m}\bigg|  \bigg| \rvx_{-i} = x_{-i} ; \bvthetaIStar\bigg)  \label{eq:lip1}
	\end{aligned}
	\end{align}
	where $(a)$ follows because for any $a,b$, we have $|ab| = |a| |b|$ and $(b)$ follows because the absolute value of an integral is smaller than or equal to the integral of an absolute value.\\
	
	We will now upper bound $\frac{\partial \blambdaStarT (x_{-i}) \bphi(x_i)}{\partial x_m}$ as follows:
	\begin{align}
	\frac{\partial \blambdaStarT (x_{-i}) \bphi(x_i)}{\partial x_m} & \stackrel{(a)}{=}  \frac{\partial \Big[ \sum_{ r \in [k] } \big( \thetaIrStar + \sum_{ j \neq i} \sum_{ s \in [k] } \thetaIJrsStar \phi_s(x_j) \big) \phi_r(x_i)\Big] }{\partial x_m} \\
	& = \sum_{ r \in [k] } \bigg(\sum_{ s \in [k] } \thetaIMrsStar \times \frac{d\phi_s(x_m)}{dx_m} \bigg) \phi_r(x_i) \\
	& \stackrel{(b)}{\leq}  k^2 \phiMax \hphiMax \thetaMax  \label{eq:lip2}
	\end{align} 
	where $(a)$ follows from the definition of $\blambdaStarT (x_{-i}) $ and $\bphi(x_i)$ and $(b)$ follows because $ \phi_r(x_i) \leq \phiMax$ $\forall r \in [k], \forall x_i \in \cX_i$ and $\frac{d\phi_s(x_m)}{dx_m} \leq \hphiMax$ $s \in [k], \forall x_m \in \cX_m$.\\
	
	Using \eqref{eq:lip2} along with the fact that $|\phi_j(x_i)| \leq \phiMax$ $\forall j \in [k], \forall x_i \in \cX_i$, we can further upper bound $\Big|\partial \muStar_{j} (x_{-i}) / \partial x_m \Big|$ as
	\begin{align}
	\bigg|\dfrac{\partial \muStar_{j} (x_{-i}) }{\partial x_m} \bigg| \leq k^2 \thetaMax \phiMax \hphiMax  \times \phiMax + k^2 \thetaMax \phiMax \hphiMax  \times \phiMax = L_1
	\end{align}
	As a result, we have $\|\nabla \bmuStar(x_{-i}) \|_{\infty} \leq  L_1$ and this concludes the proof.
\end{proof}

\subsection{Learning conditional mean parameters as a sparse linear regression}
\label{subsec:learning conditional mean parameters as a sparse linear regression} 
The following Lemma shows that learning the conditional mean parameters $\bmuStar (x_{-i})$ as a function of $x_{-i}$ using an estimate of the graph structure is equivalent to solving a sparse linear regression problem.

\begin{lemma}\label{lemma:mean-parameter-lasso-equivalence}
	Suppose we have an estimate $\hG$ of $G(\bthetaStar)$ such that for any $\deltaFour \in (0,1)$, $\hG = G(\bthetaStar)$ with probability at least $1-\deltaFour$.
	Let $t$ be a parameter and $\tp$ be such that $\tp \leq  \big(\Xupper / \length\big)^d$.
	The following holds with probability at least $1-\deltaFour$.
	For every $i \in [p], j \in [k]$, $x_{-i} \in \Pi_{j \in [p] \setminus \{i\}} \cX_j$, we can write $\muStar_{j}(x_{-i})$ as the following sparse linear regression:
	\begin{align}
	\muStar_{j} (x_{-i}) = \bPsiJT \svbb + \bareta 
	\end{align}
	where $\bPsiJ \in \Reals^{\tp}$ is the unknown parameter vector and $\svbb \in \Reals^{\tp}$ is the covariate vector and it is a function of $x_{-i}$.
	Further, we also have $|\bareta| \leq L_1d\length$, $\| \svbb \|_{\infty} \leq 1$ and $\|\bPsiJ\|_1 \leq \phiMax \big(\Xupper / \length\big)^d$.
\end{lemma}

\begin{proof}[Proof of Lemma \ref{lemma:mean-parameter-lasso-equivalence}]
	For mathematical simplicity, $\forall i \in [p]$ let the interval $\cX_i = \cX_b = [0,b]$ where $b$ is such that $\Xlower \leq b \leq \Xupper$. 
	Divide the interval $\cX_b$ into non-overlapping intervals of length $\length$. 
	For the sake of simplicity, we assume that $b/\length$ is an integer. 
	Let us enumerate the resulting $b/\length$ intervals as the set of integers $\cI \coloneqq \{ 1, \cdots, b/\length\}$. 
	For any $x \in \cX_b,$ $\exists \zeta \in \cI$ s.t  $x \in ( (\zeta - 1)\length , \zeta \length] $ and this allows us to define a map $\cM : \cX_b \rightarrow \cI$ s.t $\cM(x) = \zeta \length$. 
	Similarly, for any $\svbx \coloneqq (x_j : j \in \cJ) \in \cX_b^{|\cJ|}$ where $\cJ$ is any subset of $[p]$, we have the mapping $\cM(\svbx) = \bzeta \length$ where $\bzeta \coloneqq (\zeta_j : j \in \cJ)$ is such that $\zeta_j = \cM(x_j) / \length$. 
	Now for any $x \in \cX_b$, consider a binary mapping $\cW : \cX_b \rightarrow \{0,1\}^{\cI}$ defined as $\cW(x) = (w_j(x) : j \in \cI)$ such that $w_{\cM(x)/t} (x)= 1$ and $w_j(x) = 0$ $\forall j \in \cI \setminus \{\cM(x)/t \}$.\\
	
	Let us condition on the event that $\hG = G(\bthetaStar)$. 
	Therefore, we know the true neighborhood $\cN(i)$ $\forall i \in [p]$ with probability at least $1-\deltaFour$.
	Using the Markov property of the graph $G (\bthetaStar)$, we know that the conditional density (and therefore the conditional mean parameters) of a node $\rvx_i$ given the values taken by the rest of nodes depend only on the values taken by the neighbors of $\rvx_i$. 
	Therefore, we have
	\begin{align}
	\muStar_{j} (x_{-i})  =\muStar_{j} (x_{\cN(i)})  \label{eq:LR0}
	\end{align}
	where $x_{\cN(i)}$ denotes the values taken by the neighbors of $\rvx_i$. 
	Using the fact that max-degree of any node in $G(\bthetaStar)$ is at-most $d$ and Lemma \ref{lemma:lipschitzness}, we can write for any $j \in [k]$
	\begin{align}
	\Big| \muStar_{j} (x_{\cN(i)})  - \muStar_{j} \big(\cM(x_{\cN(i)})\big) \Big| & \leq L_1 \sqrt{d} \Big\| x_{\cN(i)} - \cM(x_{\cN(i)}) \Big\|_2 \\
	&  \stackrel{(a)}{\leq} L_1 d \length  \label{eq:LR1}
	\end{align}
	where $(a)$ follows because $\forall m \in [p]$, $| x_m - \cM(x_m) | \leq \length$ and cardinality of $\cN(i)$ is no more than $d$. 
	Now using the binary mapping $\cW$ defined above, we can expand $\muStar_{j} \big(\cM(x_{\cN(i)})\big)$ as
	\begin{align}
	\muStar_{j} \big(\cM(x_{\cN(i)})\big)  = \sum_{r=1}^{|\cN(i)|}\sum_{\substack{k_r \in \cI  }} \bigg( \prod_{m=1}^{|\cN(i)|} w_{k_m}(x_{\cN(i)_m})\bigg)
	\muStar_{j} \big(k_1 \length , \cdots, k_{|\cN(i)|} \length\big)  \label{eq:LR2}
	\end{align} 
	where $\cN(i)_m$ denotes the $m^{th}$ element of $\cN(i)$.
	Observe that, $\prod_{m=1}^{|\cN(i)|} w_{k_m}(x_{\cN(i)_m}) = 1$ only when $k_m = \cM(x_{\cN(i)_m})$ $\forall m \in [|\cN(i)|]$.\\
	
	Combining \eqref{eq:LR0}, \eqref{eq:LR1} and \eqref{eq:LR2} we have the following regression problem: 
	\begin{align}
	\muStar_{j} (x_{-i}) = \bPsiJT \svbb + \bareta 
	\end{align}
	where $\bPsiJ \coloneqq \Big(\muStar_{j} \big(k_1 \length , \cdots, k_{|\cN(i)|} \length\big) : k_r \in \cI \hspace{2mm} \forall r  \in [|\cN(i)|]\Big) \in \Reals^{\tp}$, $\tp = \big(b / \length\big)^{\cN(i)}$, $\svbb = \Big(\prod_{m=1}^{|\cN(i)|} w_{k_m}(x_{\cN(i)_m}) : k_r \in \cI \hspace{2mm} \forall r  \in [|\cN(i)|]\Big) \in \{0,1\}^{\tp}$,  and $\bareta$ is such that $|\bareta| \leq L_1d\length$. 
	Observe that $\| \svbb \|_{\infty} \leq 1$.
	Using the fact that cardinality of $\cN(i)$ is no more than $ d$, we have
	\begin{align} \label{eq:sparsityL0}
	\tp \leq \bigg(\frac{b}{\length}\bigg)^{d} \leq \bigg(\frac{\Xupper}{\length}\bigg)^{d}
	\end{align}
	Using the fact that the conditional mean parameters are upper bounded by $\phiMax$, we have the following sparsity condition:
	\begin{align} \label{eq:sparsityL1}
	\Big\|\bPsiJ\Big\|_1  \leq \phiMax \bigg(\frac{\Xupper}{\length}\bigg)^{d}
	\end{align}
\end{proof}

\subsection{Proof of Proposition \ref{proposition:learning conditional mean}}
\label{subsec:proof of proposition: learning conditional mean}
\begin{proof}[Proof of Proposition~\ref{proposition:learning conditional mean}]
	Let us condition on the event that $\hG = G(\bthetaStar)$. The following holds with probability at least $1-\deltaFour$.
	From Lemma \ref{lemma:mean-parameter-lasso-equivalence}, for a parameter $t$ and for every $i \in [p], j \in [k]$, $x_{-i} \in \Pi_{j \in [p] \setminus \{i\}} \cX_j$, we have
	\begin{align}
	\muStar_{j} (x_{-i}) = \bPsiJT \svbb + \bareta 
	\end{align}
	where $\bPsiJ \in \Reals^{\tp}$ is an unknown parameter vector and $\svbb \in \Reals^{\tp}$, a function of $x_{-i}$, is the covariate vector. 
	Further, we also have $\tp = \big(\Xupper / \length\big)^{d}$, $|\bareta| \leq L_1d\length$, $\| \svbb \|_{\infty} \leq 1$ and $\|\bPsiJ\|_1 \leq \phiMax \big(\Xupper / \length \big)^{d}$.\\
	
	Suppose $\svbx^{(1)} , \cdots, \svbx^{(n)}$ are the $n$ independent samples of $\rvbx$. 
	We tranform these to obtain the corresponding covariate vectors $\svbb^{(1)},\cdots, \svbb^{(n)}$ where $\svbb^{(l)} = \Big(\prod_{m=1}^{|\cN(i)|} w_{k_m}(x^{(l)}_{\cN(i)_m}) : k_r \in \cI \hspace{2mm} \forall r  \in [|\cN(i)|]\Big)$.
	Let $\bB$ be a $n \times \tp$ matrix such that $l^{th}$ row of $\bB$ is $\svbb^{(l)}$. 
	We also obtain the vector $\barbmu_j(x_{-i}) \coloneqq ( \mu_{j}^{(r)}(x_{-i}) : r \in [n] )$ where $\mu_{j}^{(r)}(x_{-i}) = \phi_j(x_i^{(r)})$. Letting $\tepsilon = \phi_j(x_i)  - \muStar_{j} (x_{-i})$, we see that $\tepsilon$ is \textit{sub-Gaussian} random variable with zero mean and variance proxy $\tsigma^2 = 4\phiMax^2$.	
	
	Let $\hbPsiJ$ be the output of algorithm \ref{alg:LASSO1} with inputs $\bV = \bB, \svby = \barbmu_j(x_{-i}) $ and $\tcTwo = \phiMax \big(\Xupper / \length \big)^{d}$. Using Lemma \ref{lemma:lasso} with $\tBoundedNoise = L_1dt$, $\tcOne = 1$, and $\tsigma = 2\phiMax$ we have
	\begin{align}
	\Expectation [\widehat{\text{MSPE}}(\hbPsiJ)] \leq 4L_1^2d^2\length^2 + 8\phiMax^2 \bigg(\frac{\Xupper}{\length}\bigg)^{d}   \sqrt{\frac{2\log 2\tp}{n}}   \label{eq:MSPE-main}
	\end{align}
	Using the upper bound on $\tp$ results in
	\begin{align}
	\Expectation [\widehat{\text{MSPE}}(\hbPsiJ)]  \leq 4L_1^2d^2\length^2 + 8 \phiMax^2 \bigg(\frac{\Xupper}{\length}\bigg)^{d}  \sqrt{\frac{2d}{n}\log \frac{2^{1/d}  \Xupper  }{ \length}}   \label{eq:MSPE-main2}
	\end{align}
	As $d \geq 1$, we have $2^{1/2d}  \leq 2$.
	Choosing the parameter $\length = \frac{\epsFour^2}{8\sqrt{2} L_1d} $ and plugging in $L_1 = 2k^2\thetaMax \phiMax^2 \hphiMax $ and $n$, we have
	\begin{align}
	\Expectation [\widehat{\text{MSPE}}(\hbPsiJ)] \leq \frac{\epsFour^4}{16}   \label{eq:MSPE-main4}
	\end{align}
	Consider $x_{-i}^{(z)}$ where $z$ is chosen uniformly at random from $[n]$. For the prediction $\hmu_{j}( x_{-i}^{(z)})$, we transform $x_{-i}^{(z)}$ to obtain the corresponding covariate vector $\svbb = \Big(\prod_{m=1}^{|\cN(i)|} w_{k_m}(x_{\cN(i)_m}^{(z)}) : k_r \in \cI \hspace{2mm} \forall r  \in [|\cN(i)|]\Big)$ and take its dot product with $\hbPsiJ$ as follows:
	\begin{align}
	\hmu_{j}( x_{-i}^{(z)}) = \hbPsiJT \svbb
	\end{align}
	Using Markov's inequality, we have
	\begin{align}
	\Prob (|\bPsiJT \svbb - \hbPsiJT \svbb  |^2 \geq \frac{\epsFour^2}{4}) \leq \frac{4 \Expectation  [ (\bPsiJT \svbb - \hbPsiJT \svbb  )^2] }{\epsFour^2} \stackrel{(a)}{=}	\frac{4 \Expectation [\widehat{\text{MSPE}}(\hbPsiJ)] }{\epsFour^2} \stackrel{(b)}{\leq} \frac{\epsFour^2}{4}
	\end{align}
	where $(a)$ follows from Definition \ref{def:estimated-MSPE} and $(b)$ follows from \eqref{eq:MSPE-main4}. Therefore, we have $|\bPsiJT \svbb - \hbPsiJT \svbb  | \leq \frac{\epsFour}{2}$ with probability at least $1 - \frac{\epsFour^2}{4}$.
	
	Further, the following holds with probability at least $1 - \frac{\epsFour^2}{4}$:
	\begin{align}
	|\muStar_{j}(x_{-i}^{(z)})  - \hmu_{j}( x_{-i}^{(z)})| &  = |\bPsiJT \svbb + \bareta  - \hbPsiJT \svbb | \\
	& \stackrel{(a)}{\leq} |\bPsiJT \svbb - \hbPsiJT \svbb  | +|\bareta|  \\
	& \stackrel{(b)}{\leq}  \frac{\epsFour}{2} + L_1 d\length \stackrel{(c)}{\leq}  \frac{\epsFour}{2} +  \frac{\epsFour^2}{8\sqrt{2}}  \stackrel{(d)}{\leq} \epsFour  \label{eq:Mean-estimation}
	\end{align}
	where $(a)$ follows from the triangle inequality, $(b)$ follows because $|\bareta| \leq L_1 d\length$ and $|\bPsiJT \svbb - \hbPsiJT \svbb  | \leq \frac{\epsFour}{2}$, $(c)$ follows by plugging in the value of $t$ and $L_1$, and $(d)$ follows because $\epsFour \leq 1$. 
	
	As \eqref{eq:Mean-estimation} holds $\forall j \in [k]$, the proof follows by using the union bound over all $j \in [k]$.
	
	Solving the sparse linear regression takes number of computations that scale as $\tp^2 \times n$. There are $k$ such sparse linear regression problems for each node. Substituting for $\tp$, $\length$, and $n$, the total number of computations required scale as
	\begin{align}
	\bigg(\frac{2^{18d+17} \Xupper^{4d} k^{8d+1} d^{4d+1} \thetaMax^{4d} \phiMax^{8d+4} \hphiMax^{4d}}{\alpha^{8d + 8}} \times p\bigg)  \log (\frac{2^{5.5}\Xupper k^2d\thetaMax \phiMax^2 \hphiMax}{ \epsFour^2 })
	\end{align}
	The $\log$ term is dominated by the preceding term.
\end{proof}

\section{Analysis of Algorithm \ref{alg:MRW}}
\label{sec:analysis of algorithm:MRW}
In this section, we discuss the theoretical properties of Algorithm \ref{alg:MRW}. 
These will be used in the proof of Proposition \ref{proposition:grad-des}.

Recall that we design a Markov chain in Algorithm \ref{alg:MRW} that estimates the mean parameter vector of an exponential family distribution whose canonical parameters are known. 
The sufficient statistic vector of this exponential family distribution is the basis vector $\bphi(\cdot)$. 
We design this Markov chain using a zeroth-order Metropolized random walk algorithm. 
We will provide an upper bound on the mixing time of this Markov chain and provide error bounds on the estimate of the mean parameter vector computed using the samples obtained from the Markov chain.

\subsection{Setup: The exponential family distribution}
\label{subsec:setup : the exponential family distribution}
Let $\cX_{0}$ be a real interval such that its length is upper (lower) bounded by known constant $\Xupper$ ($\Xlower$). Suppose that $\rvw$ is a random variable that takes value in $\cX_{0}$ with probability density function as follows,
\begin{align}
f_{\rvw}(w; \brho) \propto \exp(\brho^T \bphi(w) )  \label{eq:MRW-density}
\end{align}
where $\brho\coloneqq (\rho_1, \cdots, \rho_k)$ is the canonical parameter vector of the density in \eqref{eq:MRW-density} and it is such that $\| \brho \|_{\infty} \leq \rhoMax$.
Let the cumulative distribution function of $\rvw$ be denoted by $F_{\rvw}(\cdot;\brho)$.
Let $\bnu(\brho) = \Expectation_{\rvw} [\bphi(\rvw)] \in \Reals^k$ be the mean parameter vector of the density in \eqref{eq:MRW-density}, i.e., $\bnu(\brho)  = (\nu_1, \cdots, \nu_k )$ such that
\begin{align}
{\nu_j \coloneqq \int_{{w} \in \cX_{0}} \phi_j(w) f_{\rvw}(w; \brho) dw}  \label{eq:MRW-meanParameterVector}
\end{align}
We aim to estimate $\bnu(\brho)$ for a given parameter vector $\brho$ using Algorithm \ref{alg:MRW}. Let the estimated vector of mean parameters be denoted by $\hunu(\brho) \coloneqq (\hnu_1, \cdots, \hnu_k )$. Let $Z(\brho)$ be the partition function of $f_{\rvw}(\cdot;\brho)$ i.e.,
\begin{align}
Z(\brho) = \int_{w \in \cX_{0}} \exp(\brho^T \bphi(w) )dw  \label{eq:MRW-partitionFn}
\end{align}

\subsection{Bounds on the probability density function}
\label{subsec:bounds on the probability density function}
Let us define $\cH(\cdot) \coloneqq \exp(|\brho^T \bphi(\cdot) | )$ and $\cHmax \coloneqq \exp(k \rhoMax  \phiMax)$. We have $\forall w \in \cX_{0}$,
\begin{align}
\cH^{-1}(w)  \leq \exp(\brho^T \bphi(w)  ) \leq \cH(w)  \label{eq:MRW-exp-bounds}
\end{align}
Bounding the density function defined in \eqref{eq:MRW-density} using \eqref{eq:MRW-exp-bounds} results in
\begin{align}
\frac{1}{\Xupper\cH^2(w) }\leq f_{\rvw}(w; \brho) \leq \frac{\cH^2(w) }{\Xlower}  \label{eq:MRW-density-bounds}
\end{align}
Let us also upper bound $\cH(\cdot)$. We have $\forall w \in \cX_{0}$, 
\begin{align}
\cH(w) \stackrel{(a)}{\leq} \exp(\sum_{j = 1}^{k} |\rho_j \phi_j(w)| ) \stackrel{(b)}{\leq}  \exp(\rhoMax \sum_{j = 1}^{k} | \phi_j(w)| ) \stackrel{(c)}{\leq}  \exp(k \rhoMax  \phiMax) = \cHmax \label{eq:MRW-H-bound}
\end{align}
where $(a)$ follows from the triangle inequality, $(b)$ follows because $| \rho_j| \leq \rhoMax $ $\forall j \in [k]$, and $(c)$ follows because $|\phi_j(w) | \leq \phiMax$ $\forall j \in [k]$ and $\forall w \in \cX_{0}$.

\subsection{Mixing time of the Markov chain in Algorithm \ref{alg:MRW}}
\label{subsec:mixing time of the markov chain in algorithm:MRW}
We set up an irreducible, aperiodic, time-homogeneous, discrete-time Markov chain, whose stationary distribution is equal to $F_{\rvw}(w;\brho)$, using a zeroth-order Metropolized random walk algorithm \cite{Hastings1970, MetropolisRRTT1953}. 
The Markov chain is defined on a measurable state space $(\cX_{0}, \cB(\cX_{0}))$ with a transition kernel $\cK : \cX_{0} \times \cB(\cX_{0}) \rightarrow \Reals_+$ where $\cB(\cX_{0})$ denotes the $\sigma-$algebra of $\cX_{0}$.

\subsubsection{Total variation distance}
\label{subsubsec:total variation distance}
\begin{definition}
	Let $Q_1$ be a distribution with density $q_1$ and $Q_2$ be a distribution with density $q_2$ defined on a measureable state space $(\cX_{0}, \cB(\cX_{0}))$. 
	The total variation distance of $Q_1$ and $Q_2$ is defined as 
	\begin{align}
	\|Q_1-Q_2\|_{TV} & = \sup_{A \in \cB(\cX_{0})} |Q_1(A) - Q_2(A)|
	\end{align}
\end{definition}
The following Lemma shows that if the total variation distance between two distributions on the same domain is small, then $\forall j \in [k]$, the difference between the expected value of $\phi_j(\cdot)$ with respect to the two distributions is also small.
\begin{lemma}\label{lemma:TV-distance}
	Let $Q_1$ and $Q_2$ be two different distributions of the random variable $\rvw$ defined on $\cX_{0}$. 
	Let $\|Q_1-Q_2\|_{TV} \leq \epsSeven$ for any $\epsSeven > 0$. 
	Then,
	\begin{align}
	\Big\|\Expectation_{Q_1}[\bphi(\rvw) ] - \Expectation_{Q_2}[\bphi(\rvw) ]\Big\|_{\infty} \leq 2\epsSeven \phiMax
	\end{align}
\end{lemma}
\begin{proof}[Proof of Lemma \ref{lemma:TV-distance}]
	We will use the following relationship between the total variation distance and the $\ell_1$ norm in the proof:
	\begin{align}
	\|Q_1-Q_2\|_{TV} = \frac{1}{2} \int_{\cX_{0}} |q_1(w) - q_2(w) | dw  \label{eq:L1TVrelation}
	\end{align}
	For any $j \in [k]$, we have,
	\begin{align}
	\Big|\Expectation_{Q_1}[\phi_j(\rvw) ] - \Expectation_{Q_2}[\phi_j(\rvw) ]\Big| & \stackrel{(a)}{=} \Big| \int_{\cX_{0}} \phi_j(w) [q_1(w) - q_2(w) ] dw \Big|\\
	& \stackrel{(b)}{\leq}  \int_{\cX_{0}} |\phi_j(w)| |q_1(w)- q_2(w)| dw\\
	& \stackrel{(c)}{\leq} \phiMax \int_{\cX_{0}} |q_1(w)-q_2(w)| dw\\
	& \stackrel{(d)}{=} 2 \phiMax \|Q_1-Q_2\|_{TV}\\
	& \leq 2\epsSeven  \phiMax
	\end{align}
	where $(a)$ follows from the definition of expectation, $(b)$ follows because the absolute value of integral is less than integral of absolute value, $(c)$ follows because $|\phi_j(w) | \leq \phiMax $ $\forall j \in [k]$, and $(d)$ follows from \eqref{eq:L1TVrelation}.
\end{proof}

\subsubsection{Definitions}
\label{subsubsec:definitions}
\begin{definition}
	Given a distribution $F_0$ with density $f_{0}$ on the current state of a Markov chain, the transition operator $\cT(F_0)$ gives the distribution of the next state of the chain. 
	Mathematically, we have
	\begin{align}
	\cT(F_0) (A) = \int_{\cX_{0}}\cK(w,A) f_0(w) dw, \text{ for any } A \in \cB(\cX_{0})  \label{eq:tran-operator}
	\end{align}
\end{definition}

\begin{definition}
	The mixing time of a Markov chain, with initial distribution $F_{0}$ and transition operator $\cT$, in a total variation distance sense with respect to its stationary distribution $F_{\rvw}$, is defined as
	\begin{align}
	\tau(\epsilon) = \inf \bigg\{ r \in \Naturals \hspace{2mm}\text{s.t}  \hspace{2mm} \big|\big| \cT^{(r)}(F_0) - F_{\rvw} \big|\big|_{TV} \leq \epsilon  \bigg\}
	\end{align}
	where $\epsilon$ is an error tolerance and $\cT^{(r)}$ stands for $r$-step transition operator.
\end{definition}

\begin{definition}
	\label{definitionConductance}
	The conductance of the Markov chain with transition operator $\cT$ and stationary distribution $F_{\rvw}$ (with density $f_{\rvw}(w))$ is defined as
	\begin{align}
	\Conductance \coloneqq \min_{0 < F_{\rvw}(A) \leq \frac{1}{2}} \frac{\int_{A} \cT(\delta_w)(A^c)f_{\rvw}(w) dw}{F_{\rvw}(A)}
	\end{align}
	where $\cT(\delta_w)$ is obtained by applying the transition operator to a Dirac distribution concentrated on $w$.
\end{definition}

\subsubsection{Upper bound on the mixing time}
\label{subsubsec:upper bound on the mixing time}
Recall that $\Uniform$ denotes the uniform distribution on $\cX_{0}$. 
We let the initial distribution of the Markov chain be $\Uniform$. 
We run independent copies of the Markov chain and use the samples obtained after the mixing time in each copy to compute $\hunu$. 
In Algorithm \ref{alg:MRW}, $\tau_1$ is the number of iterations of the Markov chain and $\tau_2$ denotes the number of independent copies of the Markov chain used.
The following Lemma gives an upper bound on the mixing time of the Markov chain defined in Algorithm  \ref{alg:MRW}.
\begin{lemma}\label{lemma:mixing-time}
	Let the mixing time of the Markov chain defined in Algorithm \ref{alg:MRW} be denoted by  $\tau_{\text{M}}(\epsEight)$ where $\epsEight > 0$ is the error tolerance. 
	Then,
	\begin{align}
	\tau_{\text{M}}(\epsEight ) \leq 8  k  \Xlower^{-2}  \rhoMax \phiMax\exp(12k \rhoMax \phiMax) \log\frac{\sqrt{\Xupper}}{\epsEight\sqrt{\Xlower}}
	\end{align}
\end{lemma}
\begin{proof}[Proof of Lemma \ref{lemma:mixing-time}]
	We will control the mixing time of the Markov chain via worst-case conductance bounds. 
	This method was introduced for discrete space Markov chains by Jerrum and Sinclair \cite{JerrumS1988} and then extended to the continuous space Markov chains by Lov\'{a}sz and Simonovits \cite{LovaszS1993}; see Vempala (2005) \cite{Vempala2005} for a detailed discussion on the continuous space setting.
	
	For any initial distribution $F_0$ and stationary distribution $F_{\rvw}$ of a Markov chain, define $c_0 \coloneqq \sup_{A} \frac{F_0(A)}{F_{\rvw}(A)}$.
	\cite{LovaszS1993} proved that,
	\begin{align}
	\big|\big| \cT^{(r)}(F_0) - F_{\rvw} \big|\big|_{TV} \leq \sqrt{c_0} \exp^{-r\Conductance^2/2}
	\end{align}
	Therefore to upper bound the total variation distance by $\epsEight$, it is sufficient to have
	\begin{align}
	\sqrt{c_0} \exp^{-r\Conductance^2/2} \leq \epsEight
	\end{align}
	This can be rewritten as
	\begin{align}
	r & \geq \frac{2}{\Conductance^2}\log \frac{\sqrt{c_0}}{\epsEight}
	\end{align}
	Therefore, after $r = \frac{2}{\Conductance^2}\log \frac{\sqrt{c_0}}{\epsEight}$ steps of the Markov chain, the  total variation distance is less than $\epsEight$ and $\tau_{M}(\epsEight) \leq \frac{2}{\Conductance^2}\log \frac{\sqrt{c_0}}{\epsEight}$. 
	In order to upper bound the mixing time, we need upper bound the constant $c_0$ and lower bound the conductance $\Conductance$.
	
	We will first upper bound $c_0$. 
	We have the initial distribution to be uniform on $\cX_{0}$. 
	Therefore,
	\begin{align}
	c_0  = \sup_{A} \frac{\Uniform(A)}{F_{\rvw}(A)} \stackrel{(a)}{\leq} \sup_{A} \frac{\int_{A} \frac{1}{\Xlower} dw}{\int_{A}  \frac{1}{\Xupper\cH^2(w)} dw} \stackrel{(b)}{\leq}  \frac{\Xupper}{\Xlower} \cHmax^2  \label{eq:UB-c3}
	\end{align}
	where $(a)$ follows from the lower bound in \eqref{eq:MRW-density-bounds} and $(b)$ from \eqref{eq:MRW-H-bound}.
	
	Let us now lower bound $\Conductance$. 
	From the Definition \ref{definitionConductance} we have,
	\begin{align}
	\Conductance &= \min_{0 < \int_{A} f_{\rvw}(w;\brho) dw \leq \frac{1}{2}} \frac{\int_{A} \cT(\delta_w)(A^c)f_{\rvw}(w;\brho) dw}{\int_{A} f_{\rvw}(w;\brho) dw}\\
	& \stackrel{(a)}{=} \min_{0 < \int_{A} f_{\rvw}(w;\brho) dw \leq \frac{1}{2}} \frac{\int_{A} \cT(\delta_w)(A^c)\exp(\brho^T\bphi(w)) dw}{\int_{A}\exp(\brho^T\bphi(w)) dw}\\
	& \stackrel{(b)}{\geq}  \min_{0 < \int_{A} f_{\rvw}(w;\brho) dw \leq \frac{1}{2}} \frac{\int_{A} \cT(\delta_w)(A^c) \cH^{-1}(w) dw}{\int_{A}  \cH(w) dw}\\
	& \stackrel{(c)}{\geq} \frac{1}{\cHmax^2}\min_{0 < \int_{A} f_{\rvw}(w;\brho) dw \leq \frac{1}{2}} \frac{\int_{A} \cT(\delta_w)(A^c)dw}{\int_{A}  dw} \\
	& \stackrel{(d)}{\geq}  \frac{1}{\cHmax^2}\min_{0 < \int_{A} f_{\rvw}(w;\brho) dw \leq \frac{1}{2}} \frac{\int_{A}\big( \int_{\cX_{0}}  \cK(w,A^c) \delta_w (w) dw\big)	dw}{\int_{A}  dw}\\
	& = \frac{1}{\cHmax^2}\min_{0 < \int_{A} f_{\rvw}(w;\brho) dw \leq \frac{1}{2}} \frac{\int_{A}  \cK(w,A^c)	dw}{\int_{A}  dw}\\
	& = \frac{1}{\cHmax^2}\min_{0 < \int_{A} f_{\rvw}(w;\brho) dw \leq \frac{1}{2}} \frac{\int_{A} \int_{A^c} \cK(w,dy) dy	dw}{\int_{A}  dw}
	\end{align}
	where $(a)$ follows by canceling out $Z(\brho)$ in the numerator and the denominator, $(b)$ follows from \eqref{eq:MRW-exp-bounds}, $(c)$ follows from \eqref{eq:MRW-H-bound}, and $(d)$ follows from \eqref{eq:tran-operator}.
	
	Recall from Algorithm \ref{alg:MRW} that we make a transition from the current state $w$ to the next state $y$ with probability $\cK(w,dy)  = \min\Big\{1,\frac{\exp(\brho^T\bphi(y))}{\exp(\brho^T\bphi(w))}\Big\} $. Therefore,
	\begin{align}
	\Conductance & \geq \frac{1}{\cHmax^2}\min_{0 < \int_{A} f_{\rvw}(w;\brho) dw \leq \frac{1}{2}} \frac{\int_{A} \int_{A^c} \min\Big\{1,\frac{\exp(\brho^T\bphi(y))}{\exp(\brho^T\bphi(w))}\Big\}  dy	dw}{\int_{A}  dw}
	\end{align}
	Using \eqref{eq:MRW-exp-bounds} and observing that $\cHmax^{-2} \leq 1$, we have $\min\Big\{1,\frac{\exp(\brho^T\bphi(y))}{\exp(\brho^T\bphi(w))}\Big\} \geq \frac{1}{\cHmax^2}$. 
	This results in,
	\begin{align}
	\Conductance \geq \frac{1}{\cHmax^4}\min_{0 < \int_{A} f_{\rvw}(w;\brho) dw \leq \frac{1}{2}} \frac{\int_{A} \int_{A^c} dy	dw}{\int_{A}  dw} = \frac{1}{\cHmax^4}\min_{0 < \int_{A} f_{\rvw}(w;\brho) dw \leq \frac{1}{2}} \int_{A^c} dw  \label{eq:conductance-bound}
	\end{align}
	We have $ \int_{A} f_{\rvw}(w;\brho) dw \leq \frac{1}{2}$. 
	This can be rewritten as,
	\begin{align}
	\int_{A^c} f_{\rvw}(w;\brho) dw  \geq \frac{1}{2} \implies	\int_{A^c}  dw \stackrel{(a)}{\geq} \frac{\Xlower}{2\cHmax^2}  \label{eq:integralBound}
	\end{align}
	where $(a)$ follows from the upper bound in \eqref{eq:MRW-density-bounds}. 
	Using \eqref{eq:integralBound} in \eqref{eq:conductance-bound}, we have
	\begin{align}
	\Conductance & \geq \frac{\Xlower}{2\cHmax^{6}}  \label{eq:conductanceLB}
	\end{align}
	Now using \eqref{eq:UB-c3} and \eqref{eq:conductanceLB} to bound the mixing time, we have
	\begin{align}
	\tau_{M}(\epsEight) \leq \frac{8\cHmax^{12}}{\Xlower^2} \log\frac{\cHmax\sqrt{\Xupper}}{\epsEight \sqrt{\Xlower}}
	\end{align}
	Using the upper bound of $\cHmax$ from \eqref{eq:MRW-H-bound}, we have 
	\begin{align}
	\tau_{M}(\epsEight)  & \leq \frac{8\exp(12k \rhoMax \phiMax)}{\Xlower^2} \log\frac{\sqrt{\Xupper} \exp(k \rhoMax \phiMax)}{\epsEight\sqrt{\Xlower}} \\ 
	& = 8  k  \Xlower^{-2} \rhoMax \phiMax\exp(12k \rhoMax \phiMax) \log\frac{\sqrt{\Xupper}}{\epsEight\sqrt{\Xlower}}
	\end{align}
\end{proof}

\subsubsection{Guarantees on the output of Algorithm \ref{alg:MRW} }
\label{subsubsec:guarantees on the output of algorithm:MRW}

The following Lemma shows that the estimate, obtained from Algorithm \ref{alg:MRW}, of the mean parameter vector, is such that the $\ell_{\infty}$ error is small with high probability.
\begin{lemma}\label{lemma:mean-parameters}
	Let $\epsNine > 0$ and $\deltaNine \in (0,1)$. Let $\hunu(\brho)$ be the output of Algorithm \ref{alg:MRW} with $w_{(0)} = 0$, $\brho = (\rho_1, \cdots, \rho_k)$, $\tau_1 = 8  k  \Xlower^{-2}  \rhoMax \phiMax\exp(12k \rhoMax \phiMax) \log\frac{4\phiMax\sqrt{\Xupper}}{\epsNine\sqrt{\Xlower}}$, and $\tau_2 = \frac{8\phiMax^2}{\epsNine^2} \log\Big(\frac{2}{\deltaNine}\Big)$.
	Then,
	\begin{align}
	\| \bnu(\brho)  - \hunu (\brho)\|_{\infty} \leq \epsNine
	\end{align}
	with probability at least $1-k \deltaNine$.
\end{lemma}
\begin{proof}[Proof of Lemma \ref{lemma:mean-parameters}]
	The distribution of the Markov chain in Algorithm \ref{alg:MRW} after $\tau_1 + 1$ steps is $\cT^{(\tau_1 + 1)}(\Uniform)$ where $\Uniform$ denotes the initial uniform distribution. 
	Let $\bnu^{M}(\brho)  \coloneqq (\nu_1^{M}, \cdots, \nu_k^{M} )$ be the vector such that $\nu_j^{M} $ is the expected value of $\phi_j(\cdot)$ with respect to the distribution $\cT^{(\tau_1 + 1)}(\Uniform)$. 
	Using Lemma \ref{lemma:mixing-time}, we have $\tau_1 \geq \tau_{\text{M}}(\frac{\epsNine}{4\phiMax})$.
	Therefore,
	\begin{align}
	\big|\big| \cT^{\tau_1 +1}(\Uniform) - F_{\rvw} \big|\big|_{TV} & \leq \frac{\epsNine}{4\phiMax}
	\end{align}	
	From Lemma \ref{lemma:TV-distance}, we have
	\begin{align}
	\| \bnu(\brho)  - \bnu^{M}(\brho)\|_{\infty} \leq \frac{\epsNine}{2}  \label{eq:true-mean}
	\end{align}
	$\hunu(\brho)$ is computed using the samples obtained from the distribution $\cT^{(\tau_1 + 1)}(\Uniform)$. 
	Using Hoeffding's inequality, we have $\forall j \in [k]$
	\begin{align}
	\Prob ( |\hnu_j - \nu_j^{M}| \geq t_0 ) \leq 2 \exp(\frac{-\tau_2 t_0^2}{2\phiMax^2})
	\end{align}
	Therefore when $\tau_2 \geq \frac{2\phiMax^2}{t_0^2} \log\big(\frac{2}{\deltaNine}\big)$, we have $|\hnu_j - \nu_j^{M}| \leq t_0$ with probability at least $ 1 -\deltaNine $.
	
	Using the union bound $\forall j \in [k]$, when $\tau_2 \geq \frac{8\phiMax^2}{\epsNine^2} \log\big(\frac{2}{\deltaNine}\big)$, we have 
	\begin{align}
	\|\hunu (\brho) -\bnu^{M}(\brho)\|_{\infty} \leq \frac{\epsNine}{2}  \label{eq:estimated-mean}
	\end{align}
	with probability at least $ 1 -k \deltaNine $.
	
	Combining \eqref{eq:true-mean} and \eqref{eq:estimated-mean} by triangle inequality, we have
	\begin{align}
	\| \bnu(\brho)  - \hunu (\brho)\|_{\infty} & = \| \bnu(\brho)  - \bnu^{M}(\brho) + \bnu^{M}(\brho) - \hunu (\brho)\|_{\infty} 	\\
	& \leq \| \bnu(\brho)  - \bnu^{M}(\brho)\|_{\infty} + \|\hunu (\brho) -\bnu^{M}(\brho)\|_{\infty} 	\leq \epsNine
	\end{align}
\end{proof}

\section{Proof of Proposition \ref{proposition:grad-des}}
\label{sec:Proof of Proposition:grad-des}
In this section, we prove Proposition \ref{proposition:grad-des}.

Recall the setup for the projected gradient descent algorithm from Appendix \ref{subsubsec:setup for the projected gradient descent algorithm}. 
Specifically, recall the definitions of $\brhoStar$, $f_{\rvw}( w; \brhoStar)$, $\cP$, $\rhoMax$, and $\bupStar$. 
Also, $\hbup$ is an estimate of $\bupStar$ such that, with probability at least $1 - \deltaFive$, we have $\| \bupStar - \hbup \|_{\infty} \leq \epsFive$.
Further, recall the setup from Appendix \ref{subsec:setup : the exponential family distribution}. 
Specifically, for any $\brho \in \cP$, recall the definitions of $f_{\rvw}(w; \brho)$, $\bnu(\brho) $, and $Z(\brho)$ from \eqref{eq:MRW-density}, \eqref{eq:MRW-meanParameterVector}, and \eqref{eq:MRW-partitionFn} respectively.
Recall the definition of $\qs$ from Section \ref{sec:problem setup}.

\subsection{Convexity of the log partition function}
\label{subsec:convextiy of the log partition function}
Let $\Phi (\brho)$ be the log partition function of $f_{\rvw}(w;\brho)$. 
Because $f_{\rvw}(w;\brho)$ is an exponential family density, $  \nabla \Phi(\brho) = \bnu(\brho) $; see \cite{WainwrightJ2008} for details. The following Lemma shows that $\Phi (\brho)$ is a strictly convex function of $\brho$.
\begin{lemma} \label{lemma:strict-convexity}
	$\Phi (\brho) $ is a strictly convex function of $\brho$.
\end{lemma}
\begin{proof}Proof of Lemma \ref{lemma:strict-convexity}
	For any non-zero $\svbe \in \Reals^k$, $\svbe^T \bphi(w)$ is not a constant with respect to $w$. 
	Therefore, 
	\begin{align}
	0 	& \stackrel{(a)}{<} \Variance\bigg(\svbe^T \bphi(w)\bigg)   \\
	& = \text{cov}\bigg(\svbe^T \bphi(w),\svbe^T \bphi(w)\bigg) \\
	& = 	\sum_{j = 1}^k \sum_{r = 1}^k e_{j} e_{r} \times \text{cov}(\phi_j(w),\phi_r(w)) \\
	& \stackrel{(b)}{=} 	\sum_{j = 1}^k \sum_{r = 1}^k e_{j} e_{r}[\nabla^2 \Phi (\brho) ]_{j,r}  \\
	& =	\svbe^T \nabla^2 \Phi (\brho)  \svbe
	\end{align}
	where $(a)$ follows because the variance of a non-constant random variable is strictly positive and $(b)$ follows because for any regular exponential family the Hessian of the log partition function is the covariance matrix of the associated sufficient statistic vector; see \cite{WainwrightJ2008} for details.
	
	Thus, $\nabla^2 \Phi (\brho)$ is a positive definite matrix and this is a sufficient condition for strict convexity of $\Phi (\brho) $.
\end{proof}

\subsection{Conjugate Duality}
\label{subsec:conjugate duality}
Expressing the relationship between the canonical and mean parameters via conjugate duality \cite{BreslerGS2014, WainwrightJ2008}, we know that for each $\bup$ in the set of realizable mean parameters, there is a unique $\brho(\bup) \in \cP$ satisfying the dual matching condition $\bnu(\brho(\bup)) = \bup$.
The backward mapping of the mean parameters to the canonical parameters ($\bup \mapsto	\brho(\bup)$) is given by,
\begin{align}
\brho(\bup) = \argmax_{\brho \in \cP} \bigg\{ \left\langle \bup, \brho \right\rangle - \Phi(\brho) \bigg\} \label{eq:backwardmappingmax}
\end{align}
Defining $\Omega(\brho,\bup) : = \Phi(\brho) - \left\langle \bup, \brho \right\rangle $, we can rewrite \eqref{eq:backwardmappingmax} as
\begin{align}
\brho(\bup)= \argmin_{\brho \in \cP} \big\{ \Omega(\brho,\bup) \big\} \label{eq:backwardmapping}
\end{align}
%Define $\qs \coloneqq \inf_{\brho : \rho_j \in \cP} \inf\limits_{\svbe : \|\svbe\|_2 \leq 1 } \svbe^T \nabla^2 \Phi (\brho) \svbe$.
For any $\brho \in \cP$, let $q(\brho)$ denote the smallest eigenvalue of the Hessian of the log partition function with canonical parameter $\brho$. 
Recall that $\qs$ denotes the minimum of $q(\brho)$ over all possible $\brho \in \cP$.

\begin{lemma}  \label{lemma:strong-convexity-smoothness}
	$\Omega(\brho,\bup)$ is a $\qs$ strongly convex function of $\brho $ and a $ 2k \phiMax^2$ smooth function of $\brho$.
\end{lemma} 
\begin{proof}[Proof of Lemma \ref{lemma:strong-convexity-smoothness}]
	Observe that $\nabla^2 \Omega(\brho,\bup) = \nabla^2 \Phi (\brho)$. 
	Therefore $\Omega(\brho,\bup)$ being a $\qs$ strongly convex function of $\brho $ and a $ 2k \phiMax^2$ smooth function of $\brho$ is equivalent to $\Phi (\brho)$ being a $\qs$ strongly convex function of $\brho $ and a $ 2k \phiMax^2$ smooth function of $\brho$.
	
	We will first show the strong convexity of $\Phi (\brho)$. 
	Consider any $\svbe \in \Reals^k$ such that $\|\svbe\|_2 =1$.
	We have
	\begin{align}
	q(\brho) = \inf\limits_{\svbe : \|\svbe\|_2 \leq 1 } \svbe^T \nabla^2 \Phi (\brho) \svbe
	\end{align}
	Using Lemma \ref{lemma:strict-convexity} we know that $q(\brho) > 0$ for any $\brho \in \cP$. 
	Observe that $[\nabla^2 \Phi (\brho)]_{j,r} = \text{cov}(\phi_j(w),\phi_r(w))$, and is a continuous function of $\brho$ ,$ \forall j,r \in [k].$ 
	Now $q(\brho) $ is a linear combination of $[\nabla^2 \Phi (\brho)]_{j,r} $ $ \forall j,r \in [k].$ 
	Therefore  $q(\brho)$ is also a continuous function of $\brho$. 
	Using the continuity of $q(\brho)$ and compactness of $\cP$, we apply the extreme value theorem and conclude that the function $q(\brho) $ will attain its minimum value of 
	\begin{align}
	\qs = \inf_{\brho\in \cP} q(\brho)
	\end{align}
	and that this value is positive. 
	Now using the fact that $\nabla^2 \Phi (\brho)$ is a symmetric matrix and the Courant-Fischer theorem, we conclude that the minimum possible eigenvalue of $\nabla^2 \Phi (\brho)$ for any $\brho \in \cP$ is greater than or equal to $\qs$. 
	Thus, the smallest possible eigenvalue of the Hessian of the log partition function is uniformly lower bounded. 
	As a result, $\Phi(\brho)$ and $\Omega(\brho,\bup)$ are $\qs$-strongly convex.
	
	We will now show the smoothness of $\Phi(\brho)$. 
	From the Gershgorin circle theorem, we know that the largest eigenvalue of any matrix is upper bounded by the largest absolute row sum or column sum. 
	Applying this, we see that the largest eigenvalue of $\nabla^2 \Phi (\brho)$ is upper bounded by $\max_{1\leq r \leq k} \sum_{j = 1}^{k} |[\nabla^2 \Phi (\brho)]_{j,r}|$. 
	Now
	\begin{align}
	\max_{1\leq r \leq k} \sum_{j = 1}^{k} |[\nabla^2 \Phi (\brho)]_{j,r}|& =\max_{1\leq r \leq k} \sum_{j = 1}^{k} |\text{cov}(\phi_j(w),\phi_r(w))|\\
	& \stackrel{(a)}{\leq} \max_{1\leq r \leq k}  \sum_{j = 1}^{k} 2\phiMax^2\\
	& \leq 2k\phiMax^2
	\end{align}
	where $(a)$ follows from the triangle inequality and because $|\phi_j(w)| \leq \phiMax$ $\forall j \in [k]$.
	
	Now because the largest eigenvalue of the Hessian matrix of the log partition function is uniformly upper bounded by $2k\phiMax^2$, $\Phi(\brho)$ and $\Omega(\brho,\bup)$ are $2k\phiMax^2$ smooth function of $\brho$.
\end{proof}

\subsection{Why projected gradient descent algorithm?}
\label{subsec:why projected gradient descent algorithm}
From Lemma \ref{lemma:strong-convexity-smoothness}, we see that there is a unique minimum in \eqref{eq:backwardmapping}. 
In other words, when the mean parameter in \eqref{eq:backwardmapping} is the true mean parameter of \eqref{eq:true_density_CD} i.e., $\bup = \bupStar$, then the unique minima in \eqref{eq:backwardmapping} is $\brhoStar$. 
Therefore, in principle, we can estimate $\brhoStar$ using a projected gradient descent algorithm. 

In each step of this algorithm, we need access to $ \bnu(\brho)$ for the estimate $\brho$. 
However, we don't have access to $\bupStar$ and $ \bnu(\brho)$. 
Instead, we have access to $\hbup$ and $ \hunu(\brho)$ (from Algorithm \ref{alg:MRW}).
Therefore, we can estimate the parameter vector $\brhoStar$ using the projected gradient descent in Algorithm \ref{alg:GradientDescent}.

\subsection{Proof of Proposition \ref{proposition:grad-des}}
\label{subsec:Proof of Proposition:grad-des}
\begin{proof}[Proof of Lemma \ref{proposition:grad-des}]
	The projection of $\tbrho$, onto a set $\cP$ is defined as
	\begin{align}
	\Pi_{\cP}(\tbrho) \coloneqq \argmin_{\brho \in \cP} \| \brho - \tbrho\|
	\end{align}
	
	If we had access to $\bupStar$ and $\bnu(\brho)$, the iterates of the projected gradient descent algorithm could be rewritten as
	\begin{align}
	\brho^{(r+1)} = \brho^{(r)} - \xi \gamma_{\cP}(\brho^{(r)})
	\end{align}
	where $\gamma_{\cP}(\brho)$ is the gradient mapping and is defined as $\gamma_{\cP}(\brho) \coloneqq \frac{1}{\xi}(\brho - \brho^{\dagger})$ with $\brho^{\dagger} \coloneqq \Pi_{\cP}(\brho - \xi [\bnu(\brho ) - \bupStar] )$. 
	See \cite{Bubeck2015} for more details. 
	Because we are using the respective estimates $\hbup$ and $\hunu (\brho)$, the iterates of the projected gradient descent algorithm are as follows:
	\begin{align}
	\brho^{(r+1)} = \brho^{(r)} - \xi \hat{\gamma}_{\cP}(\brho^{(r)})
	\end{align}
	where $\hat{\gamma}_{\cP}(\brho) \coloneqq \frac{1}{\xi}(\brho - \brho^{\dagger\dagger})$ with $\brho^{\dagger\dagger} \coloneqq \Pi_{\cP}(\brho - \xi [\hunu(\brho ) - \hbup ] )$. 
	
	Using Lemma \ref{lemma:mean-parameters}, we have
	\begin{align}
	\| \bnu(\brho)  - \hunu (\brho)\|_{\infty} \leq \epsFive
	\end{align}
	with probability at least $1-\deltaFive / \tau_3$.
	
	Let us condition on the events that $\| \bupStar - \hbup \|_{\infty} \leq \epsFive$ and that, for each of the $\tau_3$ steps of Algorithm \ref{alg:GradientDescent}, $\| \bnu(\brho)  - \hunu (\brho)\|_{\infty} \leq \epsFive$. 
	These events simultaneously hold with probability at least $1 - 2 \deltaFive$.

	Now for any $r \leq \tau_3+1$ the following hold with probability at least $1 - 2 \deltaFive$:
	\begin{align}
	\| \brho^{(r)} -  \brhoStar\|_2  & = \| \brho^{(r-1)} - \xi \hat{\gamma}_{\cP}(\brho^{(r-1)} ) - \brhoStar\|_2\\
	& = \| \brho^{(r-1)} - \xi [ \hat{\gamma}_{\cP}(\brho^{(r-1)} ) - \gamma_{\cP}(\brho^{(r-1)} ) + \gamma_{\cP}(\brho^{(r-1)} )] - \brhoStar\|_2\\
	& = \| \brho^{(r-1)} - \xi \gamma_{\cP}(\brho^{(r-1)} ) - \xi [ \hat{\gamma}_{\cP}(\brho^{(r-1)} ) - \gamma_{\cP}(\brho^{(r-1)} ) ]  - \brhoStar\|_2 \\
	& \stackrel{(a)}{\leq} \| \brho^{(r-1)}  - \xi \gamma_{\cP}(\brho^{(r-1)} )  - \brhoStar\|_2 + \xi \|  \hat{\gamma}_{\cP}(\brho^{(r-1)} ) - \gamma_{\cP}(\brho^{(r-1)} ) \|_2  \\
	& \stackrel{(b)}{\leq} \| \brho^{(r-1)}  - \xi \gamma_{\cP}(\brho^{(r-1)} )  - \brhoStar\|_2  + \xi \|  \hunu (\brho^{(r-1)}) - \bnu (\brho^{(r-1)}) + \bupStar - \hbup\|_2 \\
	& \stackrel{(c)}{\leq} \| \brho^{(r-1)}  - \xi \gamma_{\cP}(\brho^{(r-1)} )  - \brhoStar\|_2  + \xi \|  \hunu (\brho^{(r-1)}) - \bnu (\brho^{(r-1)}) \|_2 + \xi \| \bupStar - \hbup\|_2 \\
	& \stackrel{(d)}{\leq} \| \brho^{(r-1)}  - \xi  \gamma_{\cP}(\brho^{(r-1)} )  - \brhoStar\|_{2}  + \xi \sqrt{k} \|  \hunu (\brho^{(r-1)}) - \bnu (\brho^{(r-1)}) \|_{\infty} + \xi \sqrt{k} \| \bupStar - \hbup\|_{\infty} \\
	& \stackrel{(e)}{\leq} \| \brho^{(r-1)}  - \xi \gamma_{\cP}(\brho^{(r-1)} )  - \brhoStar\|_2  + 2 \xi \sqrt{k} \epsFive  \label{eq:intermediateGradStep}
	\end{align}
	where $(a)$ follows from the triangle inequality, $(b)$ follows from the definitions of $\gamma_{\cP}(\brho)$ and $\hat{\gamma}_{\cP}(\brho)$ and because the projection onto a convex set is non-expansive i.e., $\| \Pi_{\cP}(\tbrho) - \Pi_{\cP}(\bbrho)\| \leq \|\tbrho - \bbrho\|$, $(c)$ follows from the triangle inequality, $(d)$ follows because $\forall$ $\svbv \in \Reals^k, \| \svbv\|_2 \leq \sqrt{k} \| \svbv \|_{\infty}$, and $(e)$ follows because of the conditioning.
	
	Squaring both sides of \eqref{eq:intermediateGradStep} the following hold with probability at least $1 - 2 \deltaFive$:
	\begin{align}
	\begin{aligned}
	\| \brho^{(r)} -  \brhoStar\|^2_2   & \leq \| \brho^{(r-1)}  - \xi \gamma_{\cP}(\brho^{(r-1)} ) - \brhoStar\|^2_2  + 4 \xi^2 k \epsFive^2 + 4 \xi \sqrt{k} \epsFive \| \brho^{(r-1)}  - \xi \gamma_{\cP}(\brho^{(r-1)} ) - \brhoStar\|_2 \\
	& \stackrel{(a)}{\leq} \| \brho^{(r-1)}  - \xi \gamma_{\cP}(\brho^{(r-1)} ) - \brhoStar\|^2_2  + 4 \xi^2 k \epsFive^2 +  4 \xi \sqrt{k} \epsFive\Big[\| \brho^{(r-1)}  - \brhoStar\|_2  + \xi \|  \gamma_{\cP}(\brho^{(r-1)} ) \|_2 \Big]\\
	&\stackrel{(b)}{\leq}  \| \brho^{(r-1)}  - \xi \gamma_{\cP}(\brho^{(r-1)} ) - \brhoStar\|^2_2  + 4 \xi^2 k \epsFive^2 +  4 \xi \sqrt{k} \epsFive \Big[\| \brho^{(r-1)}  - \brhoStar\|_2  + \xi \|\bnu (\brho) - \bupStar \|_2 \Big]\\
	&\stackrel{(c)}{\leq}  \| \brho^{(r-1)}  - \xi \gamma_{\cP}(\brho^{(r-1)} ) - \brhoStar\|^2_2  + 4 \xi^2 k \epsFive^2+ 8 \xi k \epsFive  (\rhoMax + \xi \phiMax)  \label{eq:squared-grad} 
	\end{aligned}
	\end{align}
	where $(a)$ follows from the triangle inequality, $(b)$ follows by using the non-expansive property to observe that $\| \gamma_{\cP}(\brho)\|_2 \leq \|\bnu (\brho) - \bupStar \|_2$, and $(c)$ follows because $\| \brho^{(r-1)} -  \brhoStar\|_2 \leq 2 \sqrt{k} \rhoMax $ and $\|\bnu (\brho^{(r-1)}) - \bupStar\|_2 \leq 2\sqrt{k}\phiMax$.
	
	Letting $\Upsilon(\xi) := 4\xi^2 k \epsFive^2 + 8 \xi k \epsFive  (\rhoMax + \xi \phiMax) $, the following hold with probability at least $1 - 2 \deltaFive$:
	\begin{align}
	\| \brho^{(r)} -  \brhoStar\|^2_2   & \leq \| \brho^{(r-1)}  - \xi \gamma_{\cP}(\brho^{(r-1)} )- \brhoStar\|^2_2  + \Upsilon(\xi) \\
	& \stackrel{(a)}{=} \| \brho^{(r-1)}- \brhoStar\|^2_2 + \xi^2 \|  \gamma_{\cP}(\brho^{(r-1)} ) \|^2_2 - 2\xi \hspace{-1mm} \left \langle \hspace{-1mm} \gamma_{\cP}(\brho^{(r-1)} ),  \brho^{(r-1)}- \brhoStar \hspace{-1mm} \right \rangle \hspace{-1mm}+ \Upsilon(\xi)  \label{eq:grad-des}
	\end{align}
	where $(a)$ follows from the fact that for any two vectors $\svbf_1, \svbf_2$, $\| \svbf_1 - \svbf_2\|^2_2 = \| \svbf_1\|^2_2 + \| \svbf_2\|^2_2 - 2 \left\langle \svbf_1 , \svbf_2 \right\rangle $.
	
	For a twice differentiable, $\bcOne$ strongly convex and $\bcTwo$ smooth function $\Omega(\brho)$, we have,  for any $\brho \in \cP$
	\begin{align}
	\left\langle  \gamma_{\cP}(\brho),  \brho - \brhoStar  \right\rangle & \geq \frac{\bcOne}{2} \| \brho - \brhoStar\|^2_2 + \frac{1}{2\bcTwo} \|  \gamma_{\cP}(\brho)\|^2_2 \label{eq:strong-convex-smooth}
	\end{align}
	where $\brhoStar$ is the minimizer of $\Omega(\brho)$. 
	See \cite{Bubeck2015} for more details. 
	Using \eqref{eq:strong-convex-smooth} in \eqref{eq:grad-des}, the following hold with probability at least $1 - 2 \deltaFive$:
	\begin{align}
	\| \brho^{(r)} -  \brhoStar\|^2_2   & \leq (1-\xi \bcOne ) \| \brho^{(r-1)}- \brhoStar\|^2_2 + \bigg(\xi^2 - \frac{\xi}{\bcTwo}\bigg)\|  \gamma_{\cP}(\brho^{(r-1)} ) \|^2_2 + \Upsilon(\xi)
	\end{align}
	Substituting $\xi = \frac{1}{\bcTwo}$, the following hold with probability at least $1 - 2 \deltaFive$:
	\begin{align}
	\| \brho^{(r)} -  \brhoStar\|^2_2 & \leq   \bigg(1- \frac{\bcOne}{\bcTwo} \bigg) \| \brho^{(r-1)}- \brhoStar\|^2_2 + \Upsilon\bigg(\frac{1}{\bcTwo}\bigg)
	\end{align}
	Unrolling the recurrence gives, we have the following with probability at least $1 - 2 \deltaFive$:
	\begin{align}
	\| \brho^{(r)} -  \brhoStar\|^2_2  & \leq \bigg(1- \frac{\bcOne}{\bcTwo} \bigg)^r \| \brho^{(0)}- \brhoStar\|^2_2 + \sum_{j = 0}^{r-1}\bigg(1- \frac{\bcOne}{\bcTwo} \bigg)^j \Upsilon\bigg(\frac{1}{\bcTwo}\bigg) \\
	& \leq \bigg(1- \frac{\bcOne}{\bcTwo} \bigg)^r \| \brho^{(0)}- \brhoStar\|^2_2 + \sum_{j = 0}^{\infty}\bigg(1- \frac{\bcOne}{\bcTwo} \bigg)^j \Upsilon\bigg(\frac{1}{\bcTwo}\bigg) \\
	& \stackrel{(a)}{=}  \bigg(1- \frac{\bcOne}{\bcTwo} \bigg)^r \|  \brhoStar\|^2_2 + \bcThree \\
	& \stackrel{(b)}{\leq}  \exp( \frac{-\bcOne r }{\bcTwo}) \|  \brhoStar\|^2_2 + \bcThree
	\end{align}
	where $(a)$ follows by observing that $\frac{\bcTwo}{\bcOne}\Upsilon\big(\frac{1}{\bcTwo}\big) = \bcThree$ and $\brho^{(0)} = (0,\cdots , 0)$, and $(b)$ follows because for any $y \in \Reals$, $1-y \leq e^{-y} $.

	A sufficient condition for $ \| \brho^{(r)} -  \brhoStar\|_2 \leq \epsSix$ with probability at least $1 - 2 \deltaFive$ is
	\begin{align}
	\exp( \frac{-\bcOne r }{\bcTwo}) \| \brhoStar\|^2_2 + \bcThree & \leq \epsSix^2
	\end{align}
	Rearraning gives us,
	\begin{align}
	\exp( \frac{\bcOne r }{\bcTwo}) & \geq \frac{\| \brhoStar\|^2_2 }{\epsSix^2 - \bcThree} 
	\end{align}
	Taking logarithm on both sides, we have
	\begin{align}
	r & \geq \frac{\bcTwo }{\bcOne} \log \bigg(\frac{\| \brhoStar\|^2_2 }{\epsSix^2 - \bcThree}\bigg)
	\end{align}
	Observe that $\| \brhoStar\|^2_2 \leq k \rhoMax^2 $. Therefore, after $\tau_3$ steps, we have $\| \hbrho -  \brhoStar\|_2 \leq \epsSix$ with probability at least $1 - 2 \deltaFive$ and this completes the proof.
\end{proof}

\section{Examples of distributions}
\label{sec: examples of distributions appendix}%LABEL
In this section, we discuss the examples of distributions from Section \ref{sec:main results} that satisfy the Condition \ref{condition:1}. We also discuss a few other examples.

Recall the definitions of $\g = \thetaMax(k+k^2d)$ and $\varphiMax = (1+\Xupper)  \max\{\phiMax,\phiMax^2\}$ from Section \ref{sec:algorithm}.
Also recall the definitions of $f_L \coloneqq \exp \big( -2\g \varphiMax \big) / \Xupper$ and $f_U \coloneqq \exp \big( 2\g \varphiMax \big) / \Xlower$ from Appendix \ref{sec:conditional density}.

\subsection{Example 1}
\label{subsec:example1}
The following distribution with polynomial sufficient statistics is a special case of density in \eqref{eq:pairwise-parametric-density} with $\bphi(x) = x$ and $k = 1$. 
Let $\forall i \in [p]$, $\cX_i = [-b,b]$. 
Therefore $\Xlower = \Xupper = 2b$, $\phiMax = b$ and $\hphiMax = 1$. 
The density, in this case, is given by
\begin{align}
\DensityParametrizedTrue \propto \exp\bigg(  \sum_{i \in [p]} \thetaIStar x_i +  \sum_{ i \in [p]} \sum_{ j > i} \thetaIJStar x_i x_j  \bigg)   \label{eq:specialDensity1}
\end{align}
For this density, we see that $\g = \thetaMax(d+1)$ and $\varphiMax = (1+2b) \max\{b,b^2\}$. 
Let us first lower bound the conditional entropy of $\rvx_j$ given $\rvx_{-j}$. 
\begin{align}
h\bigg( \rvx_j \bigg| \rvx_{-j} \bigg) & = - \int_{\svbx \in \cX} \DensityParametrizedTrue \log \ConditionalDensityNodeJ d\svbx \\
& \stackrel{(a)}{\geq} - \int_{\svbx \in \cX} \DensityParametrizedTrue \log (f_U) d\svbx \\
& \stackrel{(b)}{=} - \log f_U  \label{eq:conditionalEntropy1}
\end{align} 
where $(a)$ follows from \eqref{eq:boundsDensity} with $f_U = \exp(2 \thetaMax (d+1) (1+2b)\max\{b,b^2\}) / 2b$ and $(b)$ follows because the integral of any density function over its entire domain is 1.

Observing that $\int_{x_i \in \cX_i} x_ix_j dx_i = 0$, the left-hand-side of Condition \ref{condition:1} can be written and simplified as follows:
\begin{align}
\Expectation\bigg[ \exp\bigg\{2h\bigg( (\ctheta_{ij} - \ttheta_{ij})  \rvx_i\rvx_j  \bigg| \rvx_{-j} \bigg) \bigg\} \bigg]  &  \stackrel{(a)}{=} \Expectation\Big[ \exp\bigg\{2h\Big( \rvx_j \Big| \rvx_{-j}\Big)  + 2 \log\Big| (\ctheta_{ij} - \ttheta_{ij}) \rvx_i  \Big|\bigg\} \Big] \\
&   \stackrel{(b)}{\geq} \Expectation\Big[ \exp\bigg\{-2 \log f_U + 2 \log\Big| (\ctheta_{ij} - \ttheta_{ij}) \rvx_i  \Big| \bigg\} \Big] \\
&   = \frac{ (\ctheta_{ij} - \ttheta_{ij})^2}{f_U^2} \Expectation\Big[ \rvx_i^2 \Big] \\
&    \stackrel{(c)}{=}  \frac{ (\ctheta_{ij} - \ttheta_{ij})^2}{f_U^2} \Expectation \bigg[\Expectation\Big[ \rvx_i^2 | \rvx_{-i}\Big] \bigg]\\
&    \stackrel{(d)}{=}  \frac{ (\ctheta_{ij} - \ttheta_{ij})^2}{f_U^2} \Expectation \bigg[\int_{x_i \in \cX_i }x_i^2 \ConditionalDensityNodeI dx_i \bigg]\\
&   \stackrel{(e)}{\geq} \frac{ f_L (\ctheta_{ij} - \ttheta_{ij})^2}{f_U^2} \Expectation \Big[\int_{x_i \in \cX_i} x_i^2 dx_i\Big]\\
&   \geq \frac{ 2b^3 f_L}{3f_U^2}  (\ctheta_{ij} - \ttheta_{ij})^2
\end{align}
where $(a)$ follows because for a constant $a$, $h(aX) = h(X) + \log |a|$, $(b)$ follows from \eqref{eq:conditionalEntropy1}, $(c)$ follows from the law of total expectation, $(d) $ follows from the definition of conditional expectation, and $(e)$ follows from \eqref{eq:boundsDensity}.

Substituting for $f_L$ and $f_U$, we see this density satisfies Condition \ref{condition:1} with $\kappa = \frac{ 4b^4}{3} \exp(-6 \thetaMax (d+1) (1+2b) \max\{b,b^2\} )$.

\subsection{Example 2}
\label{subsec:example2}
The following distribution with harmonic sufficient statistics is a special case of density in \eqref{eq:pairwise-parametric-density} with $\bphi(x) = \Big(\sin\big(\pi x/b\big), \cos\big(\pi x/b\big)\Big)$ and $k = 2$. Let $\forall i \in [p]$, $\cX_i =  [-b,b]$. 
Therefore $\Xlower = \Xupper = 2b$, $\phiMax = 1$, and $\hphiMax = \pi/b$. The density in this case is given by
\begin{align}
\DensityParametrizedTrue \propto & \exp\bigg( \sum_{i \in [p]}  \Big[ \thetaIoneStar \sin \frac{\pi x_i}{b} + \thetaItwoStar \cos \frac{\pi x_i}{b} \Big]    \\
& \qquad + \sum_{ i \in [p]  j > i}  \Big[ \thetaIJoneStar \sin \frac{\pi(x_i + x_j)}{b} + \thetaIJtwoStar \cos \frac{\pi (x_i + x_j)}{b} \Big]  \bigg)  \label{eq:density_harmonic}
\end{align}
For this density, we see that $\g = \thetaMax(4d+2)$ and $\varphiMax = 1+2b$.
%\sin \Big( \frac{\pi(\rvx_i + \rvx_j)}{b} - \taninv(\frac{\beta}{\alpha}) \Big)
Let $\rvy_j = \sin \Big( \frac{\pi \rvx_j}{b} + z \Big)$ where $z$ is a constant with respect to $\rvx_j$. 
Then, the conditional density of $\rvy_j$ given $\rvx_{-j}$ can be obtained using the change of variables technique to be as follows:
\begin{align}
\ConditionalDensityNodeJY & =
\begin{cases}
\dfrac{b \Big[f_{\rvx_j}\big( \frac{b}{\pi} [\sininv y_j - z] \big| x_{-j} ; \bvthetaJStar\big) +  f_{\rvx_j}\big( -b - \frac{b}{\pi} [\sininv y_j +z] \big| x_{-j} ; \bvthetaJStar\big)\Big]}{\pi \sqrt{1-y_j^2}} \\ \qquad \qquad \qquad \qquad \qquad \qquad    \qquad \qquad \qquad \qquad  \hspace{1mm}\text{if} \hspace{1mm}  y_j \in [-1,0]\\\\
\dfrac{b \Big[f_{\rvx_j}\big( \frac{b}{\pi} [\sininv y_j - z] \big| x_{-j} ; \bvthetaJStar\big)   f_{\rvx_j}\big( b - \frac{b}{\pi} [\sininv y_j + z] \big|  x_{-j} ; \bvthetaJStar\big)\Big]}{\pi \sqrt{1-y_j^2}} \\ \qquad \qquad \qquad \qquad \qquad \qquad    \qquad \qquad \qquad \qquad \hspace{1mm}\text{if}  \hspace{1mm} y_j \in [0,1]
\end{cases}
\end{align}
Using \eqref{eq:boundsDensity}, we can bound the above the conditional density as:
\begin{align}
\ConditionalDensityNodeJY & \leq  \dfrac{2b f_U}{\pi \sqrt{1-y_j^2}}  \label{eq:bound_sinusoid_density}
\end{align}
where $f_U = \exp(2 \thetaMax (4d+2) (1+2b)) / 2b$.

Let us now lower bound the conditional entropy of $\rvy_j $ given $\rvx_{-j} =  x_{-j}$. 
\begin{align}
h\bigg( \rvy_j \bigg| \rvx_{-j} = x_{-j} \bigg) & = - \int_{y_j = -1}^{y_j = 1} \ConditionalDensityNodeJY \log \ConditionalDensityNodeJY dy_j \\
& \stackrel{(a)}{\geq} - \int_{y_j = -1}^{y_j = 1} \dfrac{2b f_U}{\pi \sqrt{1-y_j^2}} \log \dfrac{2b f_U}{\pi \sqrt{1-y_j^2}} dy_j\\ 
& = - \dfrac{2b f_U}{\pi} \bigg[\log \dfrac{2b f_U}{\pi}  \int_{y_j = -1}^{y_j = 1}  \dfrac{1}{\sqrt{1-y_j^2}} dy_j - \int_{y_j = -1}^{y_j = 1}  \dfrac{1}{\sqrt{1-y_j^2}} \log \sqrt{1-y_j^2} dy_j  \bigg]\\ 
& \stackrel{(b)}{=} - \dfrac{2b f_U}{\pi} \bigg[\pi \log \dfrac{2b f_U}{\pi}  + \pi \log 2 \bigg]\\ 
& = -2bf_U \log  \dfrac{4b f_U}{\pi}  \label{eq:LB_conditional_sinusoid_instance}
\end{align} 
where $(a)$ follows from \eqref{eq:bound_sinusoid_density} and $(b)$ follows from standard definite integrals. Now, we are in a position to lower bound the conditional entropy of $\rvy_j $ given $\rvx_{-j} $. 
\begin{align}
h\bigg( \rvy_j \bigg| \rvx_{-j} \bigg) & = \int_{{x_{-j}} \in \prod_{r \neq j} \cX_r }  f_{ \rvx_{-j} } ( x_{-j} ; \bthetaStar) h\bigg( \rvy_j \bigg| \rvx_{-j} = x_{-j} \bigg) dx_{-j}  \\
& \stackrel{(a)}{\geq} -2bf_U \log  \dfrac{4b f_U}{\pi}  \label{eq:conditional_entropy_sinusoid_LB}
\end{align} 
where $(a)$ follows from \eqref{eq:LB_conditional_sinusoid_instance} and because the integral of any density function over its entire domain is 1.

Observe that $\int_{x_i \in \cX_i} \sin\big(\frac{\pi(x_i+x_j) }{b}\big) dx_i = \int_{x_i \in \cX_i} \cos\big(\frac{\pi (x_i+x_j)}{b}\big) dx_i  = 0$. Letting $\ctheta_{ij}^{(1,1)} - \ttheta_{ij}^{(1,1)} = \alpha$ and $\ctheta_{ij}^{(2,1)} - \ttheta_{ij}^{(2,1)} = \beta$, the left-hand-side of Condition \ref{condition:1} can be written and simplified as follows:
\begin{align}
& \Expectation\bigg[ \exp\bigg\{2h\bigg(  \alpha \sin\big(\frac{\pi(\rvx_i+\rvx_j) }{b}\big) + \beta \cos\big(\frac{\pi(\rvx_i+\rvx_j) }{b}\big)  \bigg| \rvx_{-j} \bigg) \bigg\} \bigg]  \\
& \qquad \stackrel{(a)}{=} \Expectation\bigg[ \exp\bigg\{2h\bigg(  \sqrt{\alpha^2 + \beta^2} \sin\Big(\frac{\pi(\rvx_i+\rvx_j) }{b} - \taninv \frac{\beta}{\alpha}\Big) \bigg| \rvx_{-j} \bigg) \bigg\} \bigg]  \\
& \qquad \stackrel{(b)}{=} \Expectation\Big[ \exp\bigg\{2h\Big( \sin\Big(\frac{\pi(\rvx_i+\rvx_j) }{b} - \taninv \frac{\beta}{\alpha}\Big)  \Big| \rvx_{-j}\Big)  + 2 \log\Big|\sqrt{\alpha^2 + \beta^2}  \Big| \bigg\} \Big] \\
& \qquad  \stackrel{(c)}{\geq} \Expectation\Big[ \exp\bigg\{-4bf_U \log  \dfrac{4b f_U}{\pi} + \log\Big|\alpha^2 + \beta^2 \Big| \bigg\} \Big] \\
& \qquad  \stackrel{(d)}{=}  \Big(\frac{ \pi}{4b f_U}\Big)^{4b f_U} \times \Big[(\ctheta_{ij}^{(1,1)} - \ttheta_{ij}^{(1,1)})^2 +  (\ctheta_{ij}^{(2,1)} - \ttheta_{ij}^{(2,1)})^2 \Big]
\end{align}
where $(a)$ follows from standard trigonometric identities, $(b)$ follows because for a constant $a$, $h(aX) = h(X) + \log |a|$, $(c)$ follows from \eqref{eq:conditional_entropy_sinusoid_LB} with $z = \pi \rvx_i / b - \taninv \beta / \alpha$, and $(d)$ follows by substituting for $\alpha$ and $\beta$.

Substituting for $f_L$ and $f_U$, we see this density satisfies Condition \ref{condition:1} with $\kappa = \Big(\frac{ \pi \exp(-2 \thetaMax (4d+2)(1+2b)) }{2} \Big)^{2\exp(2 \thetaMax (4d+2)(1+2b))}$.

\subsection{Example 3}
\label{subsec:example3}
The following distribution with polynomial sufficient statistics is a special case of density in \eqref{eq:pairwise-parametric-density} with $\bphi(x) = (x, x^2)$, $k = 2$ and with the assumption that the parameters associated with $x_i x_j^2$ and $x_i^2 x_j^2$ are zero $\forall i \in [p], j > i$. Let $\forall i \in [p]$, $\cX_i = [-b,b]$. 
Therefore $\Xlower = \Xupper = 2b$$, \phiMax = \max\{b,b^2\}$, and $\hphiMax = \max\{1,2b\}$.
The density in this case is given by
\begin{align}
\DensityParametrizedTrue \propto \exp\bigg(  \sum_{i \in [p]} [ \thetaIoneStar x_i + \thetaItwoStar x_i^2 ]+  \sum_{ i \in [p]} \sum_{ j > i} [\thetaIJoneoneStar x_i x_j  + \thetaIJtwooneStar x_i^2 x_j ] \bigg)   \label{eq:specialDensity3}
\end{align}
For this density, we see that $\g = \thetaMax(4d+2)$ and $\varphiMax = (1+2b)\max\{b,b^4\}$. 
As in Appendix \ref{subsec:example1}, we have the following lower bound on the conditional entropy of $\rvx_j$ given $\rvx_{-j}$
\begin{align}
h\bigg( \rvx_j \bigg| \rvx_{-j}\bigg) & \geq - \log f_U  \label{eq:conditionalEntropy2}
\end{align} 
where $f_U = \exp(2 \thetaMax (4d+2) (1+2b)\max\{b,b^4\}) / 2b$.

Observing that $\int_{x_i \in \cX_i} x_ix_j dx_i = 0$ and $\int_{x_i \in \cX_i} x_i^2x_j dx_i = \frac{2b^3}{3} x_j$, the left-hand-side of Condition \ref{condition:1} can be written and simplified as follows:
\begin{align}
& \Expectation\bigg[ \exp\bigg\{2h\bigg( (\ctheta_{ij}^{(1,1)} - \ttheta_{ij}^{(1,1)})  \rvx_i\rvx_j  + (\ctheta_{ij}^{(2,1)} - \ttheta_{ij}^{(2,1)}) \Big( \rvx_i^2\rvx_j-  \frac{2b^3}{3} \rvx_j \Big) \bigg| \rvx_{-j} \bigg) \bigg\} \bigg]  \\
& \qquad  = \Expectation\bigg[ \exp\bigg\{2h\bigg( \bigg[ (\ctheta_{ij}^{(1,1)} - \ttheta_{ij}^{(1,1)})  \rvx_i + (\ctheta_{ij}^{(2,1)} - \ttheta_{ij}^{(2,1)}) \Big( \rvx_i^2 - \frac{2b^3}{3}\Big) \bigg] \rvx_j \bigg| \rvx_{-j} \bigg) \bigg\} \bigg]  \\
& \qquad  \stackrel{(a)}{=} \Expectation\bigg[ \exp\bigg\{2h\bigg( \rvx_j \bigg| \rvx_{-j}\bigg)  + 2 \log\bigg|  (\ctheta_{ij}^{(1,1)} - \ttheta_{ij}^{(1,1)})  \rvx_i + (\ctheta_{ij}^{(2,1)} - \ttheta_{ij}^{(2,1)}) \Big( \rvx_i^2 - \frac{2b^3}{3} \Big) \bigg| \bigg\} \bigg] \\
& \qquad  \stackrel{(b)}{\geq}  \Expectation\bigg[ \exp\bigg\{-2 \log f_U  + 2 \log\bigg|  (\ctheta_{ij}^{(1,1)} - \ttheta_{ij}^{(1,1)})  \rvx_i + (\ctheta_{ij}^{(2,1)} - \ttheta_{ij}^{(2,1)}) \Big( \rvx_i^2 - \frac{2b^3}{3} \Big) \bigg|\bigg\} \bigg] \\
& \qquad  = \frac{ 1}{f_U^2} \Expectation\bigg[ \bigg((\ctheta_{ij}^{(1,1)} - \ttheta_{ij}^{(1,1)})  \rvx_i + (\ctheta_{ij}^{(2,1)} - \ttheta_{ij}^{(2,1)}) \Big( \rvx_i^2 - \frac{2b^3}{3} \Big) \bigg)^2\bigg] \\
& \qquad  \stackrel{(c)}{=}  \frac{ 1}{f_U^2} \Expectation\bigg[ \Expectation\bigg[ \bigg((\ctheta_{ij}^{(1,1)} - \ttheta_{ij}^{(1,1)})  \rvx_i + (\ctheta_{ij}^{(2,1)} - \ttheta_{ij}^{(2,1)}) \Big( \rvx_i^2 - \frac{2b^3}{3} \Big) \bigg)^2 \bigg| \rvx_{-i}\bigg] \bigg]\\
& \qquad  \stackrel{(d)}{\geq} \frac{ f_L}{f_U^2} \Expectation \bigg[(\ctheta_{ij}^{(1,1)} - \ttheta_{ij}^{(1,1)})^2 \int_{x_i \in \cX_i} x_i^2 dx_i + (\ctheta_{ij}^{(2,1)} - \ttheta_{ij}^{(2,1)})^2  \bigg(\int_{x_i \in \cX_i} \Big[ x_i^4 + \frac{4b^6}{9} - \frac{4b^3}{3} x_i^2 \Big] dx_i \bigg)  \\ & \qquad \qquad \qquad + 2 (\ctheta_{ij}^{(1,1)} - \ttheta_{ij}^{(1,1)}) (\ctheta_{ij}^{(2,1)} - \ttheta_{ij}^{(2,1)}) \bigg(\int_{x_i \in \cX_i} \Big[ x_i^3 - \frac{2b^3}{3} x_i \Big] dx_i \bigg) \bigg]\\
& \qquad  = \frac{ f_L}{f_U^2}  \bigg[\frac{2b^3}{3}(\ctheta_{ij}^{(1,1)} - \ttheta_{ij}^{(1,1)})^2 + \bigg(\frac{2b^5}{5} + \frac{8b^7}{9} -  \frac{8b^6}{9} \bigg) (\ctheta_{ij}^{(2,1)} - \ttheta_{ij}^{(2,1)})^2 \bigg]\\
& \qquad  \stackrel{(e)}{\geq}  \frac{ f_L}{f_U^2} \bigg[\frac{2b^3}{3}(\ctheta_{ij}^{(1,1)} - \ttheta_{ij}^{(1,1)})^2 +  \frac{8b^5}{45}(\ctheta_{ij}^{(2,1)} - \ttheta_{ij}^{(2,1)})^2 \bigg] \\
& \qquad  \geq  \frac{ 8f_L b^3 \min\{45/12,b^2\}}{45 f_U^2} \bigg[(\ctheta_{ij}^{(1,1)} - \ttheta_{ij}^{(1,1)})^2 +  (\ctheta_{ij}^{(2,1)} - \ttheta_{ij}^{(2,1)})^2 \bigg]
\end{align}
where $(a)$ follows because for a constant $a$, $h(aX) = h(X) + \log |a|$, $(b)$ follows from \eqref{eq:conditionalEntropy2}, $(c)$ follows from the law of total expectation, $(d) $ follows from the definition of conditional expectation and \eqref{eq:boundsDensity}, and $(e)$ follows because $8b^2/9 - 8b/9 + 2/5 \geq 8/45$.

Substituting for $f_L$ and $f_U$, we see this density satisfies Condition \ref{condition:1} with $\kappa = \frac{ 16 b^4 \min\{45/12,b^2\}}{45} \exp(-6 \thetaMax (4d+2) (1+2b)\max\{b,b^4\})$.

\subsection{Example 4}
\label{subsec:example4}
The following distribution with polynomial sufficient statistics is a special case of density in \eqref{eq:pairwise-parametric-density} with $\bphi(x) = (x, x^2)$, $k = 2$ and with the assumption that the parameters associated with $x_i x_j$, $x_i^2 x_j$, $x_i x_j^2$ and $x_i^2 x_j^2$ are same $\forall i \in [p], j > i$.
Let $\forall i \in [p]$, $\cX_i = [-b,b]$. 
Therefore $\Xlower = \Xupper = 2b$$, \phiMax = \max\{b,b^2\}$, and $\hphiMax = \max\{1,2b\}$.
The density in this case is given by
\begin{align}
\DensityParametrizedTrue \propto \exp\bigg(  \sum_{i \in [p]} [ \thetaIoneStar x_i + \thetaItwoStar x_i^2 ]+  \sum_{\substack{ i \in [p] \\j > i}} \thetaIJStar (x_i +x_i^2)(x_j +x_j^2) \bigg)  \label{eq:specialDensity4}
\end{align}
For this density, we see that $\g = \thetaMax(4d+2)$ and $\varphiMax = (1+2b)\max\{b,b^4\}$. 
Let $\rvy_j = \rvx_j + \rvx_j^2$. 
It is easy to obtain the range of $\rvy_j$ as follows:
\begin{align}
\rvy_j \in \cY & \coloneqq 
\begin{cases}
[-1/4 , b+ b^2] &\quad\text{if} \quad b \geq 1/2\\
[b - b^2 , b+ b^2] &\quad\text{if} \quad b < 1/2\\
\end{cases}  \label{eq:y_range}
\end{align}
We obtain the conditional density of $\rvy_j$ given $\rvx_{-j}$ using the change of variables technique and upper bound it using \eqref{eq:boundsDensity} as follows:
\begin{align}
\ConditionalDensityNodeJY & \leq  \dfrac{2 f_U}{\sqrt{1+4y_j}}  \label{eq:bound_poly_deg2_density}
\end{align}
where $f_U = \exp(2 \thetaMax (4d+2) (1+2b)\max\{b,b^4\}) / 2b$.

We will now lower bound the conditional entropy of $\rvy_j $ given $\rvx_{-j} =  x_{-j}$. 
In the first scenario, let $b \geq 1/2$. 
\begin{align}
h\bigg( \rvy_j \bigg| \rvx_{-j} = x_{-j} \bigg) & = - \int_{y_j = -1/4}^{y_j = b+b^2} \ConditionalDensityNodeJY \log \ConditionalDensityNodeJY dy_j \\
& \stackrel{(a)}{\geq} - \int_{y_j = -1/4}^{y_j = b+b^2} \dfrac{2 f_U}{\sqrt{1+4y_j}}  \log \dfrac{2 f_U}{\sqrt{1+4y_j}}  dy_j\\ 
& \stackrel{(b)}{=} - f_U \sqrt{1+4y_j} \log \dfrac{2 f_Ue}{\sqrt{1+4y_j}} \Biggr|_{-1/4}^{b+b^2}\\
& = -(1+2b)f_U \log \dfrac{2 f_Ue}{1+2b}  \label{eq:LB_conditional_entropy_poly_degree2_instance1}
\end{align} 
where $(a)$ follows from \eqref{eq:bound_poly_deg2_density} and $(b)$ follows from standard indefinite integrals. 
Now, we will lower bound the conditional entropy of $\rvy_j $ given $\rvx_{-j} =  x_{-j}$ and $b < 1/2$. 
\begin{align}
h\bigg( \rvy_j \bigg| \rvx_{-j} = x_{-j} \bigg) & = - \int_{y_j = b - b^2}^{y_j = b+b^2} \ConditionalDensityNodeJY \log \ConditionalDensityNodeJY dy_j \\
& \stackrel{(a)}{\geq} - \int_{y_j = b - b^2}^{y_j = b+b^2} \dfrac{2 f_U}{\sqrt{1+4y_j}}  \log \dfrac{2 f_U}{\sqrt{1+4y_j}}  dy_j\\ 
& \stackrel{(b)}{=} - f_U \sqrt{1+4y_j} \log \dfrac{2 f_Ue}{\sqrt{1+4y_j}} \Biggr|_{b - b^2}^{b+b^2}\\
& = -(1+2b)f_U \log \dfrac{2 f_Ue}{1+2b} + (1-2b)f_U \log \dfrac{2 f_Ue}{1-2b} \\
& \stackrel{(c)}{\geq}  -(1+2b)f_U \log \dfrac{2 f_Ue}{1+2b}  \label{eq:LB_conditional_entropy_poly_degree2_instance2}
\end{align} 
where $(a)$ follows from \eqref{eq:bound_poly_deg2_density}, $(b)$ follows from standard indefinite integrals and $(c)$ follows because $(1-2b) \log \frac{2 f_Ue}{1-2b} > 0$ when $b < 1/2$. 
Now, we are in a position to lower bound the conditional entropy of $\rvy_j $ given $\rvx_{-j} $. 
\begin{align}
h\bigg( \rvy_j \bigg| \rvx_{-j} \bigg) & \int_{{x_{-j}} \in \prod_{r \neq j} \cX_r }  f_{ \rvx_{-j} } ( x_{-j} ; \bthetaStar) h\bigg( \rvy_j \bigg| \rvx_{-j} = x_{-j} \bigg) dx_{-j}  \\
& \stackrel{(a)}{\geq} -(1+2b)f_U \log \dfrac{2 f_Ue}{1+2b}   \label{eq:LB_conditional_entropy_poly_degree2}
\end{align} 
where $(a)$ follows from \eqref{eq:LB_conditional_entropy_poly_degree2_instance1}, \eqref{eq:LB_conditional_entropy_poly_degree2_instance2}, and because the integral of any density function over its entire domain is 1.

Observing that $\int_{x_i \in \cX_i} x_i dx_i = 0$ and $\int_{x_i \in \cX_i} x_i^2 dx_i = \frac{2b^3}{3} $, the left-hand-side of Condition \ref{condition:1} can be written and simplified as follows:
\begin{align}
& \Expectation\bigg[ \exp\bigg\{2h\bigg( \bigg(\ctheta_{ij} - \ttheta_{ij}\bigg)  \bigg(\rvx_i + \rvx_i^2 -  \frac{2b^3}{3}\bigg)  \bigg(\rvx_j + \rvx_j^2\bigg) \bigg| \rvx_{-j} \bigg) \bigg\} \bigg]  \\
& \qquad  \stackrel{(a)}{=} \Expectation\bigg[ \exp\bigg\{2h\bigg( \rvx_j + \rvx_j^2 \bigg| \rvx_{-j}\bigg)  + 2 \log\bigg|  \bigg(\ctheta_{ij} - \ttheta_{ij}\bigg)  \bigg(\rvx_i + \rvx_i^2 -  \frac{2b^3}{3}\bigg)  \bigg| \bigg\} \bigg] \\
& \qquad  \stackrel{(b)}{\geq}  \Expectation\bigg[ \exp\bigg\{-2(1+2b)f_U \log \dfrac{2 f_Ue}{1+2b}  + 2 \log\bigg|  \bigg(\ctheta_{ij} - \ttheta_{ij}\bigg)  \bigg(\rvx_i + \rvx_i^2 -  \frac{2b^3}{3}\bigg)  \bigg| \bigg\} \bigg] \\
& \qquad  = \bigg(\frac{ 1+2b}{2f_Ue}\bigg)^{2f_U(1+2b)} \Expectation\bigg[ \bigg(\ctheta_{ij} - \ttheta_{ij}\bigg)^2  \bigg(\rvx_i + \rvx_i^2 -  \frac{2b^3}{3}\bigg)^2\bigg] \\
& \qquad  \stackrel{(c)}{=}  f_L  (\ctheta_{ij} - \ttheta_{ij})^2 \bigg(\frac{2b^3}{3} + \frac{8b^5}{45} \bigg) \bigg(\frac{ 1+2b}{2f_Ue}\bigg)^{2f_U(1+2b)}
\end{align}
where $(a)$ follows because for a constant $a$, $h(aX) = h(X) + \log |a|$, $(b)$ follows from \eqref{eq:LB_conditional_entropy_poly_degree2}, $(c)$ follows from steps similar to the ones in Appendix \ref{subsec:example3}.

Substituting for $f_L$ and $f_U$, we see this density satisfies Condition \ref{condition:1} with  $\kappa = \frac{ e(15b+4b^3)}{45(1+2b)}  \Big(\frac{ b(1+2b) \exp(-2 \thetaMax  (4d+2) (1+2b)\max\{b,b^4\}) }{e} \Big)^{\frac{1+2b}{b}\exp(2 \thetaMax  (4d+2) (1+2b)\max\{b,b^4\}) + 1}$.

\section{Discussion on Theorem \ref{theorem:GRISE-consistency-efficiency}}
\label{sec: discussion on theorem}%LABEL
In this section, we discuss the invertibility of the cross-covariance matrix $B(\bvthetaIStar)$ via a simple example as well as explicitly show that the matrix $B(\bvthetaIStar)^{-1} A(\bvthetaIStar) B(\bvthetaIStar)^{-1}$ need not be equal to the inverse of the Fisher information matrix of $\rvbx$. This concludes that even though the estimator $\hbvthetaIn$ is asymptotically normal, it is not asymptotically efficient.

\subsection{Invertibility of the cross-covariance matrix}
\label{subsec:Invertibility of the cross-covariance matrix}
We will look at the special case of $\bphi(x) = x$ and $k = 1$ and show that the cross-covariance matrix of $\bvphiI(\rvbx)$ and $\bvphiI(\rvbx)\exp\big( -\bvthetaIStarT \bvphiI(\rvbx) \big)$ i.e., $B(\bvthetaIStar)$ is invertible $\forall i \in [p]$ when $p = 2$. Let $\cX_1 = [-b,b]$ and $\cX_2 = [-b,b]$.

The density in this special case is as follows.
\begin{align}
\DensityParametrizedTrue \propto \exp\bigg(  \thetaoneStar x_1 + \thetatwoStar x_2 +   \thetaonetwoStar x_1 x_2  \bigg)  \label{eq:continuous-ising-density-two-nodes}.
\end{align}
It is easy to see that the basis functions $x_1, x_2$, and $x_1 x_2$ are already locally centered. Therefore, we have
\begin{align}
\bvphione(\svbx) & = (x_1, x_1x_2)  \qquad \text{and} \qquad \bvphitwo(\svbx)  = (x_2, x_1x_2)\\
\bvthetaOneStar & = (\thetaoneStar, \thetaonetwoStar) \qquad \text{and} \qquad \bvthetaTwoStar  = (\thetatwoStar, \thetaonetwoStar)
\end{align}
Then, the cross-covariance matrices $B(\bvthetaOneStar)$ and $B(\bvthetaTwoStar)$ are\\\\
$B(\bvthetaOneStar)= 
\begin{bmatrix}
\Expectation[\rvx_1^2 \exp{(- \thetaoneStar \rvx_1 - \thetaonetwoStar \rvx_1\rvx_2)}] & \Expectation[\rvx_1^2 \rvx_2 \exp{(- \thetaoneStar \rvx_1 - \thetaonetwoStar \rvx_1\rvx_2)}]\\
\Expectation[\rvx_1^2 \rvx_2 \exp{(- \thetaoneStar \rvx_1 - \thetaonetwoStar \rvx_1\rvx_2)}] & \Expectation[\rvx_1^2 \rvx_2^2 \exp{(- \thetaoneStar \rvx_1 - \thetaonetwoStar \rvx_1\rvx_2)}]
\end{bmatrix}
$\\
and\\ 
$B(\bvthetaTwoStar)= 
\begin{bmatrix}
\Expectation[\rvx_2^2 \exp{(- \thetatwoStar \rvx_2 - \thetaonetwoStar \rvx_1\rvx_2)}] & \Expectation[\rvx_2^2 \rvx_1 \exp{(- \thetatwoStar \rvx_2 - \thetaonetwoStar \rvx_1\rvx_2)}]\\
\Expectation[\rvx_2^2 \rvx_1 \exp{(- \thetatwoStar \rvx_2 - \thetaonetwoStar \rvx_1\rvx_2)}] & \Expectation[\rvx_2^2 \rvx_1^2 \exp{(- \thetatwoStar \rvx_2 - \thetaonetwoStar \rvx_1\rvx_2)}]
\end{bmatrix}
$.\\\\
Using the Cauchy-Schwarz inequality, for random variables $M$ and $N$, we have
\begin{align}
[\Expectation(MN)]^2 \leq \Expectation(M^2) \Expectation(N^2)
\end{align}
with equality only if $M$ and $N$ are linearly dependent. Using the Cauchy-Schwarz inequality with $M = \rvx_1 \exp(-0.5 \thetaoneStar \rvx_1 - 0.5 \thetaonetwoStar \rvx_1 \rvx_2)$, $N = \rvx_1 \rvx_2 \exp(-0.5 \thetaoneStar \rvx_1 - 0.5 \thetaonetwoStar \rvx_1 \rvx_2)$ and observing that $M$ and $N$ are not linearly dependent (because $\rvx_2$ is a random variable), we have invertibility of $B(\bvthetaOneStar)$. Similarly, using the Cauchy-Schwarz inequality with $M = \rvx_2 \exp(-0.5 \thetatwoStar \rvx_2 - 0.5 \thetaonetwoStar \rvx_1 \rvx_2)$, $N = \rvx_2 \rvx_1 \exp(-0.5 \thetatwoStar \rvx_2 - 0.5 \thetaonetwoStar \rvx_1 \rvx_2)$ and observing that $M$ and $N$ are not linearly dependent (because $\rvx_1$ is a random variable), we have invertibility of $B(\bvthetaTwoStar)$.

\subsection{Fisher information matrix}
\label{subsec: Fisher information matrix}
Let $J(\bvthetaIStar)$ denote the Fisher information matrix of $\rvbx$ with respect to node $i$. For any $l \in [k+k^2(p-1)]$, let $\bvphiI_{l}(\svbx)$ denote the $l^{th}$ component of $\bvphiI(\svbx)$. Using the fact that for any regular exponential family the Hessian of the log partition function is the covariance matrix of the associated sufficient statistic vector, we have the Fisher information matrix
\begin{align}
\big[	J(\bvthetaIStar) \big]_{l_1,l_2} = \text{Cov} \big( \bvphiI_{l_1}(\svbx) , \bvphiI_{l_2}(\svbx)\big) 
\end{align}
Consider the density in \eqref{eq:continuous-ising-density-two-nodes} with $b = 1, \thetaoneStar = \thetatwoStar = 0$ and $\thetaonetwoStar = 1$. We will evaluate the matrices $B(\bvthetaOneStar)$, $A(\bvthetaOneStar)$, and $J(\bvthetaOneStar)$. We have \\
$B(\bvthetaOneStar)= 
\begin{bmatrix}
\Expectation[\rvx_1^2 \exp{(-  \rvx_1\rvx_2)}] & \Expectation[\rvx_1^2 \rvx_2 \exp{(-  \rvx_1\rvx_2)}]\\
\Expectation[\rvx_1^2 \rvx_2 \exp{(-  \rvx_1\rvx_2)}] & \Expectation[\rvx_1^2 \rvx_2^2 \exp{(-  \rvx_1\rvx_2)}]
\end{bmatrix} = \begin{bmatrix}
\frac{1}{3\text{Shi}(1)} & 0\\
0 & \frac{1}{9\text{Shi}(1)}\\
\end{bmatrix}
$\\\\\\
$A(\bvthetaOneStar)= 
\begin{bmatrix}
\Expectation[\rvx_1^2 \exp{(-  2\rvx_1\rvx_2)}] & \Expectation[\rvx_1^2 \rvx_2 \exp{(-  2\rvx_1\rvx_2)}]\\
\Expectation[\rvx_1^2 \rvx_2 \exp{(-  2\rvx_1\rvx_2)}] & \Expectation[\rvx_1^2 \rvx_2^2 \exp{(-  2\rvx_1\rvx_2)}]
\end{bmatrix} = \begin{bmatrix}
\frac{1}{e\text{Shi}(1)} & 0\\
0 & \frac{2\text{Shi}(1) + 2/e - e}{\text{Shi}(1)}\\
\end{bmatrix}
$\\\\\\
$J(\bvthetaOneStar)= 
\begin{bmatrix}
\text{Cov} (\rvx_1,\rvx_1)& \text{Cov} (\rvx_1,\rvx_1\rvx_2)\\
\text{Cov} (\rvx_1,\rvx_1\rvx_2) & \text{Cov} (\rvx_1\rvx_2,\rvx_1\rvx_2)
\end{bmatrix} = \begin{bmatrix}
\frac{1}{e\text{Shi}(1)} & 0\\
0 & \frac{2\text{Shi}(1) + 2/e - e}{\text{Shi}(1)} - \Big[\frac{\text{sinh}(1)}{\text{Shi}(1)}-1\Big]^2\\
\end{bmatrix}
$\\
where $\text{sinh}$ is the hyperbolic sine function and $\text{Shi}$ is the hyperbolic sine integral function. Plugging in the values of $\text{Shi}(1), \text{sinh}(1), $ and $e$, we have (upto two decimals)\\
$B(\bvthetaOneStar)^{-1} A(\bvthetaOneStar) B(\bvthetaOneStar)^{-1}= 
\begin{bmatrix}
3.50 & 0\\
0 & 11.30
\end{bmatrix} \neq
J^{-1}(\bvthetaOneStar)= 
\begin{bmatrix}
3.007 & 0\\
0& 8.90
\end{bmatrix} $

\end{document}